\documentclass[12pt,onecolumn]{article}
\usepackage[margin=0.75in]{geometry}

\usepackage[round]{natbib}
\renewcommand{\bibname}{References}
\renewcommand{\bibsection}{\subsubsection*{\bibname}}
\usepackage{booktabs}
\usepackage{graphicx}
\usepackage{float}
\usepackage{subcaption}
\usepackage{enumitem}
\usepackage[toc,page]{appendix}
\usepackage{bbm}
\usepackage{color}
\usepackage{comment}
\usepackage{amssymb}
\usepackage{amsmath}
\usepackage{amsthm}
\usepackage{bbm,bm}

\makeatletter
\newtheorem*{rep@theorem}{\rep@title}
\newcommand{\newreptheorem}[2]{%
\newenvironment{rep#1}[1]{%
 \def\rep@title{#2 \ref{##1}}%
 \begin{rep@theorem}}%
 {\end{rep@theorem}}}
\makeatother
\newtheorem{theorem}{Theorem}
\newreptheorem{theorem}{Theorem}
\newtheorem{lemma}{Lemma}
\newreptheorem{lemma}{Lemma}
\newtheorem{prop}{Proposition}
\newreptheorem{prop}{Proposition}

\newreptheorem{corollary}{Corollary}
\newtheorem{definition}{Definition}

\DeclareMathOperator*{\argmax}{arg\,max}

\usepackage{stfloats}
\usepackage{multirow}
\usepackage{multicol} 
\usepackage{algorithmic}
\usepackage{algorithm}
\DeclareUnicodeCharacter{00A0}{ }
\usepackage{soul}

\usepackage[skins,theorems]{tcolorbox}
\tcbset{highlight math style={enhanced,
  colframe=red,colback=white,arc=0pt,boxrule=1pt}}

\newcommand{\negspace}{\vspace{-0.05cm}}
\newcommand{\bignegspace}{\vspace{-0.3cm}}

\usepackage{mathtools}

\begin{document}

%

%


\title{Smoothly Giving up: Robustness for Simple Models}

\author{%
Tyler Sypherd$^{1}$, Nathan Stromberg$^{1}$,  Richard Nock$^{2}$, \\ Visar Berisha$^{1}$, and Lalitha Sankar$^{1}$ \\
$^{1}$Arizona State University \quad $^{2}$Google Research \\
}

\date{}
\maketitle

\begin{abstract}
There is a growing need for models that are interpretable and have reduced energy and computational cost (e.g., in health care analytics and federated learning). 
Examples of algorithms to train such models include logistic regression and boosting. 
However, one challenge facing these algorithms is that they provably suffer from label noise; this has been attributed to the joint interaction between oft-used convex loss functions and simpler hypothesis classes, resulting in too much emphasis being placed on outliers.
In this work, we use the margin-based $\alpha$-loss, which continuously tunes between canonical convex and quasi-convex losses, to robustly train simple models.
We show that the $\alpha$ hyperparameter smoothly introduces non-convexity and offers the benefit of ``giving up'' on noisy training examples. 
We also provide results on the Long-Servedio dataset for boosting and a COVID-19 survey dataset for logistic regression, highlighting the efficacy of our approach across multiple relevant domains.
\end{abstract}

\section{INTRODUCTION} \label{sec:intro}
In several critical infrastructure applications, simple models are favored over complex models.
In health care analytics, simple models are typically preferred for their interpretability so that practitioners can
audit the correlations the model uses for
decision making~\citep{rudin_2019,Caruana2015,pmlr-v139-nori21a,Caruana2021ebm}. 
In federated learning, simple models can be preferred for computational and energy efficiency, since edge devices are heterogeneous~\citep{kairouz2019advances,viola2001rapid}. 
Examples of learning algorithms that train simple models include logistic regression and boosting, particularly when the weak learner of the boosting algorithm is \textit{weaker} (e.g., decision/regression trees with low maximum depth).

While simple models may offer more interpretability or energy efficiency, 
they are known to suffer, provably, from label noise~\citep{ben2012minimizing,pmlr-v162-ji22a,rolnick2017deep}.
Indeed,~\cite{long2010random} showed that boosting algorithms that minimize convex losses over linear weak learners can achieve fair coin test accuracy after being trained with an arbitrarily small amount of (symmetric) label noise. 
In essence,~\cite{long2010random} construct a pathological dataset which exploits the sensitivity of linear classifiers and the \textit{inability} of convex losses to ``give up'' on noisy training examples, even if the convex boosting algorithm is regularized or stopped early.
\begin{figure}[h]
    \centerline{\includegraphics[trim={0.2cm .2cm .2cm 0.2cm},clip,width=.75\linewidth]{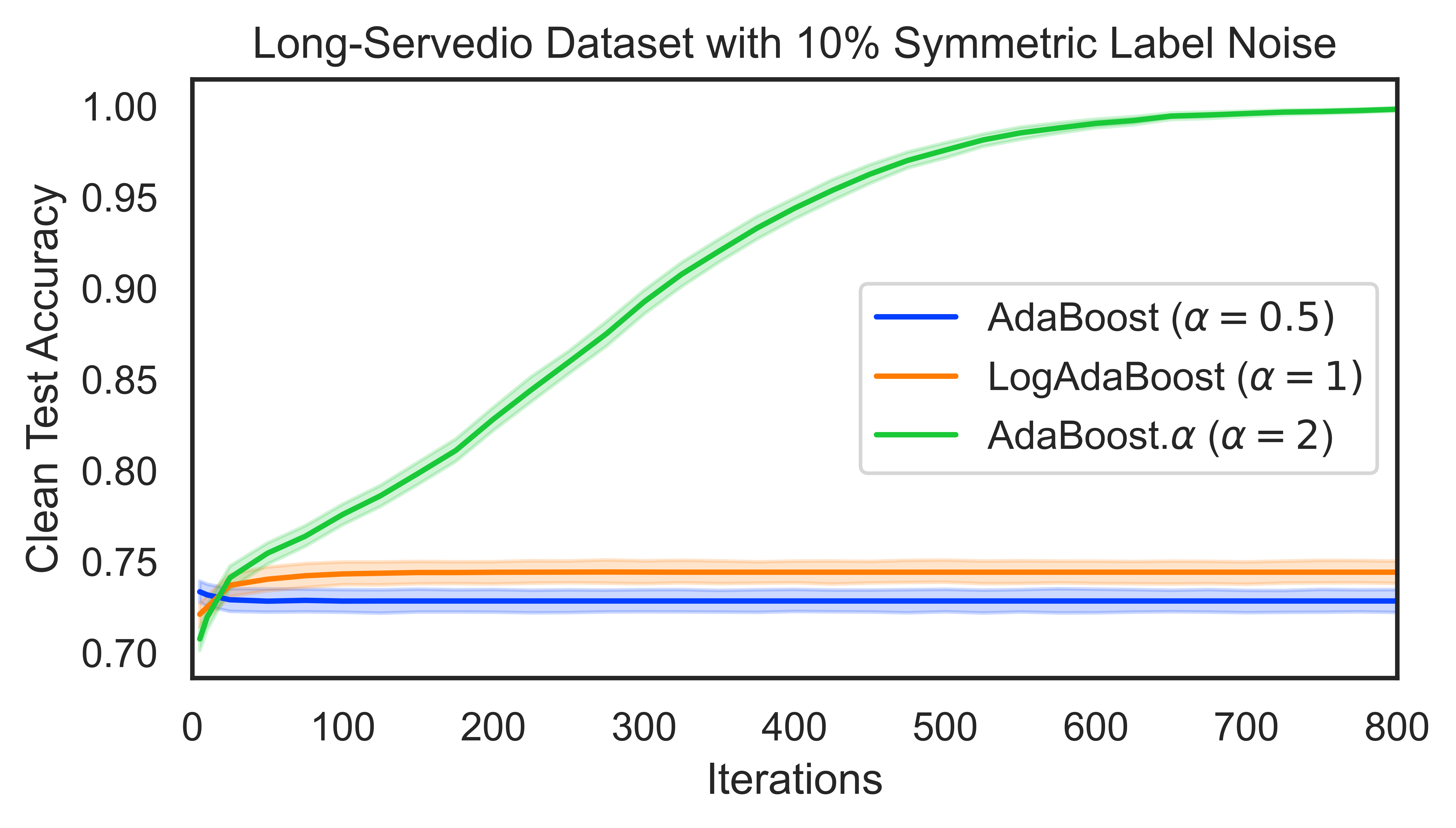}}
    \caption{Quasi-convex $\alpha$-loss booster ($\alpha = 2$) vs.~convex boosters ($\alpha \leq 1$) on decision stumps for the Long-Servedio dataset. Full version and details in Section~\ref{sec:experiments}.}
    \label{fig:ls_accuracy_iterations_shortversion}
\end{figure}
Recent work argues that the negative result of~\cite{long2010random} could perhaps be circumvented by increasing the complexity of the weak learner~\citep{nockLS2022}, however, there are certain benefits for utilizing simple models. 
Thus, one remaining degree of freedom to robustly train a simple model is by tuning the loss function itself. 
To this end, we use the recently introduced margin-based $\alpha$-loss, which smoothly tunes through the exponential ($\alpha = 1/2$), logistic ($\alpha = 1$), and sigmoid ($\alpha = \infty$) losses~\citep{sypherd2022journal}.
The $\alpha$ hyperparameter controls the convexity of the loss, since for $0 < \alpha \leq 1$ the loss is convex, and for $\alpha > 1$ the loss is quasi-convex.
We show that tuning $\alpha > 1$ allows the loss to ``give up'', which refers to how it evaluates large negative margins (preview Figure~\ref{Fig:marginalphalossplusder} and see the exponential vs.~sigmoid losses).
Hence, ``giving up'' on noisy training examples reduces the sensitivity of a simple hypothesis class (see Figure~\ref{fig:ls_accuracy_iterations_shortversion}). 
%
%
%
%


Our contributions are as follows:
\bignegspace
\begin{enumerate}
    \item In Theorem~\ref{thm:LSmarginalphalossrobust}, we show that there exist robust solutions of the margin-based $\alpha$-loss for $\alpha > 1$ to the problem of~\cite{long2010random}; we verify this result with simulation (Figure~\ref{Fig:LSclassificationlines}) and experimental results (Section~\ref{sec:boostingexperiments}), where we show increased gains when the maximum depth of the (decision/regression) tree weak learner is restricted, i.e., for simpler models.
    \negspace
    \negspace
    \item Building on the results in 1, we present a novel boosting algorithm (Algorithm~\ref{algo:adaboost.alpha} in Section~\ref{sec:adaboost.alpha}), called AdaBoost.$\alpha$, that may be of independent interest. 
    The novelty of AdaBoost.$\alpha$ is that it smoothly tunes through vanilla AdaBoost (minimizing the exponential loss, $\alpha = 1/2$), LogAdaBoost (minimizing the logistic loss, $\alpha = 1$)~\citep{schapire2013boosting}, to non-convex ``AdaBoost-type'' algorithms for $\alpha > 1$, all with the single $\alpha$ hyperparameter. 
    \negspace
    \negspace
    \item Noticing that the boosting setup of~\cite{long2010random} ultimately reduces to a two-dimensional linear problem, we theoretically demonstrate robustness of the margin-based $\alpha$-loss for $\alpha > 1$ under linear models of \textit{arbitrary} dimensions with an upperbound (Theorem~\ref{thm:taylorlagrangeupperbound}) and dominating terms also appearing in a lowerbound (Theorem~\ref{thm:lowerbound}).  
    In essence, we provide guarantees on the quality of optima, showing with upper and lower bounds on the noisy gradient that $\alpha > 1$ is better for ``good solutions'' than $\alpha \leq 1$.
    \negspace
    \negspace
    \item Finally, in Section~\ref{sec:logisticexperiments}, we report experimental results on the logistic model for a synthetic Gaussian Mixture Model (GMM) and a COVID-19 survey dataset~\citep{Salomone2111454118}. In particular, we show that $\alpha > 1$ is able to preserve the interpretability of the linear model for the COVID-19 data, while also providing robustness to label noise. In addition, we provide straightforward heuristics for tuning $\alpha$.
\end{enumerate}
\bignegspace



\subsection{Related Work} \label{sec:relatedwork}
%
\textbf{Convex and Non-Convex Losses}
While a small amount of carefully introduced label noise has been observed to improve the generalization capabilities of a model~\citep{pmlr-v108-li20e}, in general label noise during training is very detrimental to learning and thus represents an important problem for the community~\citep{frenay2013classification,10.1158/1055-9965.EPI-07-2629,10.1093/ntr/ntn010}.
In an effort to address this, many works propose reweighting/augmenting/regularizing/tuning \textit{convex} loss functions to train robust models~\citep{natarajan2013learning,ma2020normalized,liu2020peer,ghosh2017robust,patrini2017making,lee2006efficient,lin2017focal,leng2022polyloss}.
Other approaches include abstention~\citep{thulasidasan2019combating,ziyin2020learning,cortes2016boosting} and early stopping~\citep{bai2021understanding}, however, both techniques also typically revolve around a convex loss. 

Despite the fact that providing strong optimization guarantees for \textit{non-convex} losses is nontrivial~\citep{mei2018landscape}, 
non-convex loss functions (satisfying certain basic conditions, e.g., differentiability, classification-calibration~\citep{lin2004note,bartlett2006convexity}) have been observed to provide superior robustness over convex losses~\citep{beigman2009learning,manwani2013noise,nguyen2013algorithms,barron2019general,zhang2018generalized,zhao2010convex,sypherd2019tunable,chapelle2008tighter,wu2007robust,cheamanunkul2014non,masnadi2009design}.
Intuitively, non-convex loss functions seem to have a sophisticated regularization ability where they \textit{implicitly} assign less weight to misclassified training examples, and thus algorithms optimizing such losses are often less perturbed by outliers during training.
This contrasts with another set of approaches~\citep{lee2016label,yao2021instance,maas2019label,bootkrajang2012label,lee2016label} which seek to \textit{explicitly} estimate the noise transition matrix, sometimes requiring many parameters to do so.

\textbf{$\alpha$-loss}
The $\alpha$-loss, where $\alpha \in (0,\infty]$, arose in information theory~\citep{liao2018tunable,arimoto1971information}, and was recently introduced to ML~\citep{sypherd2019tunable}.
It smoothly tunes through several important losses (exponential for $\alpha = 1/2$, log for $\alpha = 1$, $0$-$1$ for $\alpha = \infty$), and has statistical, optimization, and generalization tradeoffs dependent on $\alpha$~\citep{sypherd2022journal}. 
Indeed, for shallow CNNs the $\alpha$-loss is more robust for $\alpha > 1$, however, the loss becomes increasingly more non-convex as $\alpha$ increases greater than $1$ (although there is a saturation effect), hence creating an optimization/robustness tradeoff~\citep{sypherd2020alpha}.
The $\alpha$-loss is equivalent (under appropriate hyperparameter restriction) to the Generalized Cross Entropy loss~\citep{zhang2018generalized}, which was motivated by the Box-Cox transformation in statistics~\citep{box1964analysis}.
Also, the $\alpha$-loss was recently shown to satisfy a statistical notion of robustness for loss functions in the class probability estimation setting~\citep{sypherdicml2022}.
More broadly, $\alpha$-loss and related quantities have been used in Generative Adversarial Networks~\citep{kurri2021realizing,kurri2022alphagan} and in robust Bayesian posterior estimation~\citep{zecchin2022robust}. 

\textbf{Convex and Non-Convex Boosting}
AdaBoost~\citep{FREUND1997119} (which minimizes the exponential loss~\citep{schapire2013boosting}) is the groundbreaking convex boosting algorithm. 
Later, the LogAdaBoost (which minimizes the logistic loss) was proposed as a more robust convex variant~\citep{collins2002logistic,mcdonald2003empirical}. 
Indeed, a SOTA boosting algorithm, XGBoost, minimizes (an approximated) logistic loss, rather than the exponential loss~\citep{chen2016xgboost}.~\cite{sypherdicml2022} recently introduced a novel boosting algorithm called PILBoost, which minimizes a \textit{convex} (proper) surrogate approximation of the $\alpha$-loss~\citep{nwLO,reid2010composite}, and presented experimental results on the robustness of PILBoost. 

However, the seminal work of~\cite{long2010random} showed that convex boosters provably suffer from label noise, particularly for simple weak learners~\citep{nockLS2022}.~\cite{van2015learning} proposed relaxing the nonnegativity condition of the convex loss in order to yield robustness, but it seems that this is unable to completely fix the problem~\citep{long2022perils}.
For this reason, non-convex boosting algorithms have been considered before~\citep{masnadi2009design,cheamanunkul2014non,miao2015rboost}, but there remains a large gap between the convex and non-convex realms.
Therefore, we propose directly using the margin-based $\alpha$-loss (rather than a ``proper'' convex approximation as in PILBoost), which smoothly tunes through several canonical convex and quasi-convex losses,
for boosting. 
Our AdaBoost.$\alpha$ (generalizing vanilla AdaBoost with $\alpha = 1/2$ and LogAdaBoost with $\alpha = 1$), ``gives up'' on noisy outliers during training (for $\alpha > 1$), thus allowing practitioners to continue using simpler models (e.g., for interpretability or energy efficiency) in noisy settings.


\section{PRELIMINARIES} \label{sec:prelim}

\subsection{Margin-Based $\alpha$-loss} \label{sec:marginbasedalphaloss}

We consider the setting of binary classification. 
The learner ideally wants to output a classifier $\overline{H}: \mathcal{X} \rightarrow \{-1,+1\}$ 
that  minimizes the probability of error, the expectation of the $0$-$1$ loss, however, this is NP-hard~\citep{ben2003difficulty}. 
Thus, the problem is relaxed by optimizing a surrogate to the $0$-$1$ loss over functions $H:\mathcal{X}\rightarrow \mathbb{R}$, whose output captures the certainty of prediction of the binary label $Y \in \{-1,1\}$ associated with the feature vector $X \in \mathcal{X}$~\citep{bartlett2006convexity}. 
The classifier is obtained by making a hard decision, i.e., $\overline{H}(X) = \text{sign}({H(X)})$. 
A surrogate loss is said to be margin-based if, the loss associated to a pair $(y,H(x))$ is given by $\tilde{l}(yH(x))$ for $\tilde{l}: \mathbb{R} \to\mathbb{R}_{+}$~\citep{lin2004note}. 
The loss of the pair $(y,H(x))$ only depends on the product $z:=yH(x)$, i.e., the (unnormalized) margin~\citep{schapire2013boosting}. A negative margin corresponds to a mismatch between the signs of $H(x)$ and $y$, i.e., a classification error by $H$; a positive margin corresponds to a correct classification by $H$.  

Since probabilities are typically the inputs to loss functions (e.g., log-loss, Matusita's loss~\citep{matusita1956decision}, $\alpha$-loss~\citep{sypherd2019tunable}), 
an important function we use is the sigmoid function $\sigma:\mathbb{R}\to[0,1]$, given by
\begin{equation}
\label{eq:DefSigmoid}
    \sigma(z) := \frac{1}{1+e^{-z}},
\end{equation}
where $z:=yH(x)$ is the margin.
The sigmoid smoothly maps real-valued predictions $H: \mathcal{X} \rightarrow \mathbb{R}$ to probabilities, and in the multiclass setting, the sigmoid is generalized by the softmax function~\citep{goodfellow2016deep}. 
Note that the inverse of $\sigma$ is the logit link~\citep{reid2010composite}.
Noticing that $\sigma(-z) = 1-\sigma(z)$, we have that 
\begin{align} \label{eq:derivativepropertysigmoid}
\sigma'(z) := \dfrac{d}{dz} \sigma(z) = \sigma(z)\sigma(-z) = \frac{e^{z}}{(1+e^{z})^{2}},
\end{align}
and note that $\sigma'$ is an even function.

\begin{figure} 
\centerline{\includegraphics[trim={3.75cm .65cm 3.5cm 1cm},clip,width=.75\linewidth]{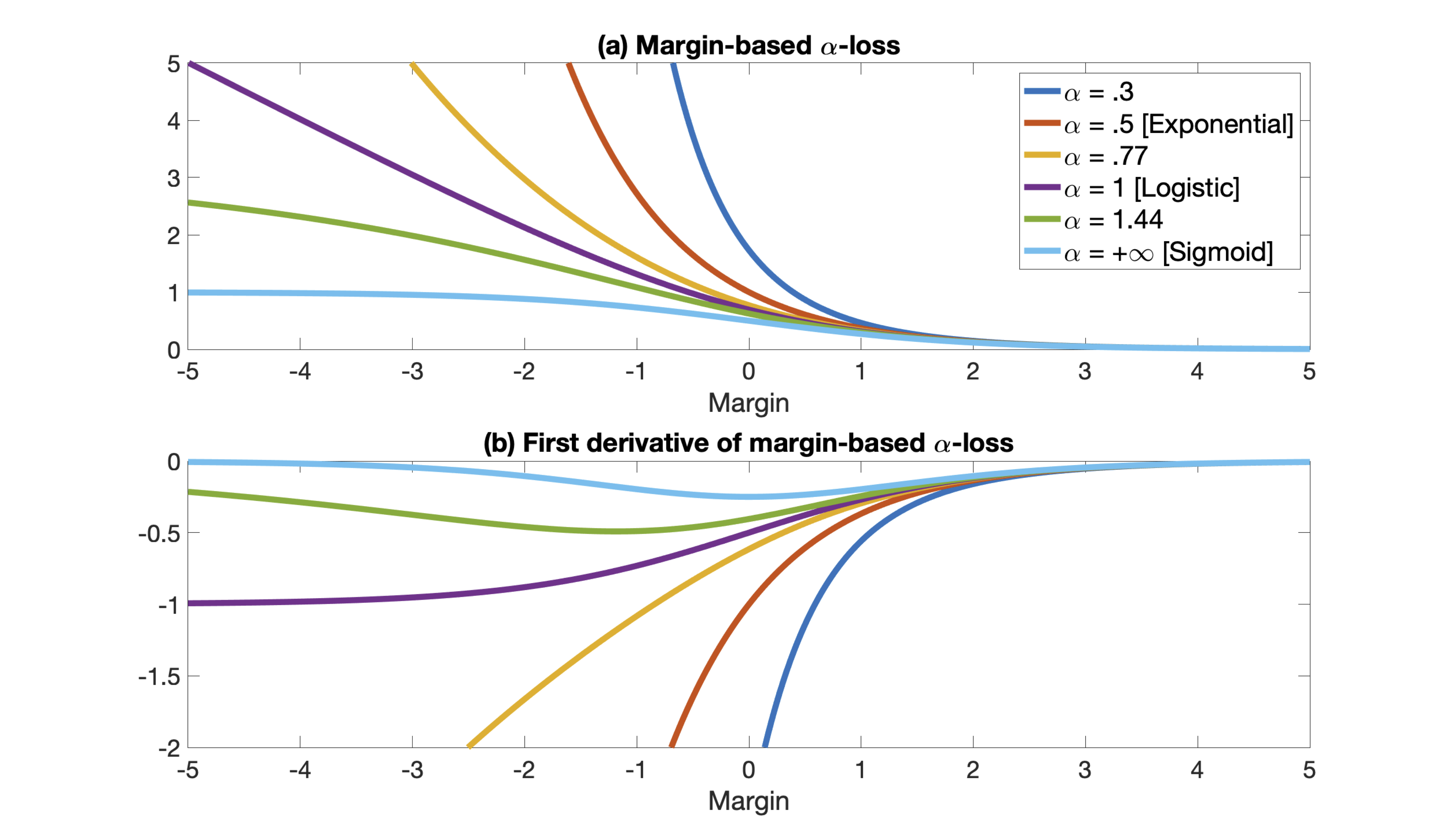}}
\caption{\textbf{(a)} Margin-based $\alpha$-loss~\eqref{eq:marginalphaloss} as a function of the margin ($z := yH(x)$) for $\alpha \in \{0.3,0.5,0.77,1,1.44,\infty\}$; \textbf{(b)} its first derivative (see Lemma~\ref{lem:derivativesmarginalphaloss} in Appendix~\ref{appendix:proofs}) with respect to the margin for the same set of $\alpha$. 
The ``giving up'' ability of the margin-based $\alpha$-loss for $\alpha > 1$ can be seen from its first derivative, where it is more constrained (than $\alpha \leq 1$) for large negative values of the margin.
}
\label{Fig:marginalphalossplusder}
\end{figure}

We now provide the definition of the margin-based $\alpha$-loss, which was first presented in~\citep{sypherd2019tunable} for $\alpha \in [1,\infty]$ and 
extended to $\alpha \in (0,\infty]$~\citep{sypherd2022journal}.
\begin{definition}
\label{def:MarginLoss}
The margin-based $\alpha$-loss, $\tilde{l}^{\alpha}: \mathbb{R} \to \mathbb{R}_{+}$, $\alpha \in (0,\infty]$, is given by, for $\alpha \in (0,1) \cup (1,\infty)$,
\begin{equation}
\label{eq:marginalphaloss}
\tilde{l}^{\alpha}(z) := \frac{\alpha}{\alpha - 1}\left(1 - \sigma(z)^{1 - 1/\alpha}\right),
\end{equation}
with $\tilde{l}^{1}(z) := -\log\sigma(z)$ and $\tilde{l}^{\infty}(z) := 1 - \sigma(z)$ by continuous extension, and note that $\tilde{l}^{1/2}(z) := e^{-z}$.
\end{definition}
Indeed, $\tilde{l}^{1/2}$, $\tilde{l}^{1}$, and $\tilde{l}^{\infty}$ recover the exponential (AdaBoost), logistic (logistic regression), and sigmoid (smooth $0$-$1$) losses, respectively~\citep{shalev2014understanding}; see Figure~\ref{Fig:marginalphalossplusder}(a) for a plot of $\tilde{l}^{\alpha}$ for several values of $\alpha$ versus the margin.
Note that for fixed $z \in \mathbb{R}$, $\tilde{l}^{\alpha}(z)$ is continuous in $\alpha$.~\cite{sypherd2022journal} showed that the margin-based $\alpha$-loss is classification-calibrated for all $\alpha \in (0,\infty]$~\citep{bartlett2006convexity}.
Thus, tuning the single $\alpha$ hyperparameter allows continuous interpolation through calibrated, important loss functions, however, different regimes of $\alpha$ have differing robustness properties.
To this end,~\cite{sypherd2022journal} presented the following result regarding the convexity characteristics of $\tilde{l}^{\alpha}$.
%
%
%
\begin{prop} \label{Prop:alpha-loss-convex}
$\tilde{l}^{\alpha}:\mathbb{R}\to\mathbb{R}_+$ is convex for $0 < \alpha \leq 1$ and quasi-convex for $\alpha > 1$. 
\end{prop}
Recall that a function $f: \mathbb{R} \rightarrow \mathbb{R}$ is quasi-convex if, for all $x,y \in \mathbb{R}$ and $\lambda \in [0,1]$, $f(\lambda x + (1-\lambda)y) \leq \max{\{f(x),f(y)\}}$, and also that any monotonic function is quasi-convex (cf.~\citep{boyd2004convex}). 

In light of Proposition~\ref{Prop:alpha-loss-convex}, consider Figure~\ref{Fig:marginalphalossplusder}(a) for $\alpha = 1/2$ (convex) and $\alpha = 1.44$ (quasi-convex), and suppose for concreteness that $z_{1} = -1$ and $z_{2} = -5$. 
The difference in loss evaluations for these two negative values of the margin, which are representative of misclassified training examples, is approximately exponential vs.~sub-linear; this is similarly observed in Figure~\ref{Fig:marginalphalossplusder}(b) with the first derivative of $\tilde{l}^{\alpha}$ (see Lemma~\ref{lem:derivativesmarginalphaloss} in Appendix~\ref{appendix:proofs}).
Intuitively, if a training example is not fit well by the currently learned parameter values, then its margin will be (large and) negative and it will incur more derivative update; if such a training example is noisy, convex losses (e.g., $\alpha \leq 1$) encourage the algorithm to continue fitting the bad example, whereas non-convex losses (e.g., $\alpha > 1$) would instead allow the algorithm to ``give up''.
This tendency of convex losses could be exacerbated for simpler models because they can suffer significant perturbation by label noise (preview Figure~\ref{Fig:LSclassificationlines}) vs.~more nuanced function classes~\citep{rolnick2017deep}.  

\subsection{Boosting Setup}
For the boosting context,
we assume access to a training sample $\mathcal{S} := \{(x^i, y^i), i \in [m]\} \subset \mathcal{X} \times \{-1,+1\}$ of $m$ examples, where $[m] := \{1, 2, ..., m\}$. 
Following the functional gradient perspective of boosting (i.e., the blueprint of~\citep{fGFA}), 
the boosting algorithm minimizes a margin-based loss $\tilde{l}$ with respect to $\mathcal{S}$ over $t \in [T]$ iterations in order to learn a function $H_{T}: \mathcal{X} \rightarrow \mathbb{R}$, given by
\begin{align} \label{eq:boostingclassificationfunction}
H_{T}(\cdot) & := \sum_{t \in [T]} \theta_{t} h_{t}(\cdot),
\end{align}
where $\theta_{t}$ are the learned parameters and the $h_{t}: \mathcal{X} \rightarrow \mathbb{R}$ are weak learners with slightly better than random classification accuracy.
On each iteration $t \in [T]$, we compute weights for each training example using the full $H_{t-1}$ via
\begin{align} \label{eq:boostingweights}
D_{t}(i) := -\tilde{l}'(y^i H_{t-1}(x^i)), \forall i \in [m]. 
\end{align}
The weights $D_{t}(i)$ are non-negative, normalized to form a distribution over the training examples, and tend to increase for an example that is incorrectly predicted (negative margin) by the previously learned $H_{t-1}$.
Thus, weighting puts emphasis on ``hard'' examples using the first derivative of the loss function, which is a kind of functional gradient descent (cf.~\citep{schapire2013boosting}). 
Then, the distribution over training examples $D_{t}$ is passed to the weak learning oracle (see Algorithm~\ref{algo:adaboost.alpha} for the general procedure).

In the next section, we show that using the derivative of the margin-based $\alpha$-loss in~\eqref{eq:boostingweights} recovers a novel robust boosting algorithm, which may be of independent interest. 
We also show that this algorithm has provable robustness guarantees on the negative result of~\cite{long2010random}.


\section{ROBUSTNESS FOR BOOSTING} \label{sec:boosting}

\subsection{AdaBoost.$\alpha$: Boosting with a Give Up Option} \label{sec:adaboost.alpha}

\begin{algorithm} 
\caption{AdaBoost.$\alpha$} 
\label{algo:adaboost.alpha}
\begin{algorithmic}[1]
\STATE \textbf{Given:} $(x^{1},y^{1}), \ldots, (x^{m},y^{m})$ where $x^{i} \in \mathcal{X}$, $y^{i} \in \{-1,+1\}$, and $\alpha \in (0,\infty]$
\STATE Initialize: $H_{0} = 0$.
\STATE \textbf{for} {$t = 1, 2, \ldots, T$:}
\STATE \hspace{0.5cm} Update, for $i = 1, \ldots, m$:
\begin{equation} \label{eq:adaboostalpha}
 D_{t}(i) = \frac{\tcbhighmath[fuzzy halo= 0mm with blue!50!white,arc=2pt,
  boxrule=2pt,frame hidden]{ \sigma'(y^{i}H_{t-1}(x^{i})) \sigma(y^{i}H_{t-1}(x^{i}))^{-\frac{1}{\alpha}}}}{\mathcal{Z}_{t}}, 
%
%
\end{equation}
\hspace{0.5cm} where $\mathcal{Z}_{t}$ is a normalization factor.
\STATE \hspace{0.5cm} Return $h_{t}$, weakly learned on $D_{t}$.
\STATE \hspace{0.5cm} Compute error of weak hypothesis $h_{t}$:
\begin{align}
\epsilon_{t} = \sum_{i: h_{t}(x^{i}) \neq y^{i}} D_{t}(i). 
\end{align}
\STATE \hspace{0.5cm} Let $\theta_{t}  = \frac{1}{2}\log\left( \frac{1 - \epsilon_t}{\epsilon_t} \right)$. 
\STATE \hspace{0.5cm} Update: $H_{t} = H_{t-1} + \theta_{t}h_{t}$
\STATE \textbf{Return} $\overline{H}(\cdot) = \text{sign} \left(H_{T}(\cdot) \right)$
\end{algorithmic}
\end{algorithm}

Using the smooth tuning of the margin-based $\alpha$-loss, we present a novel robust boosting 
algorithm, AdaBoost.$\alpha$ in Algorithm~\ref{algo:adaboost.alpha}, which is obtained by noticing (from the functional gradient perspective~\citep{schapire2013boosting}) that the exponential weighting of vanilla AdaBoost is really the negative first derivative of the exponential loss (i.e., $\alpha = 1/2$).
Generalizing this observation for all $\alpha \in (0,\infty]$ (via Lemma~\ref{lem:derivativesmarginalphaloss} in Appendix~\ref{appendix:proofs}) in~\eqref{eq:adaboostalpha}, we obtain a hyperparameterized family of ``AdaBoost-type'' algorithms.

Indeed, AdaBoost.$\alpha$ also recovers LogAdaBoost (see Section~\ref{sec:relatedwork}) for $\alpha = 1$. 
For $\alpha > 1$, AdaBoost.$\alpha$ becomes a non-convex boosting algorithm minimizing the quasi-convex margin-based $\alpha$-losses (Proposition~\ref{Prop:alpha-loss-convex}).
As argued in Section~\ref{sec:marginbasedalphaloss}, non-convex losses enable the boosting algorithm to give up on noisy examples, and hence yield a more robust model $H_{T}$.
Indeed, for these same robustness reasons, non-convex boosting algorithms have been considered before (see Section~\ref{sec:relatedwork}).
However, the novelty of AdaBoost.$\alpha$ is that it continuously interpolates through convex AdaBoost variants ($\alpha \leq 1$) to non-convex ``AdaBoost-type'' algorithms ($\alpha > 1$).
Thus, AdaBoost.$\alpha$ allows the practitioner or meta-algorithm~\citep{he2021automl} to tune how much one would like the algorithm to give up on hard, possible noisy, training examples, which may be useful in a distributed context~\citep{cooper2017improved}. 
Lastly, we note that because of the modularity of the $\alpha$ hyperparameter generalization, a multiclass extension of AdaBoost.$\alpha$ readily follows from standard approaches of multiclass AdaBoost, e.g.,~\cite{hastie2009multi}. 


\subsection{Robustness on the Long-Servedio Dataset} \label{sec:lsdataset}

\begin{figure}[h] 
\centerline{\includegraphics[trim={0.4cm .2cm 1.2cm 0.2cm},clip,width=.75\linewidth]{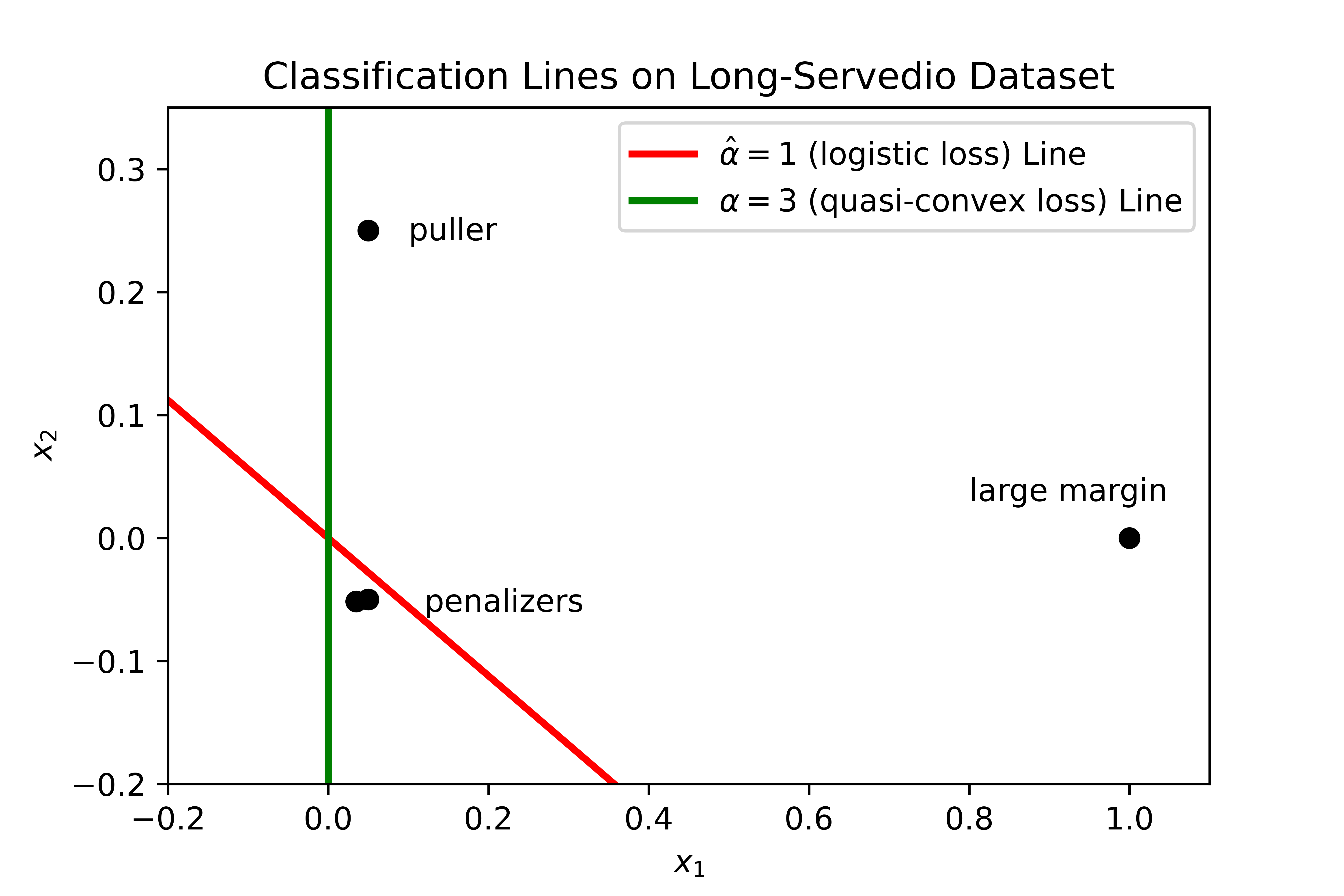}}
\caption{
A plot depicting optimal classification lines of $\hat{\alpha} = 1$ and $\alpha = 3$ for the \textit{clean} Long-Servedio dataset $\mathcal{S}$, where the penalizer examples are slightly separated for display.
The $\hat{\alpha}, \alpha$ optima are obtained by grid-search on the noisy Long-Servedio dataset $\hat{\mathcal{S}}$, where the noise level is chosen as $p = 1/3$, and $\gamma_{\hat{\alpha}} = 1/20$ is subsequently chosen for the negative result of~\cite{long2010random} to ``kick-in''. 
The $\hat{\alpha} = 1$ (logistic loss) line (red) is given by $(\theta_{1}^{\hat{\alpha}},\theta_{2}^{\hat{\alpha}}) = (0.79, 1.41)$ for~\eqref{eq:lsclassifcationfunction}, and has fair coin accuracy on $\mathcal{S}$, misclassifying both penalizers. 
The $\alpha = 3$ (quasi-convex loss) line (green) is given by $(\theta_{1}^{\alpha},\theta_{2}^{\alpha}) = (41.59, -1.19\times 10^{-11})$, and has perfect accuracy on $\mathcal{S}$.
This simulation aligns with Theorem~\ref{thm:LSmarginalphalossrobust} in that the quasi-convex $\alpha = 3$ loss is able to ``give up'' on the noisy copies of the training examples and recover perfect classification parameters.
More $\alpha$'s are presented in Appendix~\ref{appendix:proofofLStheorem}.
}
\bignegspace
\label{Fig:LSclassificationlines}
\end{figure}

In~\cite{long2010random}, 
the training sample $\mathcal{S}$ is a multiset consisting of three distinct examples, one of which is repeated twice, where the data margin $0 < \gamma < 1/6$:
\bignegspace
\begin{itemize}
    \item $\mathcal{S}$ contains one copy of the example $x = (1,0)$ with label $y = +1$. (Called the ``large margin'' example.)
    \item $\mathcal{S}$ contains two copies of the example $x = (\gamma,-\gamma)$ with label $y = +1$. (Called the ``penalizers'' since these are the points that the booster will misclassify.) 
    \item $\mathcal{S}$ contains one copy of the example $x = (\gamma, 5\gamma)$ with label $y = +1$. (Called the ``puller''.)
\end{itemize}
\bignegspace

Thus, all four examples in $\mathcal{S}$ have positive label and lie in the unit disc $\{x:\|x\|\leq 1\}$; see Figure~\ref{Fig:LSclassificationlines} for a plot of the dataset. 
Notice that $\overline{H}(x) = \text{sign}(x_{1})$ (sign of first coordinate of $x$) corrrectly classifies all four examples in $\mathcal{S}$ with margin $\gamma > 0$, so the weak learner hypothesis class $\mathcal{H} = \{h_{1}(x) = x_{1}, h_{2}(x) = x_{2}\}$ is sufficient for perfect classification of the dataset.
The task for the boosting algorithm is to learn parameters $(\theta_{1},\theta_{2})$ such that, from~\eqref{eq:boostingclassificationfunction}, 
\begin{align} \label{eq:lsclassifcationfunction}
H_{\tilde{l},\gamma}(x_{1},x_{2}) := \theta_{1} x_{1} + \theta_{2} x_{2},   
\end{align}
achieves perfect classification accuracy on $\mathcal{S}$, where the dependency on the loss $\tilde{l}$ and data margin $\gamma$ is clear.
Note that~\eqref{eq:lsclassifcationfunction} (we abbreviate $H_{\tilde{l},\gamma} = (\theta_{1},\theta_{2})$) is a 2D linear model, so this setup parallels with logistic regression, which we consider in the sequel.
%
Following~\citep{nockLS2022}, we obtain a noisy sample $\hat{\mathcal{S}}$ with label flip probability $0 < p < 1/2$ by including $p^{-1} - 1$ copies of $\mathcal{S}$ and $1$ copy of $\mathcal{S}$ with the labels flipped.~\cite{long2010random} showed that for any calibrated, \textit{convex} loss $\tilde{l}$:
\begin{itemize}
\bignegspace
    \item When $p = 0$, i.e., the training sample is $\mathcal{S}$, the optimal $H_{\tilde{l}} = (\theta_{1}^{\tilde{l}},\theta_{2}^{\tilde{l}})$ of $\tilde{l}$ has \textit{perfect} accuracy on $\mathcal{S}$.
    \item For any $0 < p < 1/2$ generating training sample $\hat{\mathcal{S}}$, there exists $0 < \gamma_{\tilde{l}} < 1/6$ such that the optimal $H_{\tilde{l},\gamma_{\tilde{l}}} = (\theta_{1}^{\tilde{l}},\theta_{2}^{\tilde{l}})$ of $\tilde{l}$ has \textit{fair coin} accuracy on $\mathcal{S}$.
\bignegspace
\end{itemize}
Intuitively, the interplay between the ``large margin'' and ``puller'' examples forces a convex booster, boosting $\mathcal{H}$, to try to fit the noisy examples in $\hat{\mathcal{S}}$; this holds even if the booster is regularized or stopped early, ultimately outputing a model that misclassifies both ``penalizers'' of $\mathcal{S}$~\citep{long2010random}.
Taking stock with $\tilde{l}^{\alpha}$, we see that this pathology holds for $\alpha \leq 1$, since these are convex losses. 
However, tuning $\alpha > 1$ to quasi-convex losses is able to induce the existence of optima which can fix the problem.
\begin{theorem} \label{thm:LSmarginalphalossrobust}
Let $0 < p < 1/2$ for $\hat{\mathcal{S}}$, and $\hat{\alpha} \leq 1$ for $\tilde{l}^{\alpha}$. 
By~\cite{long2010random}, there exists $0 < \gamma_{\hat{\alpha}} < 1/6$ such that the optimal $H_{\hat{\alpha},\gamma_{\hat{\alpha}}} = (\theta_{1}^{\hat{\alpha}}, \theta_{2}^{\hat{\alpha}})$ is a fair coin on $\mathcal{S}$. 
On the other hand, for $\alpha \in (1,\infty)$, $\tilde{l}^{\alpha}$ has optimum $H_{\alpha,\gamma_{\hat{\alpha}}} = (\theta_{1}^{\alpha}, \theta_{2}^{\alpha})$, where $\theta_{1}^{\alpha} = \mathcal{O}\left(\alpha \gamma_{\hat{\alpha}}^{-1} \log{\left(p^{-1} - 1 \right)} \right)$ and $\theta_{2}^{\alpha} = 0$, with perfect classification accuracy on $\mathcal{S}$.
\end{theorem}
The proof of Theorem~\ref{thm:LSmarginalphalossrobust} (in Appendix~\ref{appendix:proofofLStheorem}) is nontrivial since $\alpha > 1$ has a non-convex optimization landscape.
In Figure~\ref{Fig:LSclassificationlines} where $p = 1/3$ and $\gamma_{\hat{\alpha}} = 1/20$, the grid search returns $(\theta_{1}^{\alpha},\theta_{2}^{\alpha}) = (41.59, -1.19\times 10^{-11})$, which aligns with Theorem~\ref{thm:LSmarginalphalossrobust}, namely that $\theta_{1}^{\alpha} \approx 3\times 20 \times \log{(2)} \approx 41.59$ and $\theta_{2}^{\alpha} \approx 0$.
Intuitively, increasing $\alpha \in (1,\infty)$ increases $\theta_{1}^{\alpha}$, which may have practical utility (see Section~\ref{sec:boostingexperiments}), but the rate for $\theta_{1}^{\alpha}$ hints at why $\alpha = \infty$ is not included, since $\alpha = \infty$ ``pushes'' $\theta_{1}^{\alpha}$ to $\infty$, an impossibility; this is an example of the robustness/optimization complexity tradeoff inherent in the margin-based $\alpha$-loss~\citep{sypherd2020alpha}.



\section{ROBUSTNESS FOR LINEAR MODELS} \label{sec:logisticmodel}

Taking inspiration from the boosting setup in Section~\ref{sec:lsdataset}, where the weak learner recovered a 2D linear model in~\eqref{eq:lsclassifcationfunction},
we now consider a generalization of the 2D linear hypothesis class to $d \in \mathbb{N}$ dimensions, which in binary classification is equivalent to the logistic model~\citep{sypherd2022journal}.
Ideally, one would like to give \textit{direct} expressions of gradient optimizers $\hat{\theta}^{\alpha}$ as we do for the Long-Servedio setup in Theorem~\ref{thm:LSmarginalphalossrobust}, however, the logistic model has sigmoid 
non-linearities that make this difficult for general data distributions. 
Instead, we take an \textit{indirect} approach where we provide guarantees on the quality of gradient optima, 
showing with upper and lower bounds that the noisy gradient for $\alpha > 1$ is smaller for ``good solutions'' than when $\alpha = 1$ (logistic regression).
Thus, the motivation for Theorems~\ref{thm:taylorlagrangeupperbound} and~\ref{thm:lowerbound} is to argue that a gradient optimizer
is more likely to converge near a ``good solution'' when $\alpha > 1$ than when $\alpha = 1$; indeed, this is another way to view
how the $\alpha > 1$ ``give up'' on the noise in the training data.

We let $X \in [0,1]^{d}$ be the normalized feature vector, $Y \in \{-1,+1\}$ the label, and we assume that the pair is drawn according to an unknown distribution $P_{X,Y}$.
We assume that the parameter vector $\theta \in \mathbb{B}_{d}(r)$ where $r>0$ and $\mathbb{B}_{d}(r) := \{\theta\in\mathbb{R}^d : \|\theta\|_{2} \leq r\}$. 
Thus, in this setting $\langle yx, \theta \rangle$ (inner product) is the margin, 
and note by the Cauchy-Schwarz inequality that $\langle yx, \theta \rangle \leq r\sqrt{d}$.



For $\alpha \in (0,\infty]$, the expected margin-based $\alpha$-loss, abbreviated the $\alpha$-risk, evaluated at $\theta \in \mathbb{B}_{d}(r)$ is given by 
\begin{align} \label{eq:cleanalpharisklogisticmodel}
R_{\alpha}(\theta) := \mathbb{E}_{X,Y}\left[\tilde{l}^{\alpha}(\langle YX, \theta \rangle) \right],
\end{align}
and for symmetric label noise rate $0 < p < 1/2$,
\begin{align} \label{eq:noisyalpharisklogisticmodel}
R_{\alpha}^{p}(\theta) := \mathbb{E}_{X,Y}\left[ \mathbb{E}_{\tau \sim \text{Rad}(p)} \left( \tilde{l}^{\alpha}(\langle -\tau YX, \theta \rangle) \right) \right], 
\end{align}
is called the noisy $\alpha$-risk, where $\tau$ is a Rademacher random variable with parameter $p$.
In order to assess the efficacy of a given parameter vector $\theta \in \mathbb{B}_{d}(r)$, we are interested in the gradient of the loss function, due to the use of gradient methods for optimization~\citep{boyd2004convex}. 
Thus, the gradient of the $\alpha$-risk in~\eqref{eq:cleanalpharisklogisticmodel} is 
\begin{align} \label{eq:gradientofmarginbasedalphaloss}
\nabla_{\theta} R_{\alpha}(\theta) := \mathbb{E}_{X,Y}\left[\nabla_{\theta} \tilde{l}^{\alpha}(\langle YX, \theta \rangle) \right], 
\end{align}
$\nabla_{\theta} \tilde{l}^{\alpha}(\langle YX, \theta \rangle) := - \sigma'(\langle YX, \theta \rangle) \sigma(\langle YX, \theta \rangle)^{-\frac{1}{\alpha}} Y X$ for $\alpha \in (0,\infty]$ from Lemma~\ref{lem:derivativesmarginalphaloss} in Appendix~\ref{appendix:proofs}.
%
%
Hence, the gradient of the noisy $\alpha$-risk~\eqref{eq:noisyalpharisklogisticmodel} is given by
\begin{align} \label{eq:noisygradientalpharisklogisticmodel}
&\nabla_{\theta} R_{\alpha}^{p}(\theta) := \mathbb{E}_{X,Y}\left[\mathbb{E}_{\tau \sim \text{Rad}(p)}\left(\nabla_{\theta} \tilde{l}^{\alpha}(\langle -\tau YX, \theta \rangle)  \right) \right]. 
\end{align}
\negspace
We now present a result in the realizable setting, indicating~\eqref{eq:noisygradientalpharisklogisticmodel} is smaller for $\alpha = \infty$ (soft $0$-$1$ loss) at any data generating vector $\theta^{*} \in \mathbb{B}_{d}(r)$ than for $\alpha = 1$ (logistic loss).
%
%
\begin{theorem} \label{thm:taylorlagrangeupperbound}
Let $0 < p < 1/2$ and let $\hat{\theta}^{1},\hat{\theta}^{\infty} \in \mathbb{B}_{d}(r)$ be such that $\nabla_{\theta} R_{1}^{p}(\hat{\theta}^{1}) = \mathbf{0} = \nabla_{\theta} R_{\infty}^{p}(\hat{\theta}^{\infty})$.
We assume that the following holds for all $(x,y) \in \mathcal{X} \times \{-1,+1\}$, 
%
\begin{align} \label{eq:taylorlagrangeupperboundassumption}
\langle yx, \hat{\theta}^{\infty} \rangle \geq \langle yx, \hat{\theta}^{1} \rangle > \ln{(2+\sqrt{3})}.
\end{align}
If for any $\theta^{*} \in \mathbb{B}_{d}(r)$ we have $\langle yx, \theta^{*} \rangle \geq \langle yx, \hat{\theta}^{\infty} \rangle$ for all $(x,y) \in \mathcal{X} \times \{-1,+1\}$, then we have that for $\alpha \in \{1,\infty\}$,
\begin{align} \label{eq:taylorlagrangeinfinitynormalpha}
\frac{\|\nabla_{\theta} R_{\alpha}^{p}(\theta^{*}) \|_{\infty}}{C_{\alpha}} \leq  d^{\frac{1}{2}} r \left|\tilde{l}^{\alpha''}(z^{*}_{\alpha}) \right| +  d r^{2} \left| \tilde{l}^{\alpha'''}(z_{\alpha}^{*}) \right|,
\end{align}
where $C_{\alpha} = 2$ for $\alpha = 1$ and $C_{\alpha} = 2-4p$ for $\alpha = \infty$, and
$z_{\alpha}^{*} := \argmax_{z \in \{\langle yx, \hat{\theta}^{\alpha} \rangle \}} \left|\tilde{l}^{\alpha''}(z) \right|$.
Furthermore, 
\begin{align} \label{eq:taylorlagrangeupperboundsinequality}
1-2p < \frac{d^{\frac{1}{2}} r \left|\tilde{l}^{1''}(z^{*}_{1}) \right| + d r^{2} \left| \tilde{l}^{1'''}(z^{*}_{1}) \right|}{d^{\frac{1}{2}} r \left|\tilde{l}^{\infty''}(z_{\infty}^{*}) \right| + d r^{2} \left| \tilde{l}^{\infty'''}(z_{\infty}^{*}) \right|}. 
\end{align}
\end{theorem}
\negspace
Theorem~\ref{thm:taylorlagrangeupperbound} uses symmetries of the first derivative of $\tilde{l}^{\alpha}$ for $\alpha \in \{1,\infty \}$; see Appendix~\ref{appendix:taylorlagrangeupperbound} for proof details.
Intuitively,~\eqref{eq:taylorlagrangeupperboundsinequality} indicates that there is a significant discrepancy between the two upper bounds as the noise rate $p \rightarrow 1/2$, suggesting that $\nabla_{\theta} R_{\alpha}^{p} (\cdot)$ is smaller at any data generating $\theta^{*}$ for $\alpha = \infty$ than for $\alpha = 1$ (logistic regression).
Note that the assumption in~\eqref{eq:taylorlagrangeupperboundassumption} is mild because \textit{both} vectors (rather than just $\hat{\theta}^{\infty}$) are assumed to achieve perfect accuracy on the clean data distribution.

 
In support of the upper bounds in Theorem~\ref{thm:taylorlagrangeupperbound}, 
we now present a uniform lower bound on the norm of~\eqref{eq:noisygradientalpharisklogisticmodel} for the skew-symmetric family of distributions (e.g., GMMs). 
\begin{theorem} \label{thm:lowerbound}
Let $0 < p < 1/2$, and for each $y\in\{-1,1\}$, let $X^{[y]}$ have the distribution of $X$ conditioned on $Y=y$.
We assume a skew-symmetric distribution, namely, that $X^{[1]} \stackrel{\textnormal{d}}{=} -X^{[-1]}$, and $\mathbb{E}[X^{[1]}]\neq \mathbf{0}$.
We also assume that $r> 0$ is small enough such that both of the following hold:
\begin{align} \label{eq:lowerboundassumps0}
(1-2p)(1-\sigma'(r \sqrt{d})) < \frac{\|\mathbb{E}(X^{[1]})\|_{2}}{\mathbb{E}(\|X^{[1]}\|_{2})},
\end{align}
and, for all $\alpha \in [1,\infty]$,
\begin{align} \label{eq:lowerboundassumps1}
e^{\frac{r \sqrt{d}}{\alpha}} \log{(e^{r \sqrt{d}} + 1)} < \left(p^{-1} - 1 \right) \log{(e^{-r \sqrt{d}} + 1)}.
\end{align}
Then, we have that for every $\theta \in \mathbb{B}_{d}(r)$, 
\begin{align} \label{eq:lowerboundresult}
\|\nabla_{\theta} R_{\alpha}^{p}(\theta)\|_{2} \geq \|\mathbb{E}[X^{[1]}]\|_{2} - \chi \mathbb{E}[\|X^{[1]}\|_{2}] > 0,
\end{align}
where (letting $\tilde{\chi} := \sigma(r\sqrt{d})^{1-\frac{1}{\alpha}} \sigma(-r\sqrt{d}) - 1$)
\begin{align} \label{eq:cerainlowerbound}
\chi := \begin{cases}
\sigma(r\sqrt{d}) - p & \alpha = 1 \\
p \tilde{\chi} - (1-p) \tilde{\chi} & \alpha \in (1,\infty) \\
(1-2p)(1-\sigma'(r\sqrt{d})) & \alpha = \infty,
\end{cases}
\end{align}
and $\chi$ is monotonically increasing in $\alpha \in [1,\infty]$. 
\end{theorem}
The proof of Theorem~\ref{thm:lowerbound} (in Appendix~\ref{appendix:uniformlowerbound}) is inspired by the Morse landscape analysis in~\citep{sypherd2019tunable}. 
Intuitively,~\eqref{eq:cerainlowerbound} implies that the RHS in~\eqref{eq:lowerboundresult} is monotonically \textit{decreasing} in $\alpha \in [1,\infty]$,
which aligns with the ordering given by the upper bounds in Theorem~\ref{thm:taylorlagrangeupperbound}.
Regarding the assumptions in~\eqref{eq:lowerboundassumps0} and~\eqref{eq:lowerboundassumps1}, they are both more easily satisfied for smaller $r > 0$, indicating alignment with the underlying optimization landscape phenomena. 
%
Taken together, Theorems~\ref{thm:taylorlagrangeupperbound} and~\ref{thm:lowerbound} suggest that larger $\alpha > 1$ are more robust than $\alpha = 1$ (logistic regression); also, notice the $1-2p$ coefficient for $\alpha = \infty$ appearing in both bounds.

\section{EXPERIMENTS} \label{sec:experiments}

We now provide empirical results in support of the previous sections, namely the efficacy of AdaBoost.$\alpha$ (Algorithm~\ref{algo:adaboost.alpha}) on the Long-Servedio dataset and the robustness of the margin-based $\alpha$-loss (Definition~\ref{def:MarginLoss}) in linear models, both for $\alpha > 1$.
Further details and results are in Appendix~\ref{appendix:experiments}.

\subsection{Boosting} \label{sec:boostingexperiments}


For the boosting experiments, we utilize the \textit{experiment version} of the Long-Servedio dataset~\citep{long2010random,cheamanunkul2014non},
where the feature vectors are 21D,
which differs from the theory version presented in Section~\ref{sec:lsdataset}, where the feature vectors are 2D. 
A full description of the dataset is presented in Appendix~\ref{appendix:boostingexperiments}.
We introduce symmetric label noise in the training data with flip probability $0 < p < 1/2$.

\textbf{Robustness for simple models} In Figure~\ref{fig:ls_accuracy_depth}, we report results of AdaBoost.$\alpha$ with $\alpha > 1$ (quasi-convex) vs. SOTA convex boosters: vanilla AdaBoost (AdaBoost.$\alpha$ with $\alpha = 1/2$), LogAdaBoost (AdaBoost.$\alpha$ with $\alpha = 1$), XGBoost, and PILBoost (see Section~\ref{sec:relatedwork}).
For lower maximum tree depth\footnote{Increasing the maximum tree depth exponentially increases the number of parameters for the weak learner, which impacts energy consumption, interpretability, and generalization (e.g., via VC dimension).} of the weak learner (i.e., simpler models), $\alpha > 1$ boosters are better able to ``give up'' on the noisy labels during training and the learned model yields better accuracy on the \textit{clean} test set, aligning with Theorem~\ref{thm:LSmarginalphalossrobust}. 
When the maximum depth is increased, all of the algorithms perform roughly the same~\citep{nockLS2022}.


\begin{figure}[h]
    \centerline{\includegraphics[trim={0.2cm .2cm .05cm 0.2cm},clip,width=.75\linewidth]{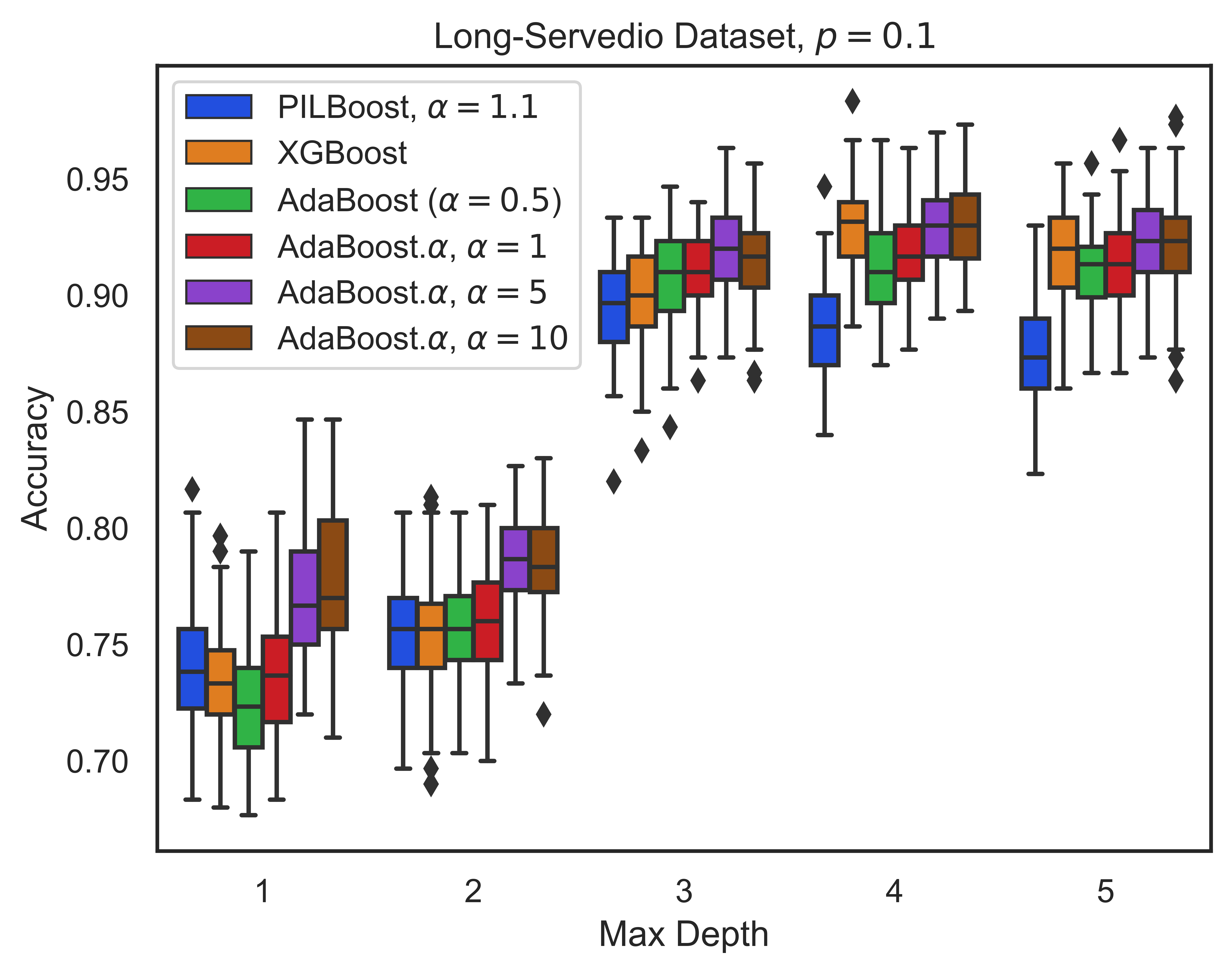}}
    \caption{
    Box and whisker plots of the clean test accuracies of several boosters with $100$ decision trees of varying maximum depth on the Long-Servedio dataset for $p = 0.1$ symmetric label noise.  
    The boxes are the interquartile ranges, the lines in the boxes are the medians, and the diamonds are the outliers.
    Note that AdaBoost.$\alpha$ with $\alpha >1$ (quasi-convex), outperforms the convex boosters when the maximum depth is $1$ or $2$.
    Further commentary is in Section~\ref{sec:boostingexperiments}, and more noise levels are in Appendix~\ref{appendix:boostingexperiments}.
    }
    \label{fig:ls_accuracy_depth}
    \bignegspace
\end{figure}

\textbf{Giving up} 
In Figure~\ref{fig:ls_accuracy_iterations}, we plot the clean test accuracy of AdaBoost.$\alpha$ boosting decision stumps for several values of $\alpha$ versus iterations (i.e., number of weak learners). 
We see that for $\alpha \leq 1$, increasing iterations does not increase accuracy; however, the $\alpha > 1$ (non-convex) boosters continue ``giving up'' on the noisy training examples, resulting in a $\approx 25\%$ gain.
For the large $\alpha > 1$, i.e. $\alpha = 8$ or $20$, the confidence intervals widen, which is an example of the robustness/non-convexity tradeoff inherent in the $\alpha$ hyperparameter~\citep{sypherd2020alpha}. 


\begin{figure}[h]
    \centerline{\includegraphics[trim={0.2cm .2cm .6cm 0.2cm},clip,width=.75\linewidth]{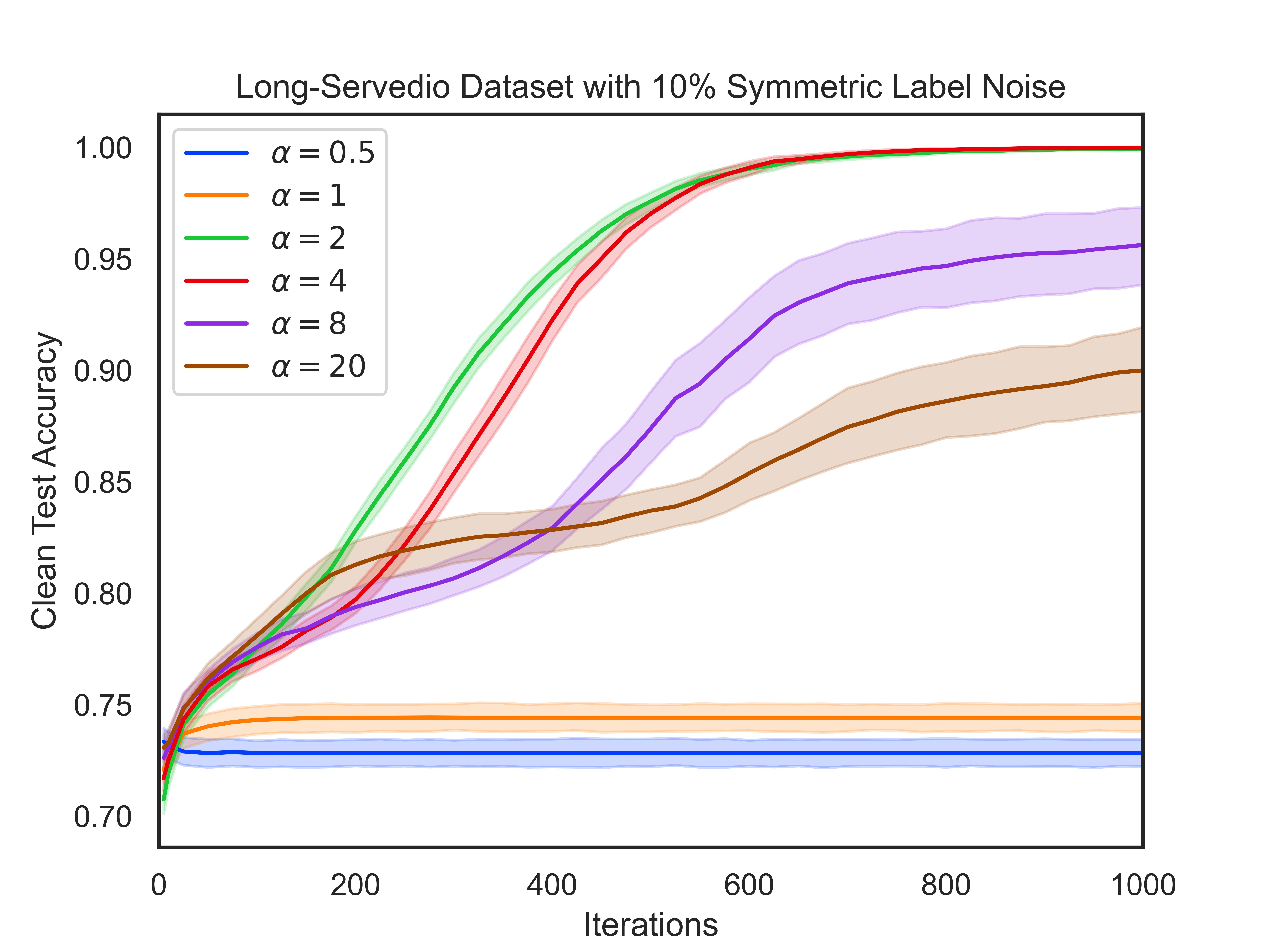}}
    \caption{
    Full version of Figure~\ref{fig:ls_accuracy_iterations_shortversion}, where we plot clean test accuracies vs.~the number of iterations of AdaBoost.$\alpha$ boosting decision stumps for several values of $\alpha$ on the Long-Servedio dataset with $p = 0.1$ symmetric label noise. 
    Note that the solid curves correspond to mean accuracy and shaded areas are the associated $95\%$ confidence intervals (from $80$ runs of the experiment).
    This result reflects the tendency of the convex $\alpha \leq 1$ boosters to continue overfitting on the noisy training examples, and the ability of the non-convex $\alpha > 1$ boosters to continue judiciously ``giving up'' on the noisy training examples.
    Further commentary is in Section~\ref{sec:boostingexperiments}, and more noise levels are in Appendix~\ref{appendix:boostingexperiments}.}
    \label{fig:ls_accuracy_iterations}
    \bignegspace
\end{figure}

\textbf{Smooth tuning} It is not difficult to tune $\alpha$ for AdaBoost.$\alpha$, see Figure~\ref{fig:ls_accuracy_alpha} in Appendix~\ref{appendix:boostingexperiments} for consideration on the Long-Servedio dataset.~\cite{sypherd2022journal} indicated that the effective range of $\alpha$ is typically bounded, e.g., $\alpha^{*} \in [.8,8]$ for shallow CNNs; AdaBoost.$\alpha$ appears to be no different. 
In part, this is due to a \textit{saturation effect}, where $\alpha > 1$ quickly ``resembles'' the $\infty$-loss~\citep{sypherd2020alpha}.
Hence, tuning $\alpha > 1$, but not too large, trades a reasonable amount of non-convexity for robustness. 

We see similar behavior in Figure~\ref{fig:bc_alpha} on the breast cancer dataset~\citep{breastcancerdataset}, namely that for every non-zero level of symmetric label noise, an $\alpha>1$ is able to achieve greater accuracy on a clean test set, and we note the smoothness of the gains with $\alpha$, implying that tuning $\alpha$ is simple for this dataset as well.
In Appendix~\ref{appendix:boostingexperiments}, we present full results of AdaBoost.$\alpha$ on the breast cancer dataset, similarly observing gains for smaller maximum tree depths. 
\begin{figure}[h]
    \centering
    \includegraphics[width=.75\linewidth]{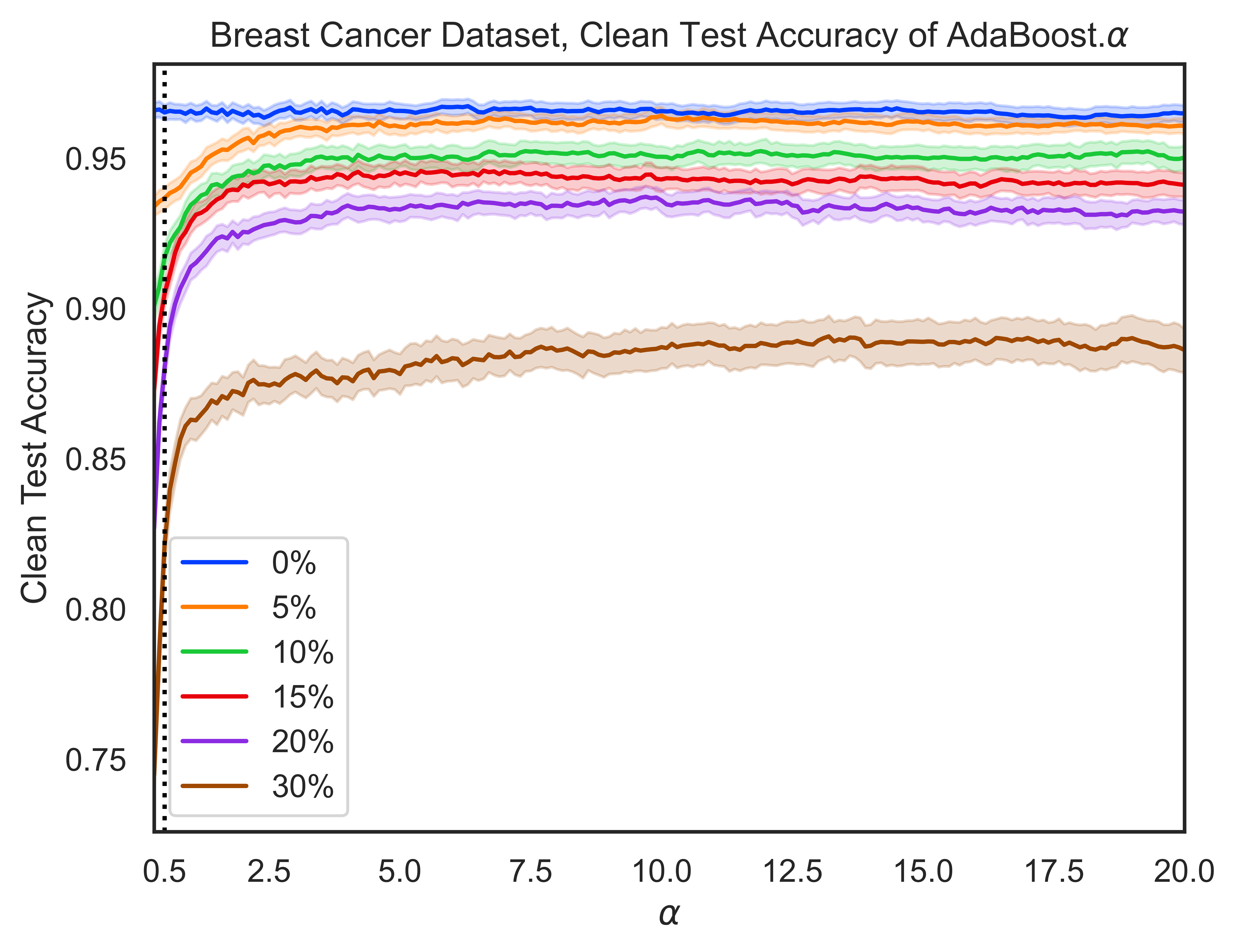}
    \caption{Clean test accuracy of AdaBoost.$\alpha$ with $100$ depth $1$ decision trees on the breast cancer dataset suffering from various levels of symmetric label noise during training. We see that vanilla AdaBoost ($\alpha=0.5$) struggles at large noise levels, but simply tuning $\alpha$ larger gives significant performance gains. For lower noise levels, tuning $\alpha$ large does not impact classification accuracy for this dataset, implying that tuning $\alpha$ is simple for this dataset.}
    \label{fig:bc_alpha}
    \bignegspace
\end{figure}



\subsection{Linear Model} \label{sec:logisticexperiments}

For the linear model experiments, we consider two datasets: a 2D GMM, and a real-world COVID-19 survey dataset~\citep{Salomone2111454118}. 
We introduce symmetric label noise into the training data for both. 

For the effectiveness metric of using the margin-based $\alpha$-loss, we consider the model parameters themselves, as they have clear interpretations in the form of odds ratios for the linear setting. 
Specifically, we examine a linear classifier trained with $\alpha$-loss on noisy data and calculate the mean squared error (MSE) of its learned parameters and those of some baseline (further described for each dataset). 
By ensuring that the model parameters are close to those of a clean model, we preserve interpretability and accuracy. 

\begin{figure}[h]
\centerline{\includegraphics[trim={0.4cm .2cm .7cm 0.2cm},clip,width=.75\linewidth]{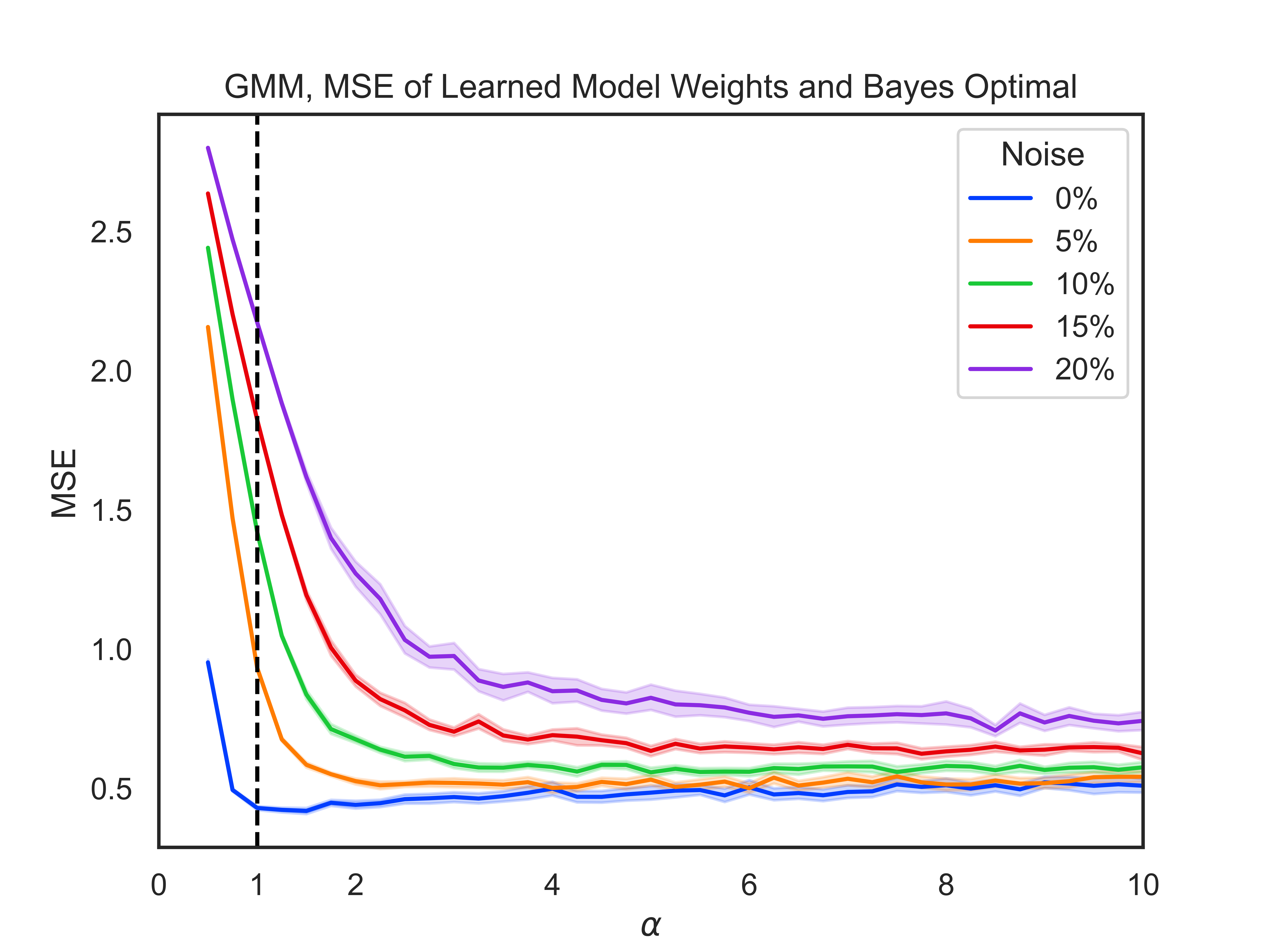}}
\caption{
MSE of Bayes optimal line and the parameters learned by $\alpha$-loss, on a 2D GMM with $86:14$ class imbalance and varying label noise levels. We see that $\alpha>1$ is able to more closely approximate the clean parameters than $\alpha\le1$, and the MSE is fairly flat in the large $\alpha$ regime, indicating that it is not difficult to tune $\alpha$.
Note that the $95\%$ confidence intervals grow wider for larger $\alpha$, indicative of the optimization/robustness tradeoff~\citep{sypherd2020alpha}.
}
\label{fig:synth_MSE}
\bignegspace
\end{figure}

\textbf{2D GMM} We first consider a 2D GMM with $\mu_1 = (1,1) = -\mu_{-1}$, identity covariance, and $\mathbb{P}[Y=1] = 0.14$ (aligning with the next experiment).
Thus, the Bayes-optimal classifier is linear, and we compare with the separator learned by training $\alpha$-loss on noisy data.
In Figure~\ref{fig:synth_MSE}, we see that tuning $\alpha>1$ results in a decreased MSE for every non-zero noise level, and
implies that the model learned by $\alpha > 1$ is closer to the Bayes optimal line than the model learned by $\alpha\le 1$, aligning with Theorems~\ref{thm:taylorlagrangeupperbound} and~\ref{thm:lowerbound}.
Tuning on this simple dataset is quite easy as the MSE is fairly flat for $\alpha > 1$, see Appendix~\ref{appendix:logisticexperiments} for more details. 


\textbf{COVID-19 survey data}
We now consider the US COVID-19 Trends and Impact Survey (US CTIS) dataset~\citep{Salomone2111454118}, which consists of self-reported survey data.
We compress the dataset from $71$ features to $42$ categorical and real-valued features including symptom data, behaviors, and comorbidities. 
For simplicity and interpretability, $8$ features, listed in Table~\ref{table:features}, were chosen using cross validation which contributed the most to the final prediction (largest odds ratios).
Each example is labeled either as RT-PCR-confirmed COVID positive ($1$) or negative ($-1$), based on self-reported diagnoses by study participants. Examples with clearly spurious responses (e.g., a negative number of people in a household) or responses with missing features were removed. This pre-processing resulted in a dataset of $864,154$ training examples with a class imbalance of $14:86$ of positive to negative COVID cases. 
\begin{table}[h]
\small\sf\centering
\begin{tabular}{lll}
\toprule
Feature&Type\\
\midrule
\verb"Age"&Categorical \\
\verb"Gender"&Categorical \\
\verb"LossOfSmellTaste"&Binary \\
\verb"ShortBreath"&Binary \\
\verb"Aches"&Binary \\
\verb"Tired"&Binary \\
\verb"Cough"&Binary \\
\verb"Fever"&Binary \\
\bottomrule
\end{tabular}
\caption{Top $8$ features of the US COVID-19 survey dataset~\citep{Salomone2111454118}, selected via the largest odds ratios on the validation set.}
\label{table:features}
\bignegspace
\end{table}


\begin{figure}[h]
\centerline{\includegraphics[trim={0.2cm .2cm .7cm 0.2cm},clip,width=.75\linewidth]{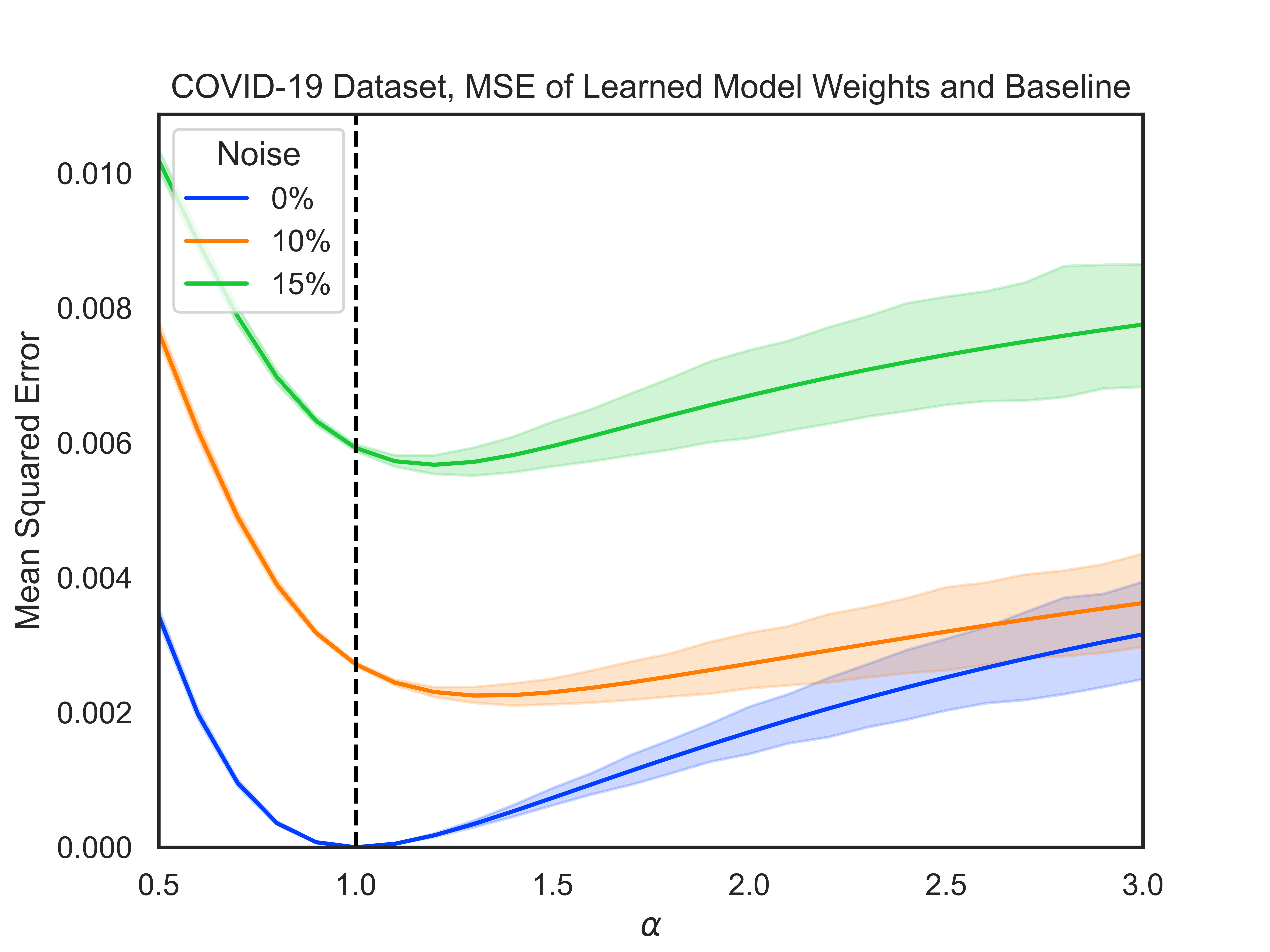}}
\caption{
A US COVID-19 survey dataset~\citep{Salomone2111454118}, plotting MSE of logistic regression ($\alpha = 1$) baseline parameters on clean data and model parameters learned using $\alpha$-loss on noisy data vs.~$\alpha$. For non-zero noise the MSE is minimized for $\alpha>1$, but some care is required in increasing $\alpha\gg 1$ as the confidence intervals widen, likely due to this being non-realizable and highly imbalanced data.
Also, note that we do not consider noise beyond $15\%$ because of stratification, i.e., higher noise would completely overwhelm the class imbalance. 
}
\label{fig:covid_MSE}
\bignegspace
\end{figure}

In Figure~\ref{fig:covid_MSE}, we compare the model parameters learned by the margin-based $\alpha$-loss on noisy data with those of the $\alpha=1$ (logistic regression) trained on \textit{clean} data, which is a calibrated model~\citep{tu1996advantages}; we are interested in the utility of $\alpha > 1$ to ``give up'' on the noisy training data and recover the clean model parameters.
We see that tuning $\alpha>1$ gives gains for both non-zero noise levels, but there is a clear tradeoff with optimization complexity; this is indicated by the widening confidence intervals as $\alpha$ increases~\citep{sypherd2020alpha}, which could be due to the COVID-19 survey data being non-realizable and highly imbalanced.  
%
However, we note that reduced MSE for $\alpha > 1$ directly translates to gains on test-time accuracy; in Figure~\ref{fig:covid_sensitivity} in Appendix~\ref{appendix:experiments} we show that the sensitivity of the model increases with increasing $\alpha$.


\section{CONCLUSION}
In this work, we have presented results indicating that the margin-based $\alpha$-loss is able to ``give up'' on noisy training data and robustly train simple models.
For boosting, we have shown, theoretically and experimentally, how tuning $\alpha > 1$ can address the negative result of~\cite{long2010random}, in the process presenting a novel robust boosting algorithm called AdaBoost.$\alpha$, which may be of independent interest.
For linear models, we have also presented theoretical and experimental results, notably showing robustness for a highly imbalanced COVID-19 survey dataset~\citep{Salomone2111454118}.
Additionally, we have presented straightforward tuning characteristics for $\alpha$ in both settings.
Lastly, regarding societal impacts, we argue that it is important to consider simple models, since they are more interpretable and have reduced energy cost; we have shown for multiple relevant domains that one can robustly train simple models with a single $\alpha$ hyperparameter.

\subsubsection*{Acknowledgements}
We thank the anonymous reviewers for their comments, and Monica Welfert at Arizona State University for her contributions to the preliminary code. 
This work is supported in part by NSF grants SCH-2205080, CIF-1901243, CIF-2134256, CIF-2007688, CIF-1815361, a Google AI for Social Good grant, and an Office of Naval Research grant N00014-21-1-2615.
This research is based on survey results from Carnegie Mellon University’s Delphi Group.

\bibliography{TS_ML}

\newpage 
\begin{appendices}

\section{Further Theoretical Results, Commentary, and Proofs} \label{appendix:proofs}

\paragraph{Classification-Calibration} Regarding the statistical efficacy of $\tilde{l}^{\alpha}$,~\cite{sypherd2022journal} showed that the margin-based $\alpha$-loss is classification-calibrated for all $\alpha \in (0,\infty]$, which is a necessary minimum condition for a ``good'' margin-based loss function to satisfy.
In words, a margin-based loss function is classification-calibrated if for each feature vector, the minimizer of its conditional risk agrees in sign with the Bayes optimal predictor; this is a pointwise form of Fisher consistency from the perspective of classification~\citep{lin2004note,bartlett2006convexity}.

\begin{lemma} \label{lem:derivativesmarginalphaloss}
For $\alpha \in (0,\infty]$, the first derivative of $\tilde{l}^{\alpha}$ with respect to the margin is given by 
\begin{align} \label{eq:dermaralphaloss}
\tilde{l}^{\alpha'}(z) := \dfrac{d}{dz} \tilde{l}^{\alpha}(z) = - \sigma'(z)\sigma(z)^{-\frac{1}{\alpha}}, 
\end{align}
its second derivative is given by 
\begin{align} \label{eq:2nddermaralphaloss}
\tilde{l}^{\alpha''}(z) := \dfrac{d^{2}}{dz^{2}} \tilde{l}^{\alpha}(z) = \frac{e^{z}\left(\alpha e^{z} - \alpha + 1\right)}{\alpha (e^{-z} + 1)^{-\frac{1}{\alpha}} (e^{z} + 1)^{3}},
\end{align}
and its third derivative is given by
\begin{align} \label{eq:3ddermaralphaloss}
\tilde{l}^{\alpha'''}(z) := \dfrac{d^{3}}{dz^{3}} \tilde{l}^{\alpha}(z) = \frac{-e^{2z} + 4 e^{z} - 1 - \frac{3 e^{z} - 2}{\alpha} - \frac{1}{\alpha^{2}}}{ e^{-z} \left(1+e^{-z} \right)^{-\frac{1}{\alpha}} (e^{z} + 1)^4}.
\end{align}
\end{lemma}

\paragraph{Discussion of Algorithm~\ref{algo:adaboost.alpha}}
The weighting used for the weak learner in Algorithm~\ref{algo:adaboost.alpha}, namely that $\theta_{t}  = \frac{1}{2}\log\left( \frac{1 - \epsilon_t}{\epsilon_t} \right)$, is the expression commonly used in vanilla AdaBoost ($\alpha = 1/2$ for AdaBoost.$\alpha$)~\citep{schapire2013boosting}. 
However, there are several other possibilities of $\theta_{t}$ for AdaBoost.$\alpha$, due to its interpolating characteristics. 
One possibility is to use $\theta_{t}  = \alpha \log\left( \frac{1 - \epsilon_t}{\epsilon_t} \right)$, for $\alpha \in (0,\infty]$, which is the optimal classification function of the margin-based $\alpha$-loss~\citep{sypherd2022journal}. 
Another possibility is to use a Wolfe line search~\citep{telgarsky2013boosting}. 
Consideration of the weighting of the weak learners, and the ensuing convergence (and consistency) characteristics for  Algorithm~\ref{algo:adaboost.alpha}, is left for future work. 


\subsection{Proof of Theorem~\ref{thm:LSmarginalphalossrobust}} \label{appendix:proofofLStheorem}

\begin{figure}[t]
     \centering
     \begin{subfigure}[b]{0.475\textwidth}
         \centering
         \includegraphics[width=\textwidth]{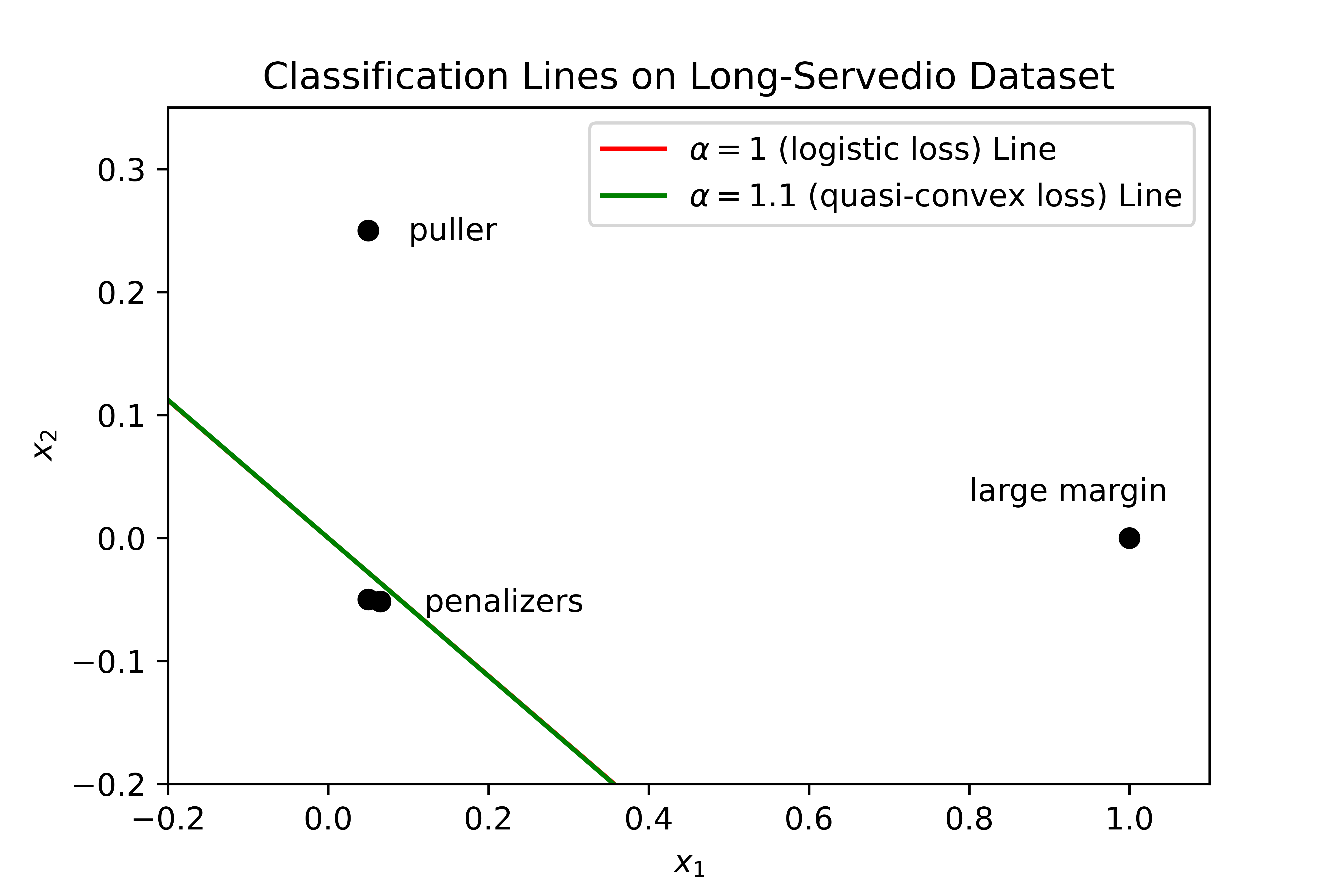}
         \caption{$\alpha = 1.1$ classification line.}
     \end{subfigure}
     \begin{subfigure}[b]{0.475\textwidth}
         \centering
         \includegraphics[width=\textwidth]{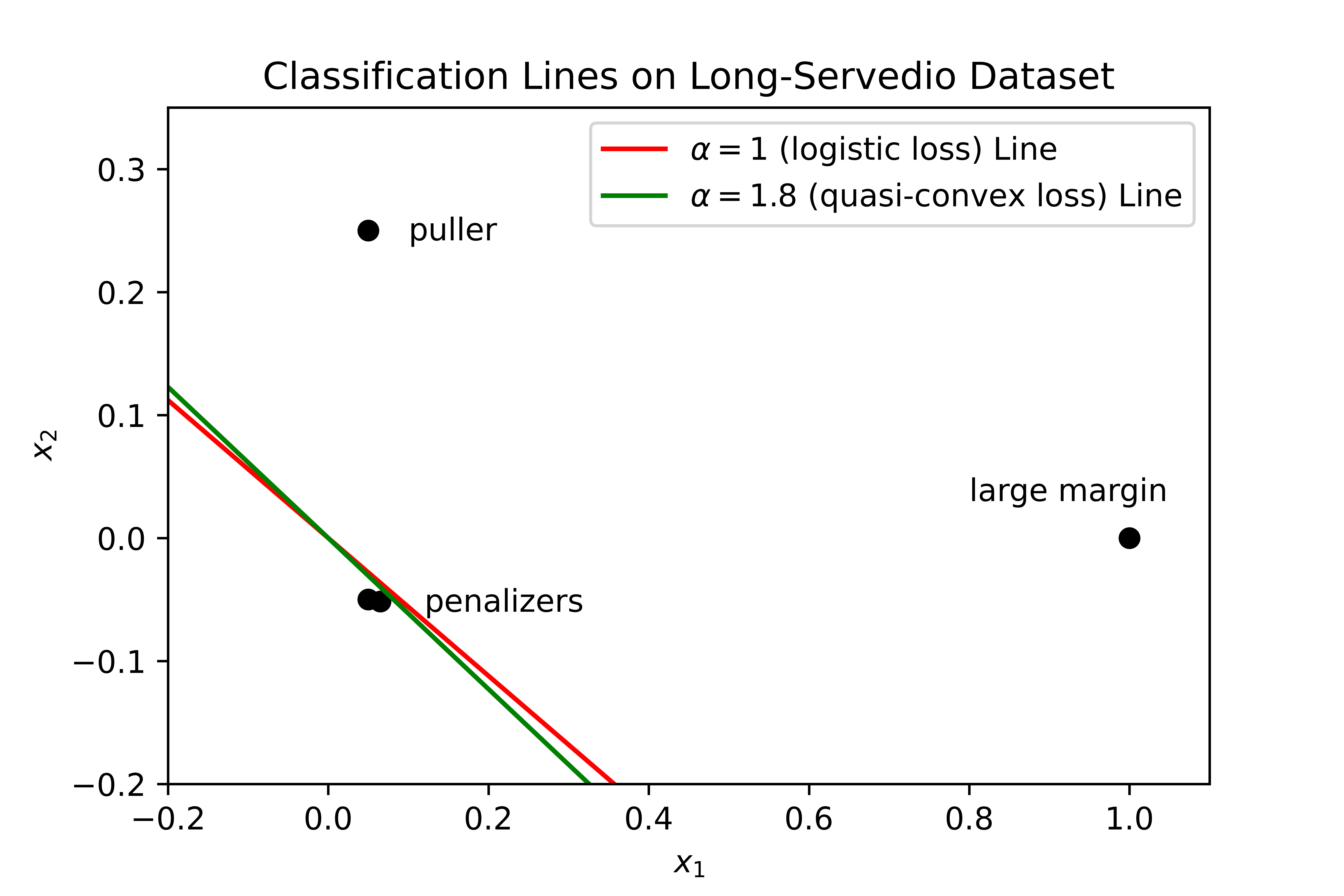}
         \caption{$\alpha = 1.8$ classification line.}
     \end{subfigure}
     \begin{subfigure}[b]{0.475\textwidth}
         \centering
         \includegraphics[width=\textwidth]{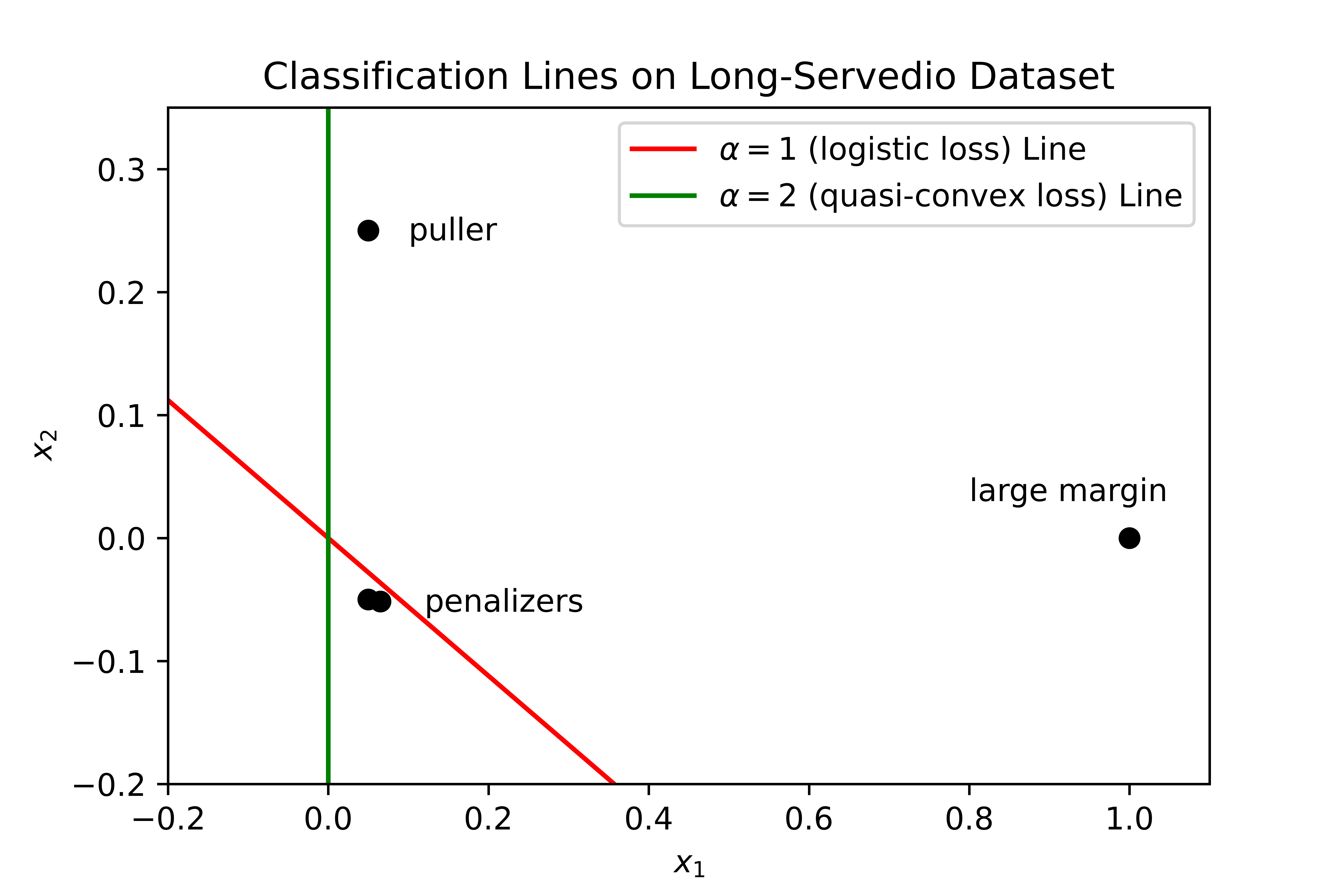}
         \caption{$\alpha = 2$ classification line.}
     \end{subfigure}
     \begin{subfigure}[b]{0.475\textwidth}
         \centering
         \includegraphics[width=\textwidth]{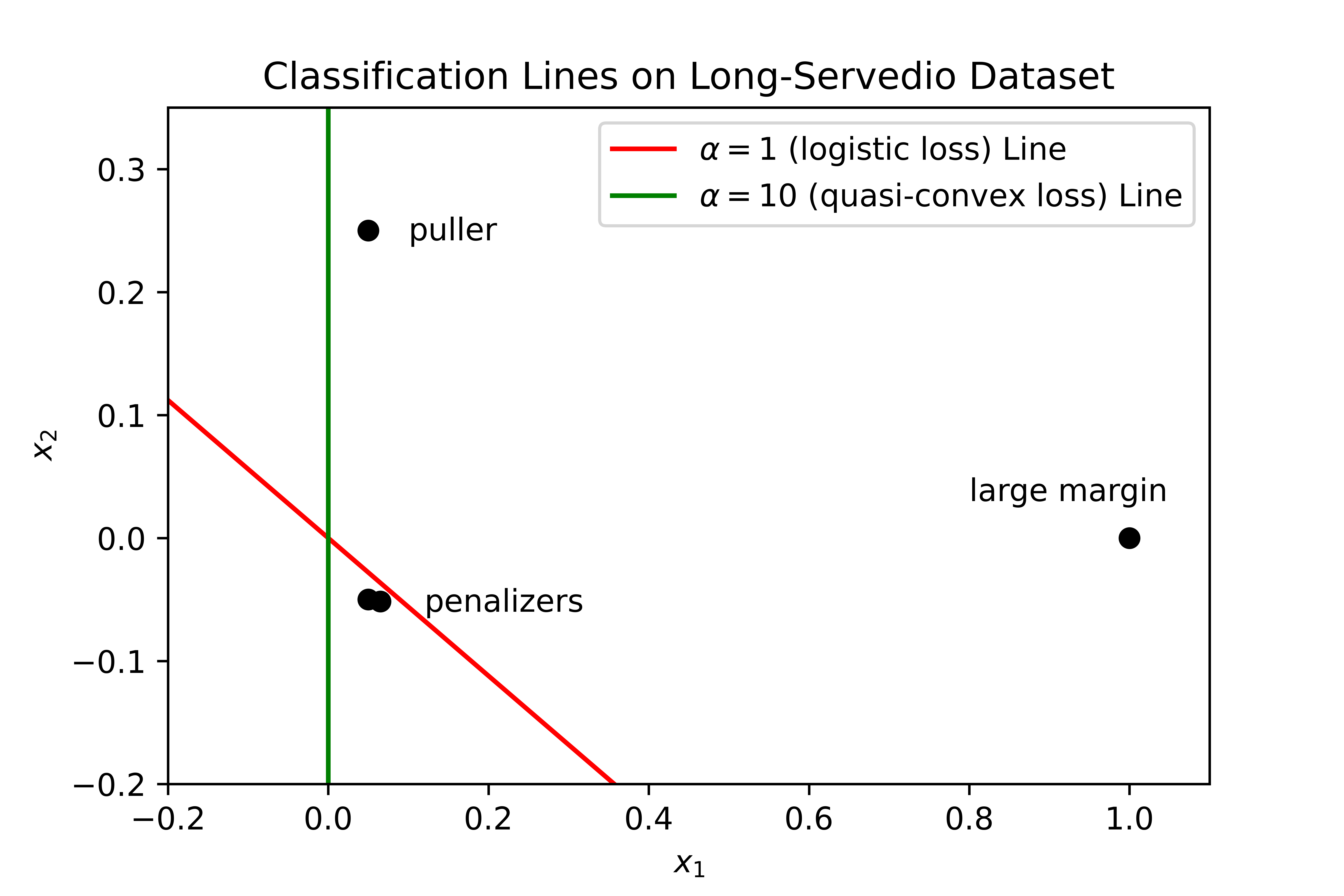}
         \caption{$\alpha = 10$ classification line.}
     \end{subfigure}
        \caption{
        Companion figure of Figure~\ref{Fig:LSclassificationlines} where $N = 2$ ($p=1/3$) and $\gamma = 1/20$ for $\alpha \in \{1.1,1.8,2,10\}$.
        }
        \label{fig:LSclassificationlines_more}
\end{figure}

The strategy of the proof is as follows:
\begin{enumerate}
    \item First, we quantify what a perfect classification solution on the Long-Servedio dataset looks like, namely, inequality requirements involving $\theta_{1}$ and $\theta_{2}$ derived from the interaction of the ``penalizers'', ``puller'', and ``large margin'' examples and the linear hypothesis class.
    \item Next, we invoke the pathological result of~\citep{long2010random}, which yields a ``bad'' margin $\gamma$ for any noise level and the margin-based $\alpha$-loss with $\alpha \leq 1$ (i.e., convex losses as articulated in Proposition~\ref{Prop:alpha-loss-convex}).
    \item Then, we reduce the first order equation of the margin-based $\alpha$-loss evaluated at the four examples over the linear weights for $\alpha \in (1,\infty)$, and through a cancellation yield an equation which has a function of $\theta_{1}$ on the LHS and a similar function of both $\theta_{1}$ and $\theta_{2}$ on the RHS, i.e., an asymmetric equation not allowing full analytical solution but allowing reasoning about possible solutions.
    \item Finally, using continuity arguments exploiting the giving up properties of the quasi-convex margin-based $\alpha$-losses for $\alpha \in (1,\infty)$, we guarantee the existence of a solution $(\theta_{1}^{*},\theta_{2}^{*})$ with perfect classification accuracy on the \textit{clean} Long-Servedio dataset under the given pathological margin $\gamma$.
\end{enumerate}


By the construction of the hypothesis class~\citep{long2010random}, namely that $\mathcal{H} = \{h_{1}(\mathbf{x}) = x_{1},h_{2}(\mathbf{x}) = x_{2}\}$, 
notice that the classification lines (constructed by the boosting algorithm in this pathological example) are given by $\theta_{1}x_{1} + \theta_{2}x_{2} = 0$ and must pass through the origin.
Rewriting this classification line, we have that $x_{2} = -\frac{\theta_{1}}{\theta_{2}}x_{1}$.
Reasoning about perfect classification weights $(\theta_{1}^{*},\theta_{2}^{*})$, notice (see Figure~\ref{Fig:LSclassificationlines}) that the ``large margin'' example forces $\theta_{1}^{*} > 0$. 
Further, reasoning about the ``penalizers'', we find that we require $\theta_{1}^{*} > \theta_{2}^{*}$, and reasoning about the ``puller'', we also find that we require $\theta_{1}^{*} > -5\theta_{2}^{*}$.
Thus, perfect classification weights on the Long-Servedio dataset must satisfy all of the following:
\begin{align} \label{eq:LSgoodsolutions}
\theta_{1}^{*} > 0 \quad \text{and} \quad  \theta_{1}^{*} > \theta_{2}^{*} \quad \text{and} \quad \theta_{1}^{*} > -5\theta_{2}^{*}.
\end{align}
We now examine the solutions to the first-order equation for $\alpha \in (0,\infty]$.

As in~\citep{long2010random}, let $1 < N < \infty$ be the noise parameter such that the noise rate $p = \frac{1}{N+1}$, and hence $1-p = \frac{N}{N+1}$.
Under the Long-Servedio setup with the margin-based $\alpha$-loss (and recalling that all four examples have classification label $y = 1$), we have that 
\begin{align}
R_{\alpha}^{p}(\theta_{1},\theta_{2}) &= \frac{1}{4} \sum\limits_{x \in S} \left[(1-p) \tilde{l}^{\alpha}(\theta_{1}x_{1} + \theta_{2}x_{2}) + p \tilde{l}^{\alpha}(-\theta_{1}x_{1} - \theta_{2}x_{2}) \right].
\end{align}

It is clear that minimizing $4(N+1)R_{\alpha}^{p}$ is the same as minimizing $R_{\alpha}^{p}$ so we shall henceforth work with $4(N+1)R_{\alpha}^{p}$ since it gives rise to cleaner expressions. 
We have that 
\begin{align}
4(N+1)R_{\alpha}^{p}(\theta_{1},\theta_{2}) &= \sum\limits_{x \in S} [N \tilde{l}^{\alpha}(\theta_{1}x_{1} + \theta_{2}x_{2}) + \tilde{l}^{\alpha}(-\theta_{1}x_{1} - \theta_{2}x_{2})] \\
\nonumber &= N \tilde{l}^{\alpha}(\theta_{1}) + \tilde{l}^{\alpha}(-\theta_{1}) + 2N \tilde{l}^{\alpha}(\theta_{1}\gamma - \theta_{2}\gamma) + 2 \tilde{l}^{\alpha}(-\theta_{1}\gamma + \theta_{2} \gamma) \\ &\quad\quad\quad\quad\quad\quad\quad\quad\quad\quad\quad\quad+ N \tilde{l}^{\alpha}(\theta_{1}\gamma + 5\theta_{2} \gamma) + \tilde{l}^{\alpha}(-\theta_{1}\gamma - 5\theta_{2}\gamma). \label{eq:LSlandscape}
\end{align}
See Figure~\ref{fig:LSoptimizationlandscape} for a visualization of~\eqref{eq:LSlandscape}.
\begin{figure}[t]
     \centering
     \begin{subfigure}[b]{0.475\textwidth}
         \centering
         \includegraphics[width=\textwidth]{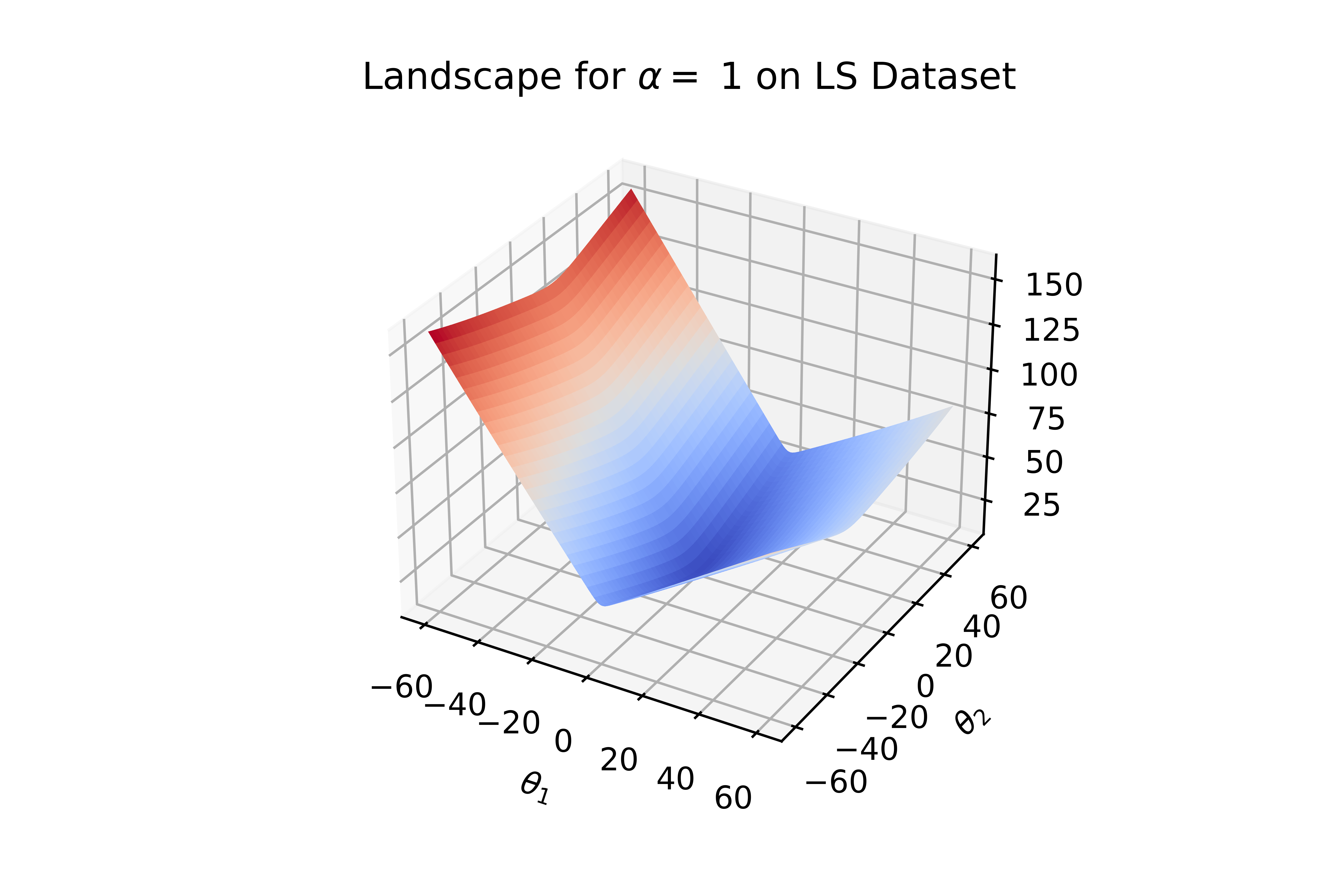}
         \caption{$\alpha = 1$ Optimization Landscape.}
     \end{subfigure}
     \begin{subfigure}[b]{0.475\textwidth}
         \centering
         \includegraphics[width=\textwidth]{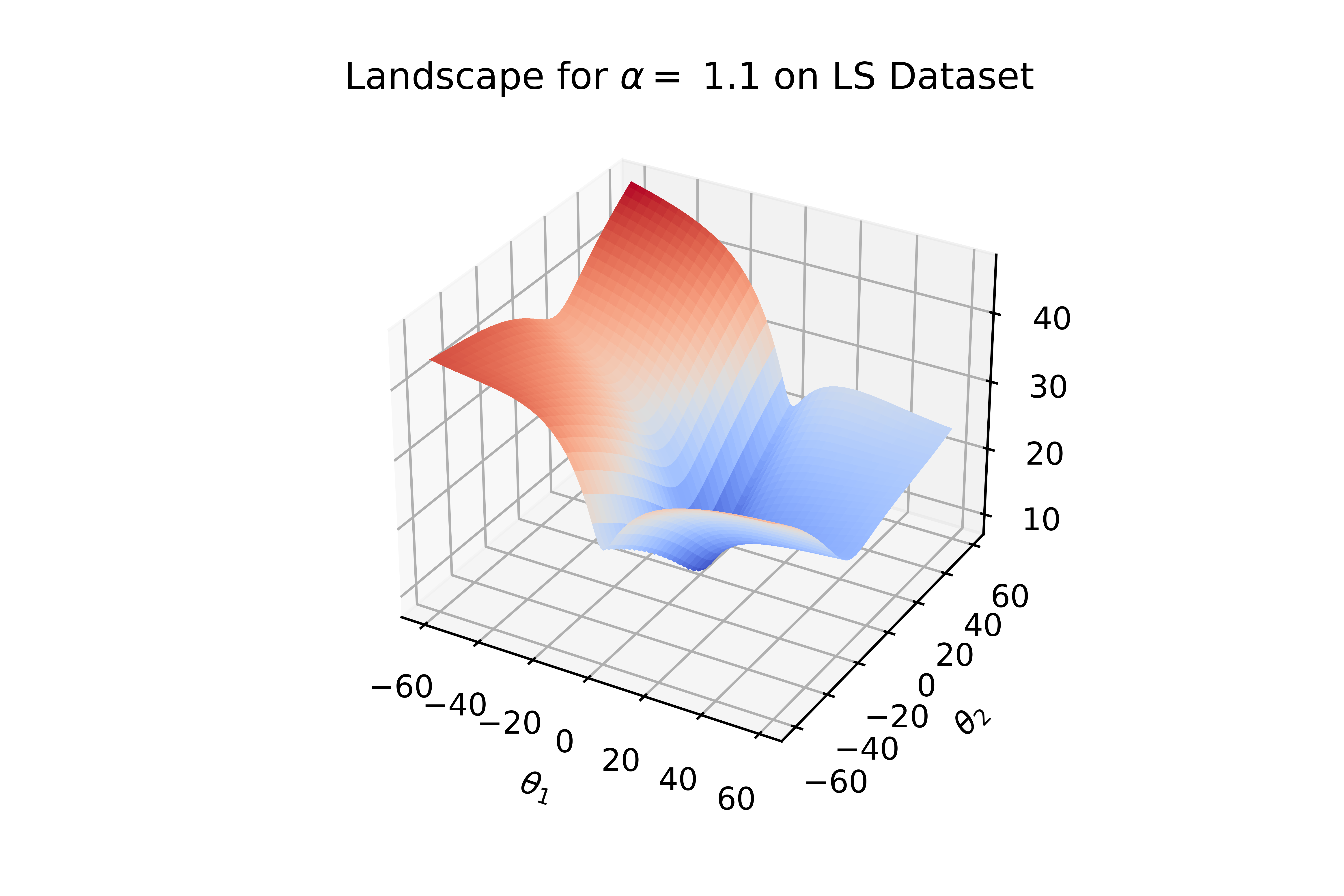}
         \caption{$\alpha = 1.1$ Optimization Landscape.}
     \end{subfigure}
     \begin{subfigure}[b]{0.475\textwidth}
         \centering
         \includegraphics[width=\textwidth]{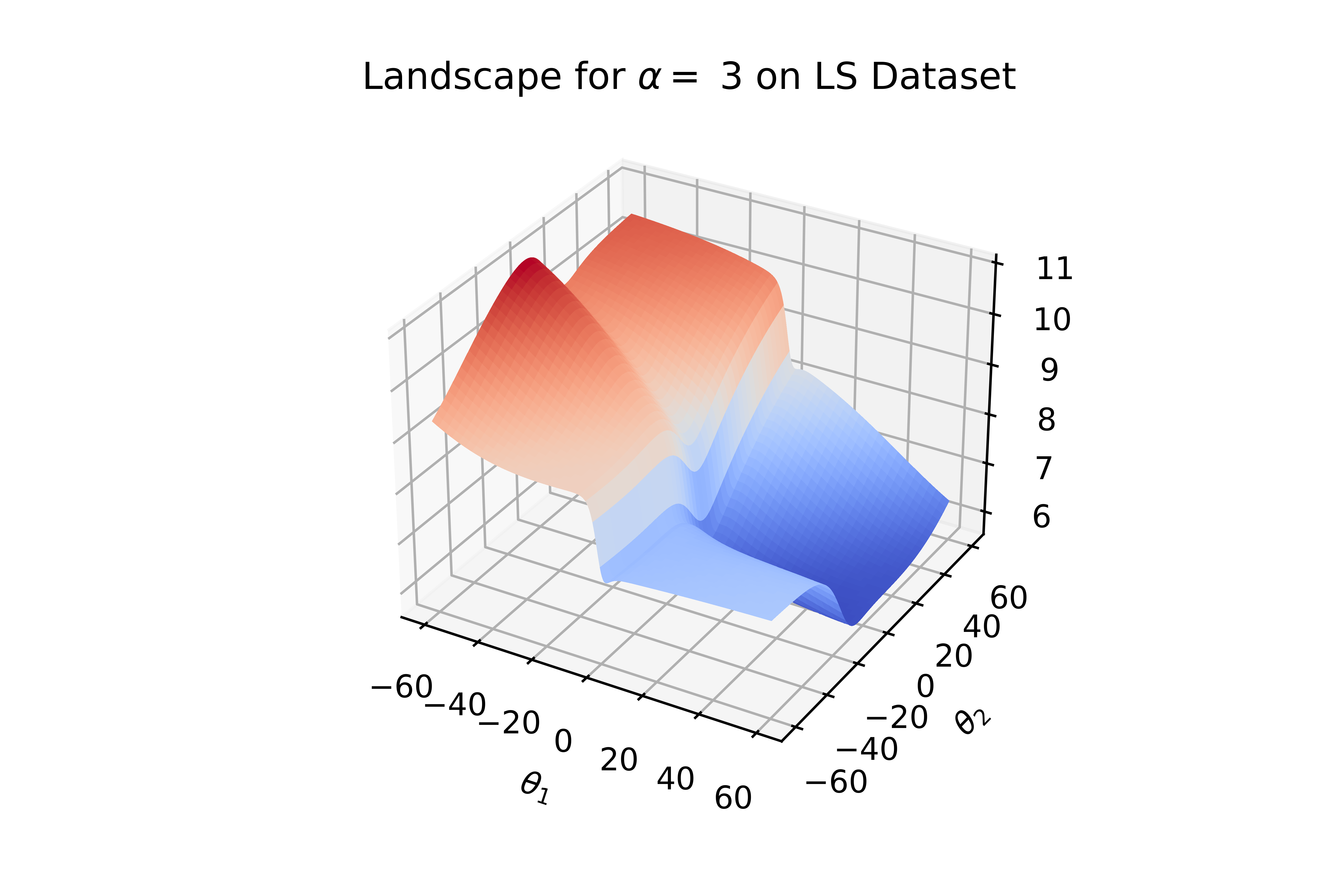}
         \caption{$\alpha = 3$ Optimization Landscape.}
     \end{subfigure}
     \begin{subfigure}[b]{0.475\textwidth}
         \centering
         \includegraphics[width=\textwidth]{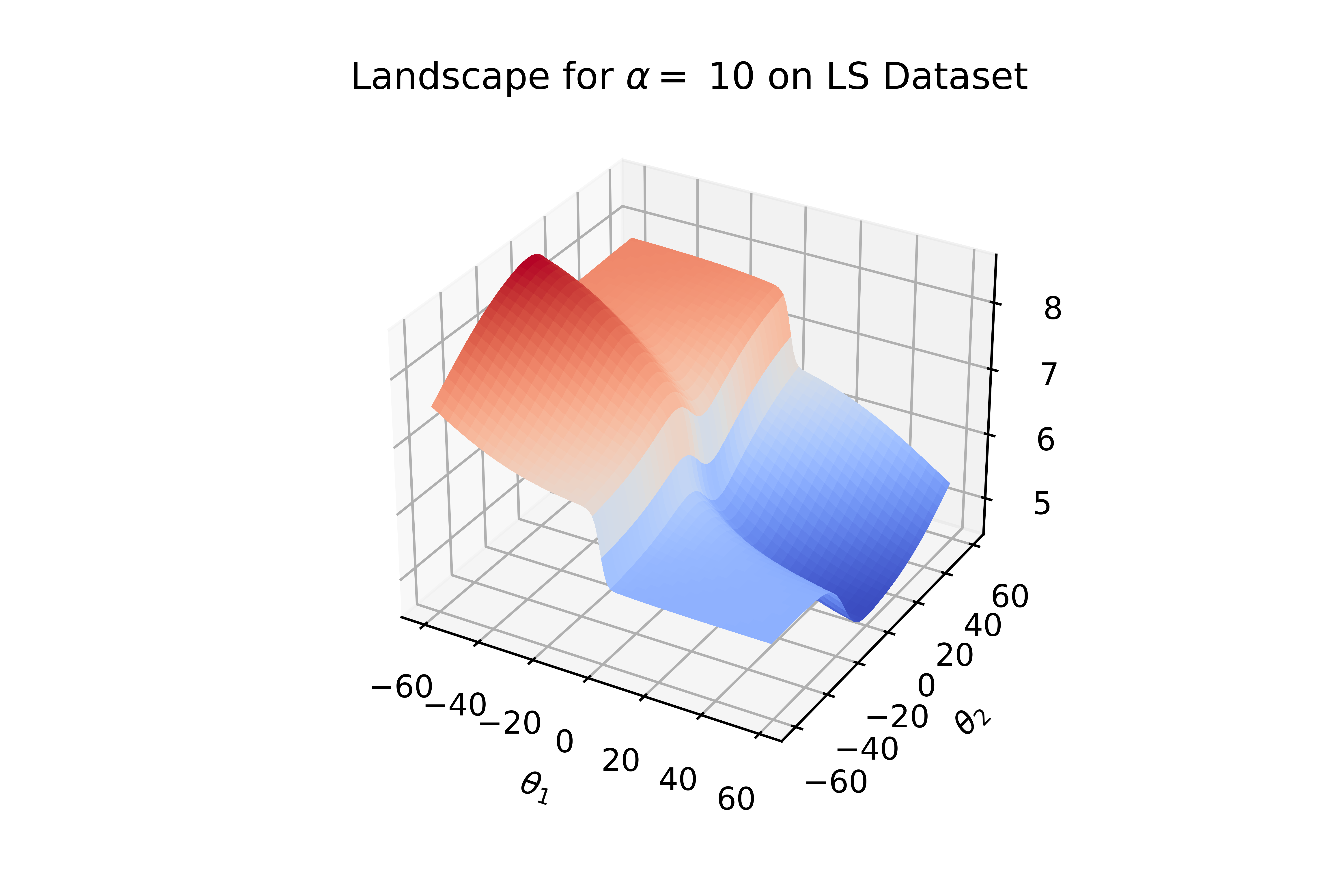}
         \caption{$\alpha = 10$ Optimization Landscape.}
     \end{subfigure}
        \caption{
        Plots of optimization landscapes on the Long-Servedio dataset, i.e.~\eqref{eq:LSlandscape}, for $\alpha \in \{1,1.1,3,10\}$.
        Aligning with Figure~\ref{Fig:LSclassificationlines}, $N = 2$ and $\gamma = 1/20$. 
        For $\alpha = 1$, the landscape is convex, which was formally proved (for any distribution) in~\citep{sypherd2020alpha}.
        For $\alpha = 1.1$, the landscape is non-convex, but not too much, which was also quantified in~\citep{sypherd2020alpha}.
        For $\alpha = 3$, the landscape is more non-convex, and notice that the quality of the solutions (in the sense of~\eqref{eq:LSgoodsolutions}) is significantly better for $\alpha = 3$.
        For $\alpha = 10$, the landscape strongly resembles the $\alpha = 3$, but is ``flatter''. 
        }
        \label{fig:LSoptimizationlandscape}
\end{figure}

Again following notation in~\citep{long2010random},
let $P_{1}^{\alpha}(\theta_{1},\theta_{2})$ and $P_{2}^{\alpha}(\theta_{1},\theta_{2})$ be defined as follows:
\begin{align}
P_{1}^{\alpha}(\theta_{1},\theta_{2}) := \frac{\partial}{\partial \theta_{1}} 4(N+1)R_{\alpha}^{p}(\theta_{1},\theta_{2}) \quad \text{and} \quad P_{2}^{\alpha}(\theta_{1},\theta_{2}) := \frac{\partial}{\partial \theta_{2}} 4(N+1)R_{\alpha}^{p}(\theta_{1},\theta_{2}).
\end{align}
Thus, differentiating~\eqref{eq:LSlandscape} by $\theta_{1}$ and $\theta_{2}$ respectively, we have
\begin{align}
\nonumber P_{1}^{\alpha}(\theta_{1},\theta_{2}) &= N \tilde{l}^{\alpha{'}}(\theta_{1}) - \tilde{l}^{\alpha{'}}(-\theta_{1}) + 2\gamma N \tilde{l}^{\alpha{'}}(\theta_{1}\gamma - \theta_{2}\gamma) \\
&\quad\quad\quad\quad\quad\quad\quad\quad -2\gamma \tilde{l}^{\alpha{'}}(-\theta_{1}\gamma + \theta_{2}\gamma) + N\gamma \tilde{l}^{\alpha{'}}(\theta_{1}\gamma + 5 \theta_{2}\gamma) - \gamma \tilde{l}^{\alpha{'}}(-\theta_{1}\gamma - 5\theta_{2}\gamma),
\end{align}
and
\begin{align}
P_{2}^{\alpha}(\theta_{1},\theta_{2}) &= -2\gamma N \tilde{l}^{\alpha{'}}(\theta_{1}\gamma - \theta_{2}\gamma) + 2\gamma \tilde{l}^{\alpha{'}}(-\theta_{1}\gamma+\theta_{2}\gamma) + 5 \gamma N \tilde{l}^{\alpha{'}}(\theta_{1}\gamma + 5 \theta_{2} \gamma) - 5 \gamma \tilde{l}^{\alpha{'}}(-\theta_{1} \gamma - 5 \theta_{2} \gamma).
\end{align}

In order to reason about the quality of the solutions to~\eqref{eq:LSlandscape} for $\alpha \in (1,\infty)$, we want to find where $P_{1}^{\alpha}(\theta_{1},\theta_{2}) = P_{2}^{\alpha}(\theta_{1},\theta_{2}) = 0$ for the margin-based $\alpha$-loss.
So, rewriting $P_{1}^{\alpha}(\theta_{1},\theta_{2}) = 0$, we obtain 
\begin{align} \label{eq:marls0}
\nonumber N \tilde{l}^{\alpha{'}}(\theta_{1}) + 2\gamma N\tilde{l}^{\alpha{'}}(\theta_{1}\gamma - \theta_{2}\gamma) &+ N\gamma \tilde{l}^{\alpha{'}}(\theta_{1}\gamma + 5 \theta_{2}\gamma) \\ &= \tilde{l}^{\alpha{'}}(-\theta_{1}) + 2\gamma \tilde{l}^{\alpha{'}}(-\theta_{1}\gamma + \theta_{2}\gamma)  + \gamma \tilde{l}^{\alpha{'}}(-\theta_{1}\gamma - 5\theta_{2}\gamma), 
\end{align}
and rewriting $P_{2}^{\alpha}(\theta_{1},\theta_{2}) = 0$, we obtain 
\begin{align} \label{eq:marls1}
2\gamma \tilde{l}^{\alpha{'}}(-\theta_{1}\gamma+\theta_{2}\gamma)
+ 5 \gamma N \tilde{l}^{\alpha{'}}(\theta_{1}\gamma + 5 \theta_{2} \gamma) = 2\gamma N \tilde{l}^{\alpha{'}}(\theta_{1}\gamma - \theta_{2}\gamma) + 5 \gamma \tilde{l}^{\alpha{'}}(-\theta_{1} \gamma - 5 \theta_{2} \gamma).
\end{align}
Substituting~\eqref{eq:marls1} into~\eqref{eq:marls0}, we are able to cancel a term and recover 
\begin{align}
N\tilde{l}^{\alpha{'}}(\theta_{1}) - \tilde{l}^{\alpha{'}}(-\theta_{1}) = 6\gamma \tilde{l}^{\alpha{'}}(-\theta_{1}\gamma - 5\theta_{2}\gamma) - 6N\gamma \tilde{l}^{\alpha{'}}(\theta_{1}\gamma + 5 \theta_{2} \gamma).
\end{align}
Rewriting, we obtain 
\begin{align} \label{eq:marls2}
N\tilde{l}^{\alpha{'}}(\theta_{1}) - \tilde{l}^{\alpha{'}}(-\theta_{1}) = - 6\gamma \left[ N\tilde{l}^{\alpha{'}}(\gamma(\theta_{1} + 5 \theta_{2})) - \tilde{l}^{\alpha{'}}(-\gamma(\theta_{1} + 5\theta_{2}))\right].
\end{align}
Notice that $B_{N}^{\alpha}(x) = N\tilde{l}^{\alpha{'}}(x) - \tilde{l}^{\alpha{'}}(-x)$, with $x \in \mathbb{R}$, is common on both sides.
From Lemma~\ref{lem:derivativesmarginalphaloss}, 
we have that $\tilde{l}^{\alpha{'}}(x) := - \sigma'(x)\sigma(x)^{-1/\alpha}$ for $\alpha \in (0,\infty]$.
Plugging this into $B_{N}^{\alpha}$ (and using the fact that $\sigma'(x)$ is an even function), we have that 
\begin{align}
B_{N}^{\alpha}(x) &= N\left(- \sigma'(x)\sigma(x)^{-1/\alpha} \right) - \left(- \sigma'(-x)\sigma(-x)^{-1/\alpha} \right) \\
&= \sigma'(x)\sigma(-x)^{-1/\alpha} - N \sigma'(x)\sigma(x)^{-1/\alpha} \\
&= \sigma'(x) \left( \sigma(-x)^{-1/\alpha} - N \sigma(x)^{-1/\alpha} \right).
\end{align}
Using this, we can rewrite~\eqref{eq:marls2} as
\begin{align} 
\label{eq:LSmarimportante} B_{N}^{\alpha}(\theta_{1}) &= - 6\gamma B_{N}^{\alpha}(\gamma(\theta_{1} + 5\theta_{2})),
\end{align}
which is equivalent to
\begin{align}
\label{eq:LSmarimportante2} \sigma'(\theta_{1})\left( \sigma(-\theta_{1})^{-1/\alpha} - N \sigma(\theta_{1})^{-1/\alpha} \right) &= - 6\gamma \sigma'(\gamma(\theta_{1} + 5\theta_{2})) \left( \sigma(-\gamma(\theta_{1} + 5\theta_{2}))^{-1/\alpha} - N \sigma(\gamma(\theta_{1} + 5\theta_{2}))^{-1/\alpha}\right), 
\end{align}
and both quantify solutions $(\theta_{1}^{*},\theta_{2}^{*})$.
Notice that it is unfortunately not possible to analytically reduce~\eqref{eq:LSmarimportante2} for general $\alpha \in (1,\infty)$ because it is a difference of $\alpha$ power expressions, i.e., a transcendental equation.
However, while we cannot analytically recover solutions $(\theta_{1}^{*}, \theta_{2}^{*})$ for $\alpha \in (1,\infty)$,
we can reason about the solutions themselves (from the perspective of~\eqref{eq:LSgoodsolutions}), because we can utilize nice properties of $B_{N}^{\alpha}$. 
For instance, one key thing to notice in~\eqref{eq:LSmarimportante} is that $B_{N}^{\alpha}$ on the LHS depends only on one component of the solution vector, namely $\theta_{1}$, whereas the RHS depends on both components of the solution vector $(\theta_{1},\theta_{2})$.

\begin{figure}[h]
     \centering
     \begin{subfigure}[b]{0.475\textwidth}
         \centering
         \includegraphics[width=\textwidth]{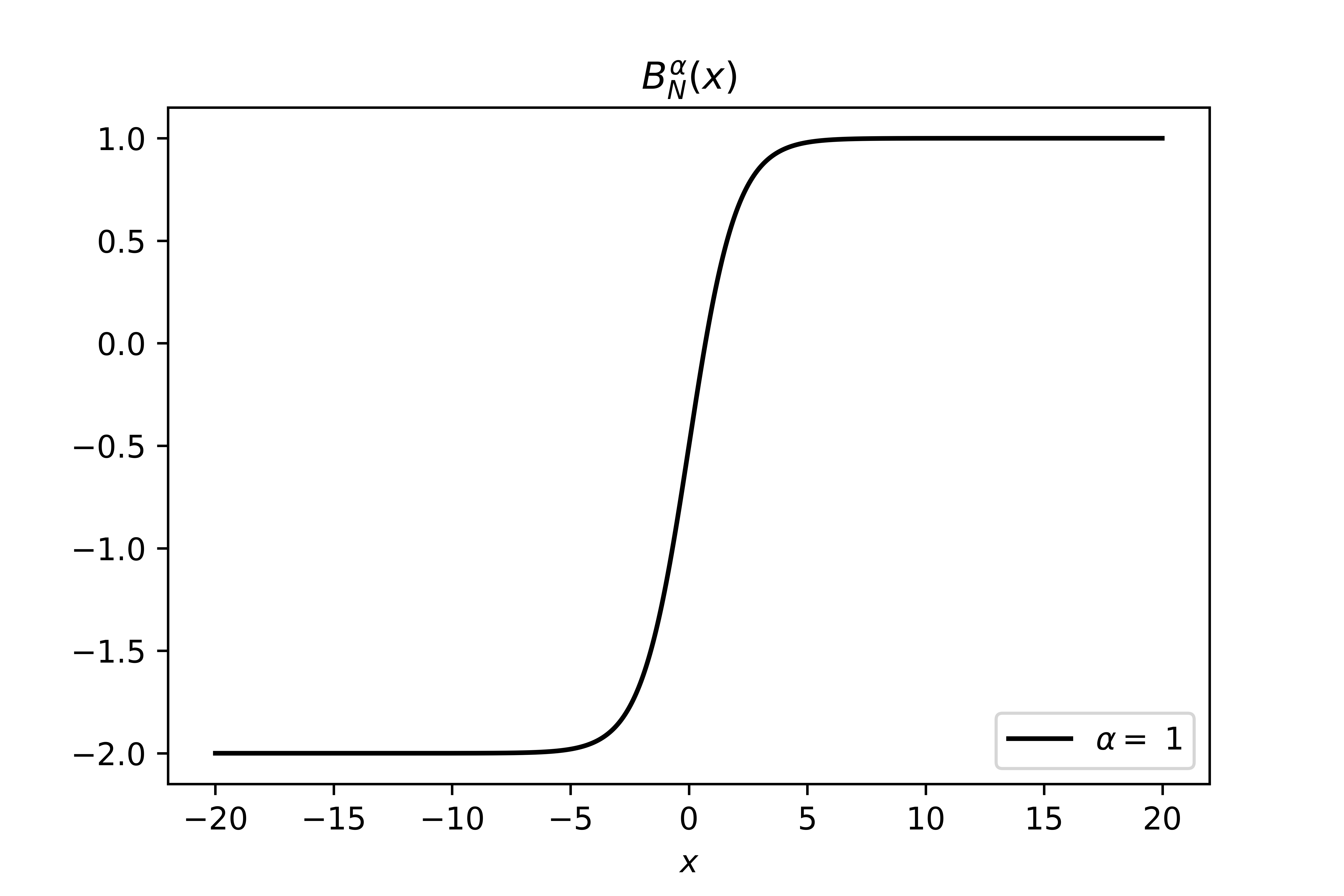}
         \caption{$B_{N}^{\alpha}(x)$ for $\alpha = 1$.}
     \end{subfigure}
     \begin{subfigure}[b]{0.475\textwidth}
         \centering
         \includegraphics[width=\textwidth]{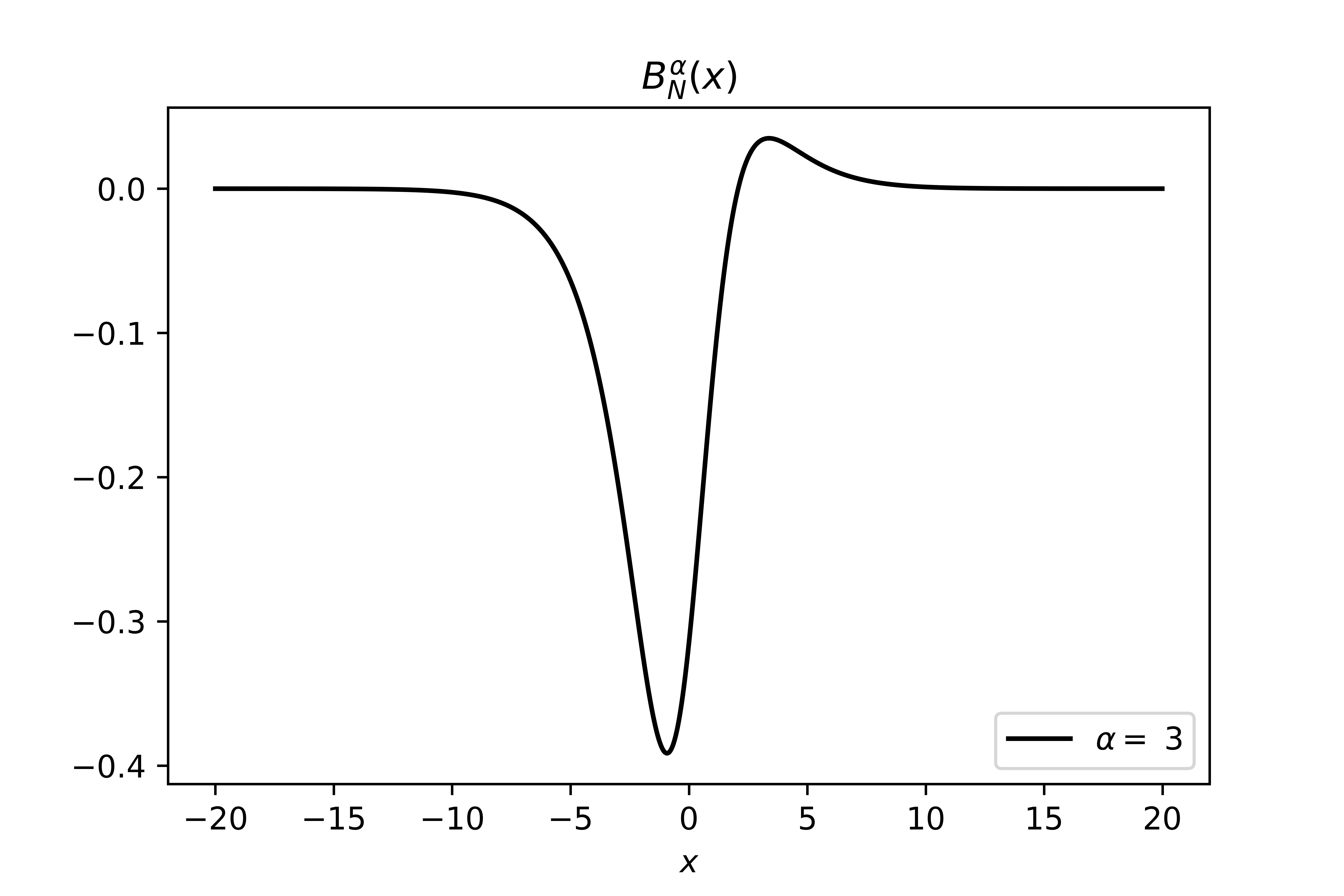}
         \caption{$B_{N}^{\alpha}(x)$ for $\alpha = 3$.}
     \end{subfigure}
        \caption{Plots of $B_{N}^{\alpha}(x)$ for $\alpha = 1$ and $3$, where $N = 2$. For $\alpha = 1$, notice that $B_{N}^{\alpha}(x)$ is non-decreasing in $x$. On the other hand, notice that for $\alpha = 3$, $B_{N}^{\alpha}(x)$ is \underline{not} non-decreasing.
        One can also see other properties of $B_{N}^{\alpha}$ in figure (b) as articulated in Lemma~\ref{lem:BNalpha}.
        }
        \label{fig:LSBN}
\end{figure}

To this end, we take a detour from the main thread to aggregate some nice properties of $B_{N}^{\alpha}$ for $\alpha > 1$. See Figure~\ref{fig:LSBN} for a plot of $B_{N}^{\alpha}$.

\begin{lemma} \label{lem:BNalpha}
Consider for $\alpha \in (0,\infty]$ and $1 < N < \infty$,
\begin{align}
B_{N}^{\alpha}(x) := \sigma'(x) \left(\sigma(-x)^{-1/\alpha} - N \sigma (x)^{-1/\alpha} \right),
\end{align}
where $x \in \mathbb{R}$.
The following are properties of $B_{N}^{\alpha}$:
\begin{enumerate}
    \item For $\alpha \leq 1$, $B_{N}^{\alpha}(x)$ is non-decreasing in $x$.
    \item For $\alpha > 1$, $B_{N}^{\alpha}(x)$ is \underline{not} non-decreasing in $x$.
    \item Note that $\lim\limits_{\alpha \rightarrow \infty} B_{N}^{\alpha}(x) = \sigma'(x) (1-N)$.
    \item For $\alpha > 1$, $\lim\limits_{x \rightarrow +\infty} B_{N}^{\alpha}(x) \rightarrow 0^{+}$ and $\lim\limits_{x \rightarrow -\infty} B_{N}^{\alpha}(x) \rightarrow 0^{-}$.
    \item For $\alpha > 1$, the resulting limits of the previous property are reversed for $-B_{N}^{\alpha}$.
    \item For $\alpha > 1$, $B_{N}^{\alpha}(x) > 0$ if and only if $x > \alpha \ln{N}$.
\end{enumerate}
\end{lemma}
%
\begin{figure}[H]
     \centering
     \begin{subfigure}[b]{0.475\textwidth}
         \centering
         \includegraphics[width=\textwidth]{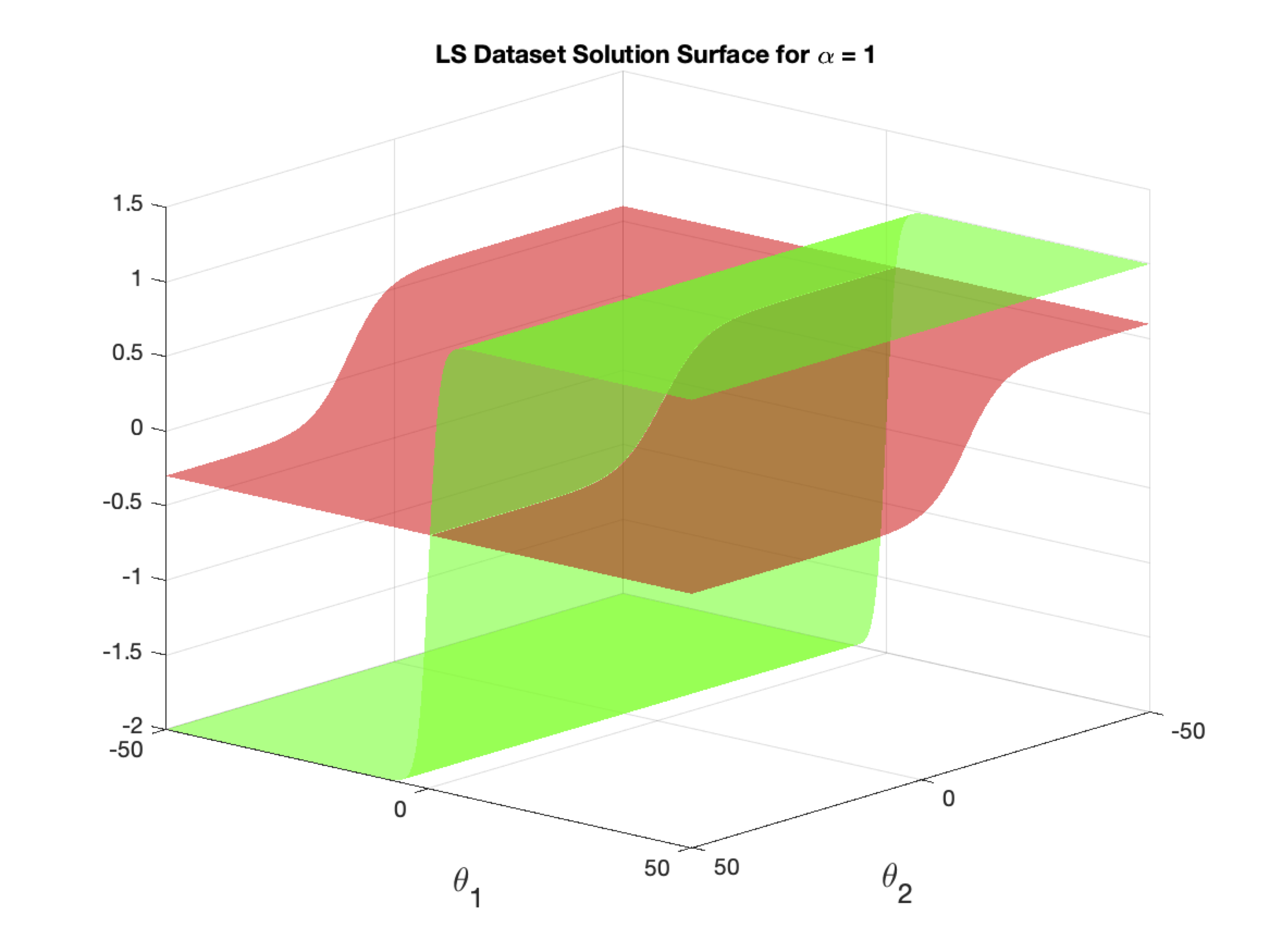}
         \caption{Plot of LHS (green) and RHS (red) of~\eqref{eq:LSimportantrestate} for $\alpha = 1$.}
         \label{fig:LSsurface1}
     \end{subfigure}
     \begin{subfigure}[b]{0.475\textwidth}
         \centering
         \includegraphics[width=\textwidth]{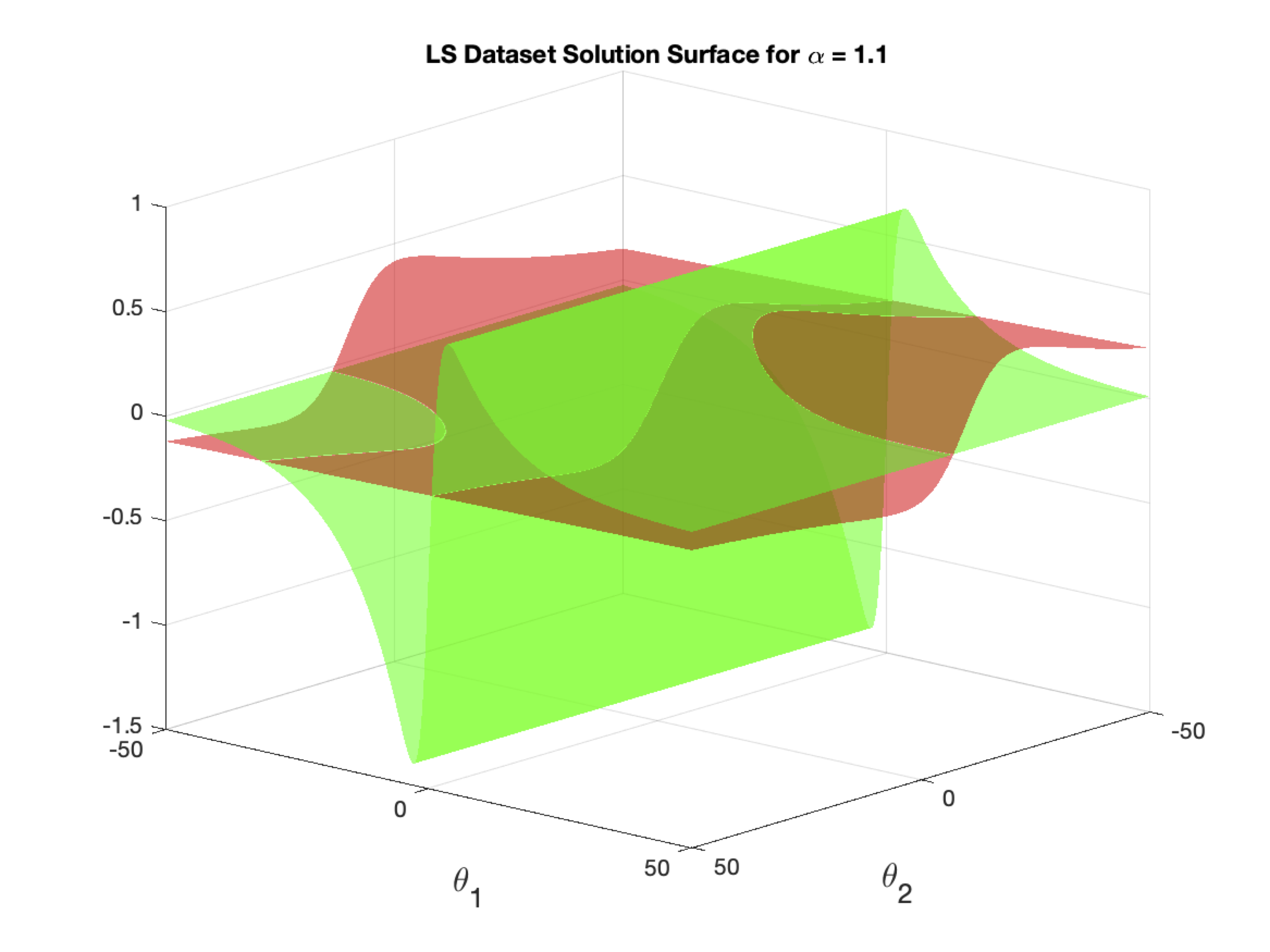}
         \caption{Plot of LHS (green) and RHS (red) of~\eqref{eq:LSimportantrestate} for $\alpha = 1.1$.}
         \label{fig:LSsurface1.1}
     \end{subfigure}
     \begin{subfigure}[b]{0.475\textwidth}
         \centering
         \includegraphics[width=\textwidth]{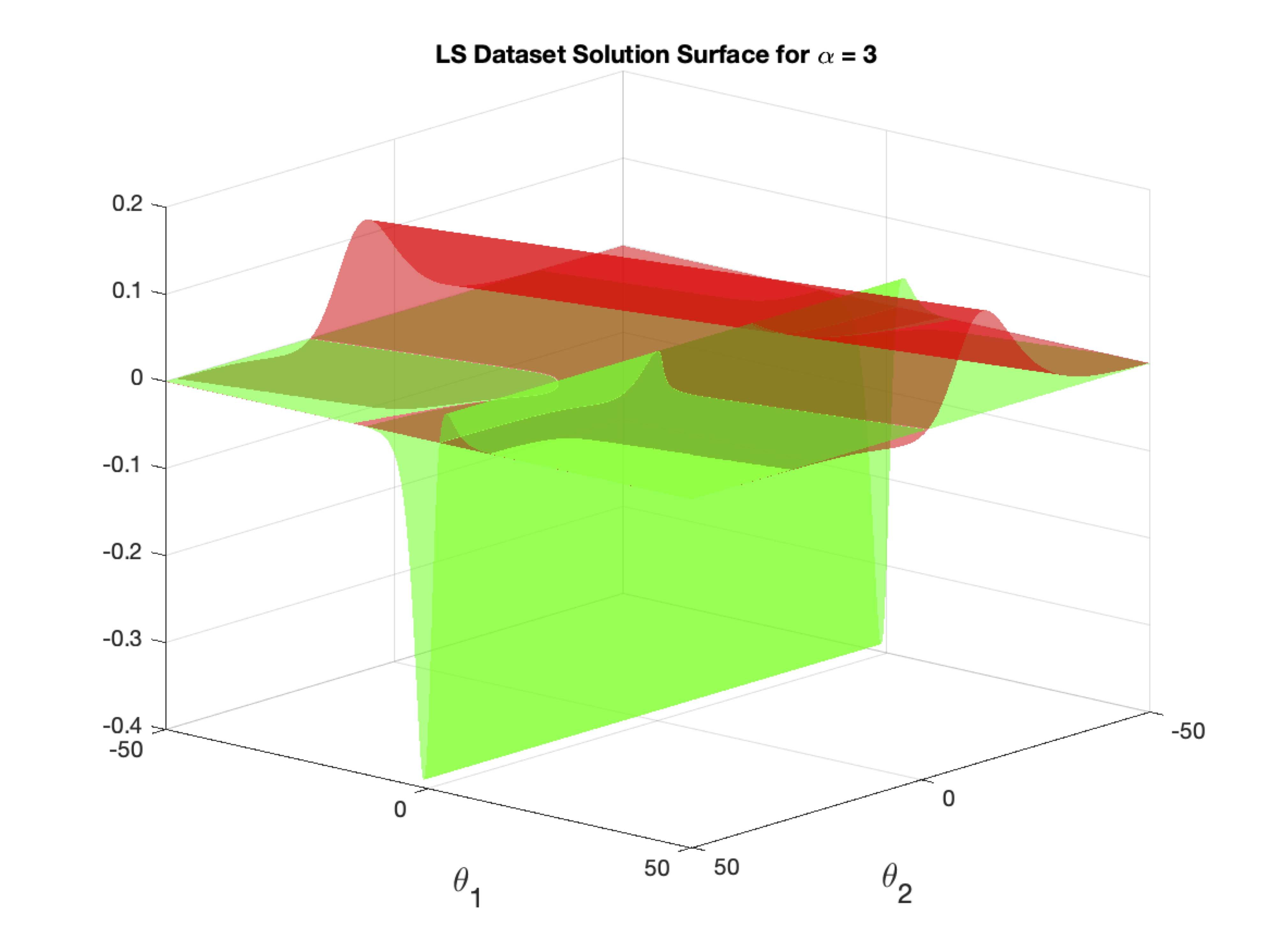}
         \caption{Plot of LHS (green) and RHS (red) of~\eqref{eq:LSimportantrestate} for $\alpha = 3$.}
         \label{fig:LSsurface3}
     \end{subfigure}
          \begin{subfigure}[b]{0.475\textwidth}
         \centering
         \includegraphics[width=\textwidth]{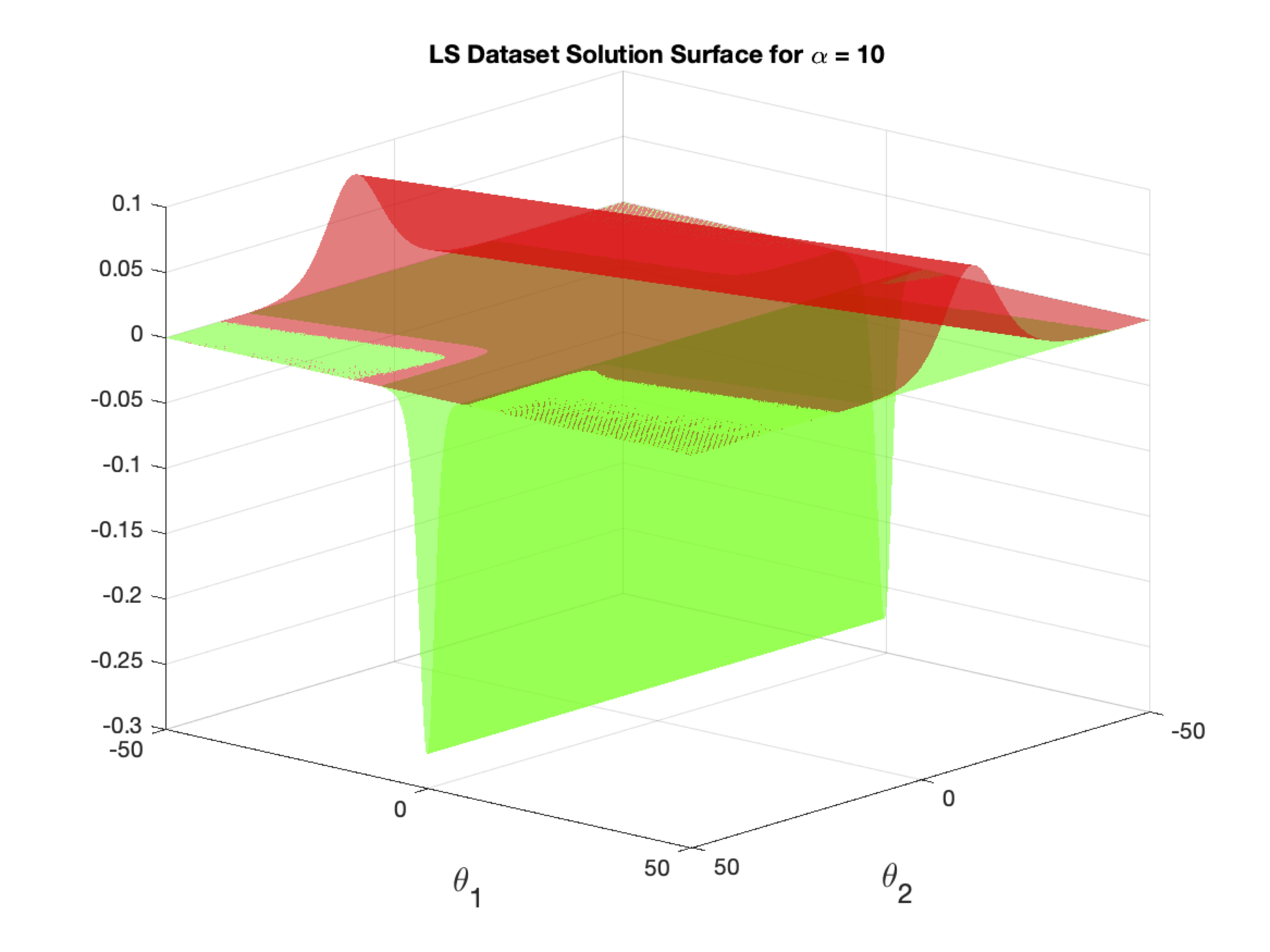}
         \caption{Plot of LHS (green) and RHS (red) of~\eqref{eq:LSimportantrestate} for $\alpha = 10$.}
         \label{fig:LSsurface10}
     \end{subfigure}
        \caption{Plots of LHS (green) and RHS (red) of~\eqref{eq:LSimportantrestate} for $\alpha \in \{1,1.1,3,10\}$, and $N = 2$ and $\gamma = 1/20$. The intersections of the surfaces indicate solutions of~\eqref{eq:LSimportantrestate}. One can see that the solutions for $\alpha = 1$ are not ``good'' in the sense of~\eqref{eq:LSgoodsolutions} because $\theta_{1}$ is small and fixed; this phenomenon was proved by~\citep{long2010random} since $\alpha = 1$ is a convex loss. 
        For $\alpha = 1.1$, one can see the resemblance of $\alpha = 1$ and $\alpha = 3$, and the fact that ``good'' solutions are starting to accumulate.
        For $\alpha = 3$, there are many solutions with diverse $(\theta_{1},\theta_{2})$ values, since the loss is no longer convex.
        ``Good'' solutions for $\alpha = 3$ can be seen where $\theta_{1}$ is positive and large with respect to $\theta_{2}$, i.e., in the middle/right side of the plot.
        For $\alpha = 10$, one can see that the ``good'' solutions have been pushed out further in the parameter space and the two surfaces are starting to separate (reflecting the fact that $\alpha = \infty$ has no solutions).
        Viewing all four plots together, one observes smooth transitions in $\alpha$, indicating that finding a good solution is not difficult.
        }
        \label{fig:LSsurfacesolutions}
\end{figure}
%
The proof of the first property is obtained by invoking one of the results of~\cite{long2010random} for convex, classification-calibrated loss functions.
The remaining properties can be readily shown using standard techniques.

With these nice properties of $B_{N}^{\alpha}$ in hand, we now return to the main thread. 
Using the properties in Lemma~\ref{lem:BNalpha}, we want to reason about the solutions of~\eqref{eq:LSmarimportante}, i.e.,
\begin{align} \label{eq:LSimportantrestate}
B_{N}^{\alpha}(\theta_{1}) &= - 6\gamma B_{N}^{\alpha}(\gamma(\theta_{1} + 5\theta_{2})),
\end{align}
as a function of $\alpha \in (0,\infty]$.
From Propositions~\ref{Prop:alpha-loss-convex} and~\citep{sypherd2022journal}, we know that $\tilde{l}^{\alpha}$ is classification-calibrated for all $\alpha \in (0,\infty]$, convex for $\alpha \leq 1$, and quasi-convex for $\alpha > 1$.
Thus, via~\citep{long2010random}, for each $\hat{\alpha} \leq 1$, there exists some $0 < \gamma_{\hat{\alpha}} < 1/6$ such that 
there exists a solution $(\theta_{1}^{\hat{\alpha}},\theta_{2}^{\hat{\alpha}})$ of~\eqref{eq:LSimportantrestate} which has classification accuracy of $0.5$ (fair coin) on the Long-Servedio dataset.
Without loss of generality, fix $\hat{\alpha} \leq 1$ and its associated pathological $0 < \gamma_{\hat{\alpha}} < 1/6$.

For $\alpha = \infty$, notice that there are \textit{no solutions} to~\eqref{eq:LSimportantrestate} since via the third property in Lemma~\ref{lem:BNalpha},~\eqref{eq:LSimportantrestate} reduces to 
\begin{align} \label{eq:LSinftyBn}
\sigma'(\theta_{1}) (1-N) = - 6\gamma_{\hat{\alpha}} \sigma'(\gamma_{\hat{\alpha}}(\theta_{1} + 5\theta_{2})) (1-N),
\end{align}
which is not satisfied 
because $\sigma'(\theta_{1}) (1-N) < 0$ and $- 6\gamma_{\hat{\alpha}} \sigma'(\gamma_{\hat{\alpha}}(\theta_{1} + 5\theta_{2})) (1-N) > 0$ for all $(\theta_{1},\theta_{2})$; intuitively, the LHS and RHS in~\eqref{eq:LSinftyBn} look like mirrored $\sigma'(x)$ type functions. 

Now, we consider $\alpha \in (1,\infty)$ in~\eqref{eq:LSimportantrestate}, which is the key region of $\alpha$ for the proof.
Examining the LHS of~\eqref{eq:LSimportantrestate}, i.e. $B_{N}^{\alpha}(\theta_{1})$, 
we note from the fourth property of Lemma~\ref{lem:BNalpha} that $\lim\limits_{\theta_{1} \rightarrow +\infty} B_{N}^{\alpha}(\theta_{1}) \rightarrow 0^{+}$.
Furthermore, we note via the sixth property in Lemma~\ref{lem:BNalpha} that 
$B_{N}^{\alpha}(\theta_{1}) > 0$ if and only if $\theta_{1} > \alpha \ln{N}$. 
So, tuning $\alpha \in (1,\infty)$ greater moves the crossover (from negative to positive) of $B_{N}^{\alpha}$ further in $\theta_{1}$. 

We now examine the RHS of~\eqref{eq:LSimportantrestate}, i.e., $- 6 \gamma_{\hat{\alpha}} B_{N}^{\alpha}(\gamma_{\hat{\alpha}}(\theta_{1} + 5\theta_{2}))$.
Set $\theta_{2} = 0$, so we reduce $- 6 \gamma_{\hat{\alpha}} B_{N}^{\alpha}(\gamma_{\hat{\alpha}}(\theta_{1} + 5\theta_{2}))$ to $- 6 \gamma_{\hat{\alpha}} B_{N}^{\alpha}(\gamma_{\hat{\alpha}} \theta_{1})$.
From the fifth property of Lemma~\ref{lem:BNalpha}, 
we have that $\lim\limits_{\theta_{1} \rightarrow \infty} - 6 \gamma_{\hat{\alpha}} B_{N}^{\alpha}(\gamma_{\hat{\alpha}} \theta_{1}) \rightarrow 0^{-}$.
Furthermore, we note via the sixth property in Lemma~\ref{lem:BNalpha} that 
$-6 \gamma_{\hat{\alpha}} B_{N}^{\alpha}(\gamma_{\hat{\alpha}} \theta_{1}) < 0$ if and only if $\theta_{1} > \frac{\alpha \ln{N}}{\gamma_{\hat{\alpha}}}$. 
So, tuning $\alpha \in (1,\infty)$ greater moves the crossover (from positive to negative) of $-6 \gamma_{\hat{\alpha}} B_{N}^{\alpha}(\gamma_{\hat{\alpha}} \theta_{1})$ further in $\theta_{1}$. 

Taking the limit and crossover behaviors in $\theta_{1}$ of $B_{N}^{\alpha}(\theta_{1})$ (the LHS of~\eqref{eq:LSimportantrestate}) and $-6 \gamma_{\hat{\alpha}} B_{N}^{\alpha}(\gamma_{\hat{\alpha}} \theta_{1})$ (the reduced RHS of~\eqref{eq:LSimportantrestate}) together, 
we have \textit{by continuity} that there must exist some $\tilde{\theta}_{1} > 0$ which satisfies 
\begin{align} \label{eq:LSapproximatedsolution}
B_{N}^{\alpha}(\tilde{\theta}_{1}) = -6 \gamma_{\hat{\alpha}} B_{N}^{\alpha}(\gamma_{\hat{\alpha}} \tilde{\theta}_{1}),
\end{align}
for each $\alpha \in (1,\infty)$. 
\begin{figure}[H]
    \centering
    \centerline{\includegraphics[width=.45\linewidth]{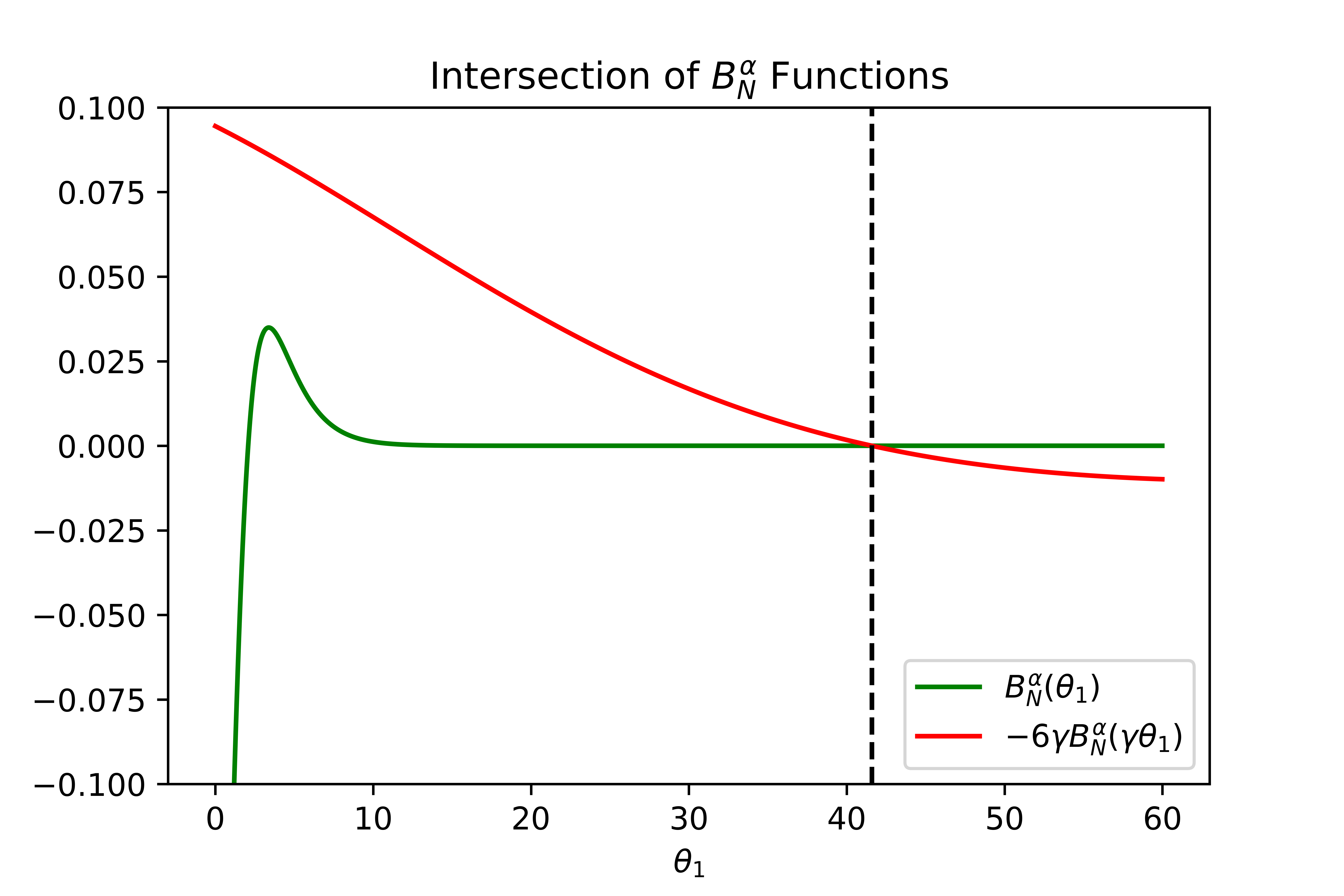}}
    \caption{A plot of LHS (green) and RHS (red) of~\eqref{eq:LSapproximatedsolution} for $\alpha =3$, where $N = 2$, $\gamma = 1/20$. Notice that the intersection point $\tilde{\theta}_{1}$ (dotted line) is \textit{very} close to the crossover point $\frac{\alpha \ln{N}}{\gamma_{\hat{\alpha}}} \approx \frac{3 \times 0.69}{1/20} \approx 41.59$, and also notice that this solution nicely coincides with the grid-search solution presented in Figure~\ref{Fig:LSclassificationlines}.}
     \label{fig:LSfinalstep}
\end{figure}
Furthermore, the choice of $\alpha \in (1,\infty)$ directly influences the magnitude of $\tilde{\theta}_{1} > 0$, with larger $\alpha$ increasing the value of $\tilde{\theta}_{1}$ because of the crossover points, particularly that we require $\tilde{\theta}_{1} > \frac{\alpha \ln{N}}{\gamma_{\hat{\alpha}}}$, which is more restrictive than the requirement that $\tilde{\theta}_{1} > \alpha \ln{N}$, since $0 < \gamma_{\hat{\alpha}} < 1/6$, i.e., $B_{N}^{\alpha}$ is more ``expansive'' when its argument is multiplied by $\gamma_{\hat{\alpha}} < 1/6$.
See Figure~\ref{fig:LSfinalstep} for a plot. 

Therefore, for each $\alpha \in (1,\infty)$, there exists a solution $(\theta_{1}^{\alpha}, \theta_{2}^{\alpha})$ to~\eqref{eq:LSimportantrestate}, where $\theta_{1}^{\alpha} = \tilde{\theta}_{1} > 0$ (indeed, we have that $\theta_{1}^{\alpha} = \mathcal{O}\left(\alpha \gamma_{\hat{\alpha}}^{-1} \ln{\left(p^{-1} - 1 \right)} \right)$) and $\theta_{2}^{\alpha} = 0$, which is a good solution in the sense of~\eqref{eq:LSgoodsolutions} and thus has perfect classification accuracy on the \textit{clean} LS dataset.



Next, while not necessary for the proof of Theorem~\ref{thm:LSmarginalphalossrobust}, we also argue for the existence of other optima near $(\theta_{1}^{\alpha},\theta_{2}^{\alpha})$.
Reconsidering the full (with $\theta_{2}$ included) expression, $- 6 \gamma_{\hat{\alpha}} B_{N}^{\alpha}(\gamma_{\hat{\alpha}}(\theta_{1} + 5\theta_{2}))$ in~\eqref{eq:LSimportantrestate}, we take $\alpha \in (1,\infty)$ large enough in~\eqref{eq:LSapproximatedsolution}
and thus $\tilde{\theta}_{1} > \frac{\alpha \ln{N}}{\gamma_{\hat{\alpha}}}$ is large enough such that
$B_{N}^{\alpha}(\tilde{\theta}_{1}) \approx 0$ and is locally very ``flat'' (as given by the third property in Lemma~\ref{lem:BNalpha}).
Hence, perturbing $\tilde{\theta}_{1}$ slightly induces an extremely slight movement in $B_{N}^{\alpha}(\tilde{\theta}_{1})$.
Now, considering $- 6 \gamma_{\hat{\alpha}} B_{N}^{\alpha}(\gamma_{\hat{\alpha}}(\tilde{\theta}_{1} + 5\theta_{2}))$, 
we fix $\theta_{2}^{*}$ to be \textit{very} small (either positive or negative). 
We then ``wiggle'' $\tilde{\theta}_{1}$ slightly to (potentially) recover a solution $\theta_{1}^{*}$ to 
\begin{align} 
B_{N}^{\alpha}(\theta_{1}^{*}) &= - 6\gamma_{\hat{\alpha}} B_{N}^{\alpha}(\gamma_{\hat{\alpha}}(\theta_{1}^{*} + 5\theta_{2}^{*})),
\end{align}
which (might) exist by continuity. 
See Figure~\ref{fig:LScontours} for a plot; intuitively, the fact that the LHS and RHS of~\eqref{eq:LSapproximatedsolution} intersect, not merely ``touch'', suggests the existence of $(\theta_{1}^{*},\theta_{2}^{*})$, indeed a ``strip'' of good solutions.

\begin{figure}[H]
     \centering
     \begin{subfigure}[b]{0.24\textwidth}
         \centering
         \includegraphics[width=\textwidth]{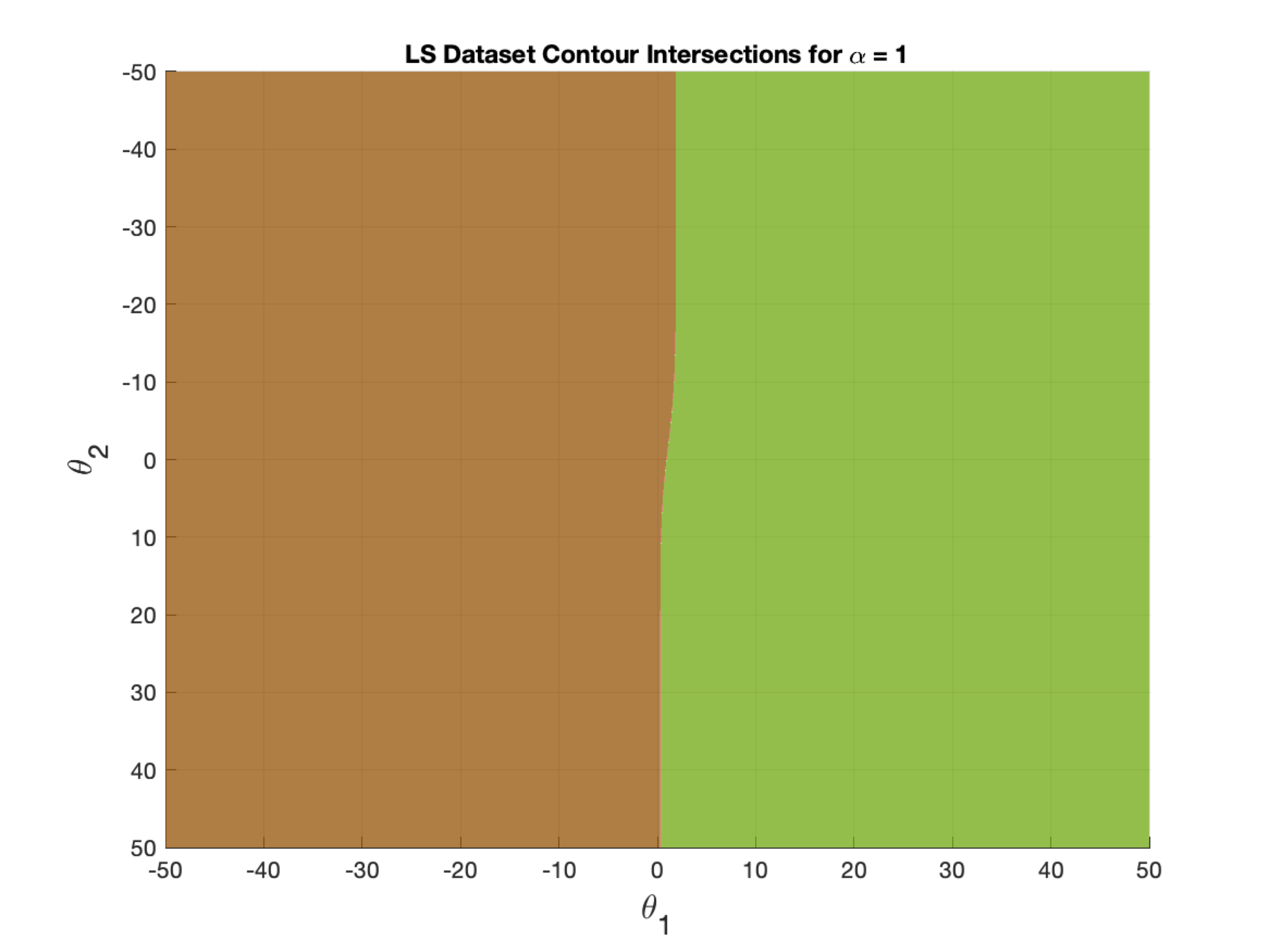}
         \caption{Rotated Figure~\ref{fig:LSsurface1}.}
     \end{subfigure}
     \begin{subfigure}[b]{0.24\textwidth}
         \centering
         \includegraphics[width=\textwidth]{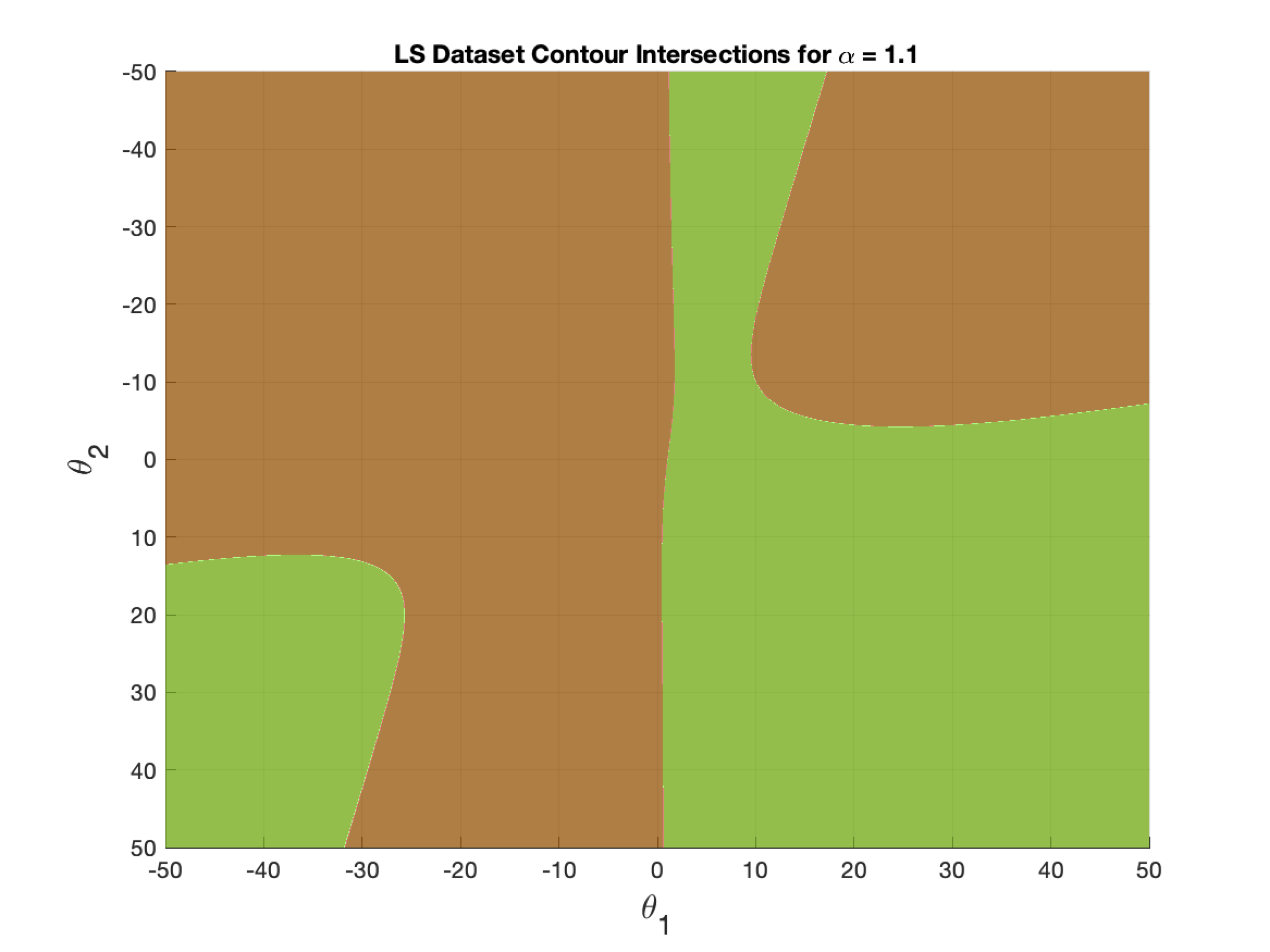}
         \caption{Rotated Figure~\ref{fig:LSsurface1.1}.}
     \end{subfigure}
     \begin{subfigure}[b]{0.24\textwidth}
         \centering
         \includegraphics[width=\textwidth]{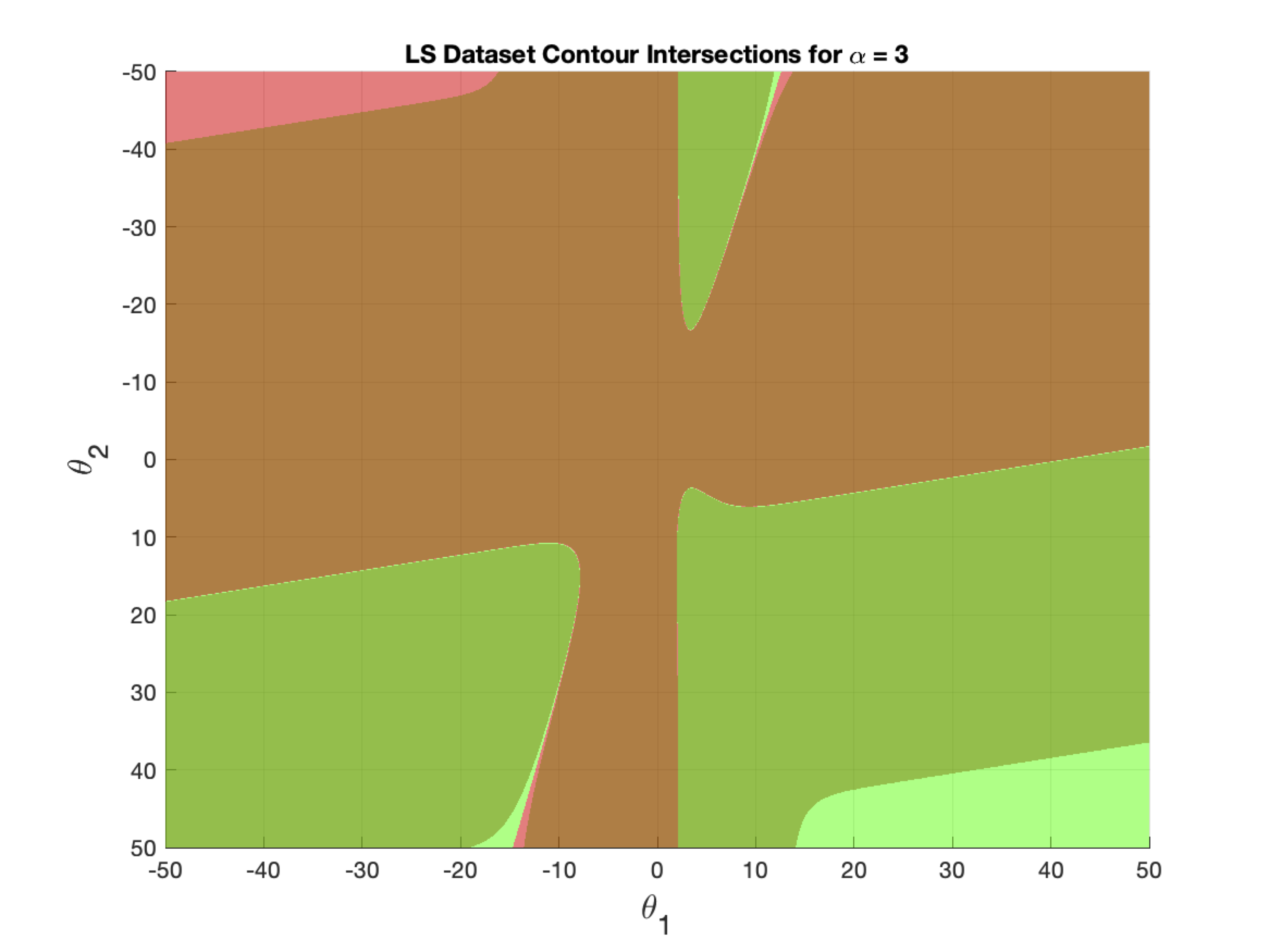}
         \caption{Rotated Figure~\ref{fig:LSsurface3}.}
         \label{fig:lscontour3}
     \end{subfigure}
          \begin{subfigure}[b]{0.24\textwidth}
         \centering
         \includegraphics[width=\textwidth]{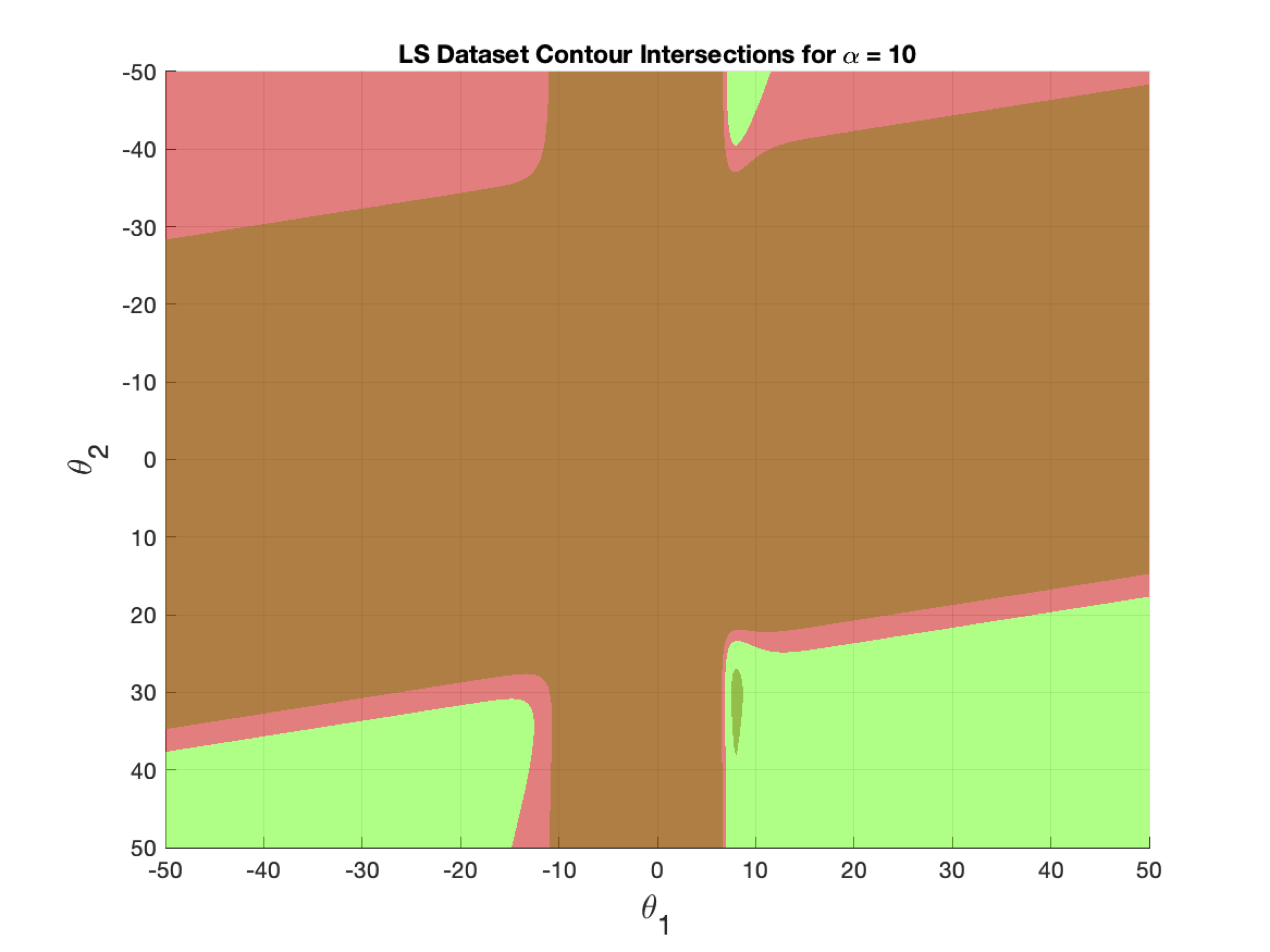}
         \caption{Rotated Figure~\ref{fig:LSsurface10}.}
         \label{fig:lscontour10}
     \end{subfigure}
        \caption{
        Companion figures of Figure~\ref{fig:LSsurfacesolutions} for $\alpha \in \{1,1.1,3,10\}$, and $N = 2$ and $\gamma = 1/20$. The contours indicate solutions of~\eqref{eq:LSimportantrestate}. 
        In Figure~\ref{fig:lscontour3}, one can see a contour of ``good'' LS solutions near where $\theta_{1} \approx 41.59$ and $\theta_{2}$ is very small.
        }
        \label{fig:LScontours}
\end{figure}

\begin{figure}[H]
    \centering
    \includegraphics[width=0.35\linewidth]{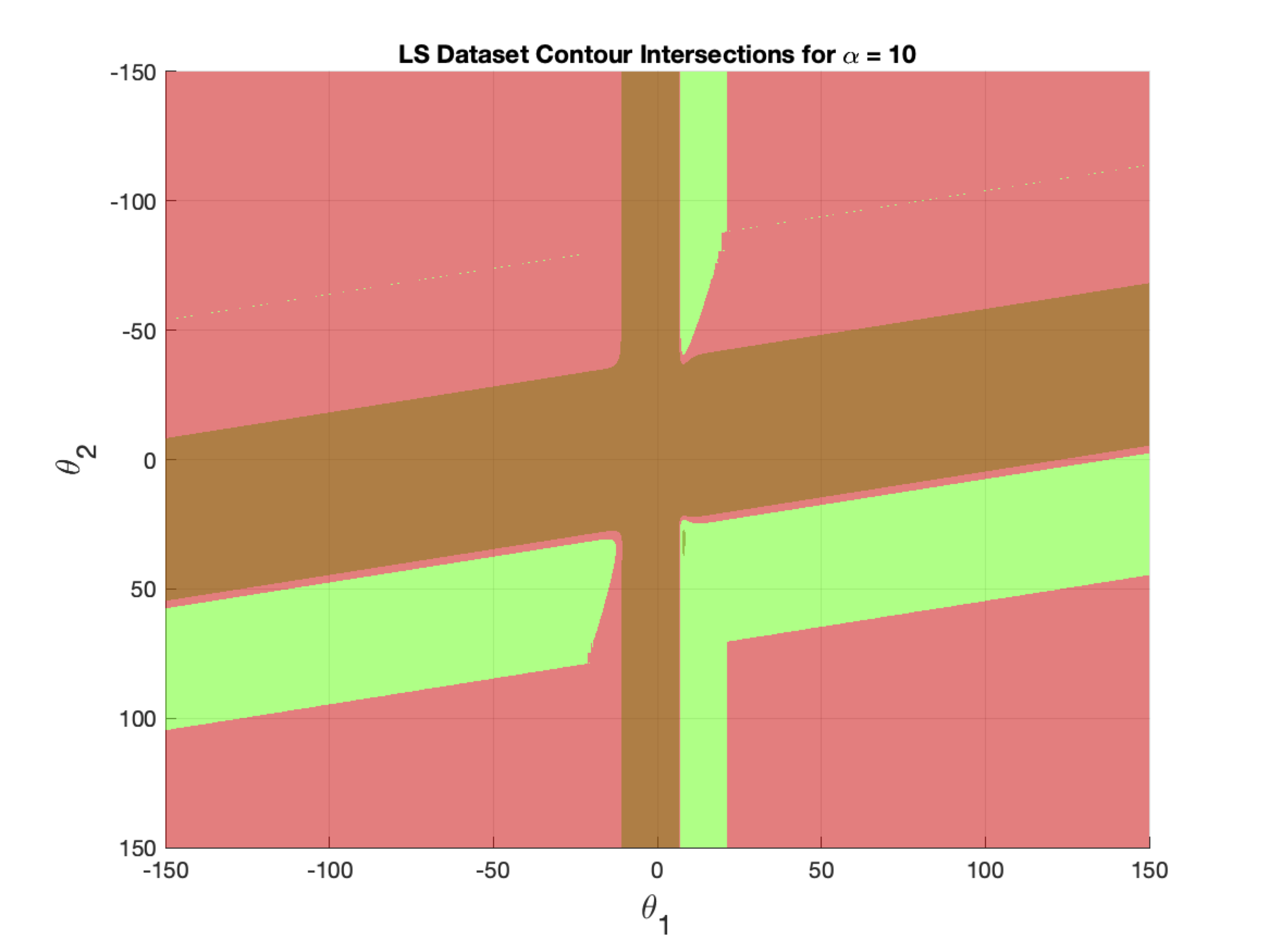}
    \caption{Companion figure of Figure~\ref{fig:lscontour10}, again for $\alpha = 10$, where the parameter space has been \textit{increased}. One can again see ``good'' LS solutions for large $\theta_{1}$ and small $\theta_{2}$.
    This is indicative of a trade off between the value of $\alpha$ and the range of the parameter space for the LS dataset.
    }
    \label{fig:LScontoursexpandedrangealpha=10}
\end{figure}

\subsection{Proof of Theorem~\ref{thm:taylorlagrangeupperbound}} \label{appendix:taylorlagrangeupperbound}

In this section, we provide the proof of Theorem~\ref{thm:taylorlagrangeupperbound}. 
First, however, we provide lemmas useful in the proof of Theorem~\ref{thm:taylorlagrangeupperbound}, which indicate useful bounds for $\alpha = 1$ and $\infty$, and their respective proofs.

\begin{lemma} \label{lemma:secondderivativemarginalphalossbound}
For all $z \in \mathbb{R}$, we have that 
\begin{align}
\left| \dfrac{d^{2}}{dz^{2}} \tilde{l}^{1}(z) \right| \geq \left| \dfrac{d^{2}}{dz^{2}} \tilde{l}^{\infty}(z) \right|,
\end{align}
\end{lemma}
\begin{proof}
Examining $\left| \dfrac{d^{2}}{dz^{2}} \tilde{l}^{1}(z) \right| = \left| \dfrac{d^{2}}{dz^{2}} \tilde{l}^{\infty}(z) \right|$, we have that
\begin{align}
\left| \dfrac{d^{2}}{dz^{2}} \tilde{l}^{1}(z) \right| &= \left| \dfrac{d^{2}}{dz^{2}} \tilde{l}^{\infty}(z) \right| \\
\left| \frac{e^{z}}{(e^{z} + 1)^{2}} \right| &= \left|\frac{e^{z}(e^{z} - 1)}{(e^{z} + 1)^{3}} \right| \\
e^{z} &= \left|\frac{e^{z}(e^{z} - 1)}{e^{z} + 1} \right|,
\end{align}
however, there are no \textit{real} solutions to this equation. Thus, $\left| \dfrac{d^{2}}{dz^{2}} \tilde{l}^{1}(z) \right|$ and $\left| \dfrac{d^{2}}{dz^{2}} \tilde{l}^{\infty}(z) \right|$ do not intersect. 

Considering the large $z > 0$ regime, we find that 
\begin{align}
e^{z} \geq e^{z} - 1,
\end{align}
for all $z \in \mathbb{R}$, where we used the fact that $\lim\limits_{z \rightarrow \infty} \frac{e^{z}(e^{z} - 1)}{e^{z} + 1} = e^{z} - 1$.
Thus, by the Intermediate Value Theorem, we have the desired conclusion. 
\end{proof}

\begin{lemma} \label{lemma:thirdderivativemarginalphalossbound}
For $|z| > \ln{(2)}$, we have that 
\begin{align}
\left|\dfrac{d^{3}}{dz^{3}} \tilde{l}^{\infty}(z) \right| \leq \left|\dfrac{d^{3}}{dz^{3}} \tilde{l}^{1}(z) \right|.
\end{align}
\end{lemma}
\begin{proof}
Consider 
\begin{align}
\left|\dfrac{d^{3}}{dz^{3}} \tilde{l}^{1}(z) \right| = \left| \frac{e^{z} - e^{2z}}{(e^{z} + 1)^{3}} \right|
\end{align}
and 
\begin{align}
\left|\dfrac{d^{3}}{dz^{3}} \tilde{l}^{\infty}(z) \right| = \left| \frac{-e^{3z} + 4 e^{2z} - e^{z}}{(e^{z} + 1)^{4}}  \right|.
\end{align}
Setting 
\begin{align}
\left|\dfrac{d^{3}}{dz^{3}} \tilde{l}^{1}(z) \right| = \left|\dfrac{d^{3}}{dz^{3}} \tilde{l}^{\infty}(z) \right|, 
\end{align}
after some algebra, we find that $z^{*} = \pm \ln{(2)}$.
Furthermore, considering the large $z > 0$ regime, we find that 
\begin{align}
\left|\dfrac{d^{3}}{dz^{3}} \tilde{l}^{\infty}(z) \right| &\overset{?}{\le} \left|\dfrac{d^{3}}{dz^{3}} \tilde{l}^{1}(z) \right| \\
\left| \frac{-e^{3z} + 4 e^{2z} - e^{z}}{(e^{z} + 1)^{4}}  \right| &\overset{?}{\le} \left| \frac{e^{z} - e^{2z}}{(e^{z} + 1)^{3}} \right| \\
\frac{e^{3z} - 4 e^{2z} + e^{z}}{e^{z} + 1} &\overset{?}{\le} e^{2z} - e^{z} \\
e^{2z} - 4 e^{z} &\leq e^{2z} - e^{z},
\end{align}
thus by the IVT and symmetry, we have the desired result. 
\end{proof}

With Lemmas~\ref{lemma:secondderivativemarginalphalossbound} and~\ref{lemma:thirdderivativemarginalphalossbound} in hand, we now present the proof of Theorem~\ref{thm:taylorlagrangeupperbound}.

Recall from~\eqref{eq:gradientofmarginbasedalphaloss} that for $\alpha \in (0,\infty]$
\begin{align}
\nabla_{\theta} \tilde{l}^{\alpha}(\langle YX, \theta \rangle) &= -\sigma(\langle YX, \theta \rangle)^{1-\frac{1}{\alpha}} \sigma(-\langle YX, \theta \rangle) Y X = \tilde{l}^{\alpha'}(\langle YX, \theta \rangle)  YX,
\end{align}
since for each $i \in [d]$, $\dfrac{\partial}{\partial \theta_{i}} \tilde{l}^{\alpha}(\langle YX, \theta \rangle) = \tilde{l}^{\alpha'}(\langle YX, \theta \rangle)  YX_{i}$.

Hence, the gradient of the noisy $\alpha$-risk from~\eqref{eq:noisygradientalpharisklogisticmodel} is  
\begin{align}
\nabla_{\theta} R_{\alpha}^{p}(\theta) &= \mathbb{E}_{X,Y}\left[(1-p)\nabla_{\theta} \tilde{l}^{\alpha}(\langle YX, \theta \rangle) + p \nabla_{\theta} \tilde{l}^{\alpha}(\langle -YX, \theta \rangle) \right] \\
 &= \mathbb{E}_{X,Y}\left[\left( (1-p) \tilde{l}^{\alpha'}(\langle YX, \theta \rangle)  - p \tilde{l}^{\alpha'}(\langle -YX, \theta \rangle)  \right) YX \right], \label{eq:gradientnoisyalphariskusefulform}
\end{align}
where we expanded the expression for clarity. 
Notice that for $\alpha = 1$ (from Lemma~\ref{lem:derivativesmarginalphaloss}), 
\begin{align}
\tilde{l}^{1'}(-z) = - \tilde{l}^{1'}(z) - 1,
\end{align}
namely that $\tilde{l}^{1'}$ is \textit{almost} an \textbf{odd} function, and for $\alpha = \infty$, 
\begin{align}
\tilde{l}^{\infty'}(-z) = \tilde{l}^{\infty'}(z),
\end{align}
namely that $\tilde{l}^{\infty'}$ is an \textbf{even} function. 

Thus, we have by the definition of $\hat{\theta}^{1}$ and $\hat{\theta}^{\infty}$ that for $\alpha = 1$
\begin{align}
\mathbf{0} = \nabla_{\theta} R_{1}^{p}(\hat{\theta}^{1}) &= \mathbb{E}_{X,Y}\left[\left( (1-p) \tilde{l}^{1'}(\langle YX, \hat{\theta}^{1} \rangle)  - p \tilde{l}^{1'}(\langle -YX, \hat{\theta}^{1} \rangle)  \right) YX \right] \\
&= (1-p) \mathbb{E}_{X,Y}\left[\tilde{l}^{1'}(\langle YX, \hat{\theta}^{1} \rangle) YX \right] - p \mathbb{E}_{X,Y} \left[\tilde{l}^{1'}(\langle -YX, \hat{\theta}^{1} \rangle)  YX \right] \\
&= (1-p) \mathbb{E}_{X,Y}\left[\tilde{l}^{1'}(\langle YX, \hat{\theta}^{1} \rangle) YX \right] - p \mathbb{E}_{X,Y} \left[ \left(-\tilde{l}^{1'}(\langle YX, \hat{\theta}^{1} \rangle) - 1 \right)  YX \right] \\
&= \mathbb{E}_{X,Y}\left[\tilde{l}^{1'}(\langle YX, \hat{\theta}^{1} \rangle) YX \right] + p \mathbb{E}_{X,Y}[YX], \label{eq:noisyfirstorderequationforone}
\end{align}
and for $\alpha = \infty$
\begin{align}
\mathbf{0} = \nabla_{\theta} R_{\infty}^{p}(\hat{\theta}^{\infty}) &= \mathbb{E}_{X,Y}\left[\left( (1-p) \tilde{l}^{\infty'}(\langle YX, \hat{\theta}^{\infty} \rangle) - p \tilde{l}^{\infty'}(\langle -YX, \hat{\theta}^{\infty} \rangle) \right) YX \right] \\
&= (1-p) \mathbb{E}_{X,Y}\left[\tilde{l}^{\infty'}(\langle YX, \hat{\theta}^{\infty} \rangle) YX \right] - p \mathbb{E}_{X,Y} \left[\tilde{l}^{\infty'}(\langle -YX, \hat{\theta}^{\infty} \rangle) YX \right] \\
&= (1-2p) \mathbb{E}_{X,Y}\left[\tilde{l}^{\infty'}(\langle YX, \hat{\theta}^{\infty} \rangle) YX \right]. \label{eq:noisyfirstorderequationforinfinity}
\end{align}
And, thus we have that for each $i \in [d]$, 
\begin{align}
\mathbb{E}_{X,Y}\left[ \tilde{l}^{1'}(\langle YX, \hat{\theta}^{1} \rangle) YX_{i} \right] + p \mathbb{E}_{X,Y}[YX_{i}] = 0,
\end{align}
and 
\begin{align}
\mathbb{E}_{X,Y}\left[ \tilde{l}^{\infty'}(\langle YX, \hat{\theta}^{\infty} \rangle) YX_{i} \right] = 0.
\end{align}

In order to evaluate the efficacy of the gradient of the noisy $\alpha$-risk at recovering any data generating vector $\theta^{*} \in \mathbb{B}_{d}(r)$, 
we seek to upper bound $\|\nabla_{\theta} R_{1}^{p}(\theta^{*}) \|_{\infty}$ and $\|\nabla_{\theta} R_{\infty}^{p}(\theta^{*}) \|_{\infty}$.  
To this end, recall the Taylor-Lagrange equality~\citep{kline1998calculus} for a twice continuously differentiable $f: \mathbb{R} \to \mathbb{R}$,
\begin{align}
f(b) = f(a) + (b-a)f'(a) + \frac{(b-a)^{2}}{2} f''(c),
\end{align}
where $c \in [a,b]$. 

Let $i \in [d]$ be arbitrary, but fixed. 
From~\eqref{eq:gradientnoisyalphariskusefulform} (and the reductions from~\eqref{eq:noisyfirstorderequationforone} and~\eqref{eq:noisyfirstorderequationforinfinity}) we have that at $\theta^{*} \in \mathbb{B}_{d}(r)$
\begin{align} \label{eq:partialderivativeofonelossatthetastar}
\dfrac{\partial}{\partial \theta_{i}} R_{1}^{p}(\theta^{*}) = \mathbb{E}_{X,Y}\left[\tilde{l}^{1'}(\langle YX, \theta^{*} \rangle) YX_{i} \right] + p \mathbb{E}_{X,Y}[YX_{i}],
\end{align}
and 
\begin{align} \label{eq:partialderivativeofinfinitylossatthetastar}
\dfrac{\partial}{\partial \theta_{i}} R_{\infty}^{p}(\theta^{*}) = (1-2p) \mathbb{E}_{X,Y}\left[ \tilde{l}^{\infty'}(\langle YX, \theta^{*} \rangle) YX_{i} \right].
\end{align}

Using the Taylor-Lagrange equality, we let $f = \tilde{l}^{\alpha'}$ (where $\alpha = 1$ or $\infty$ for simplicity for the time being), and thus we have that for each $(X,Y) \in \mathcal{X} \times \{-1,+1\}$, 
\begin{align}
\tilde{l}^{\alpha'}(b_{(X,Y)}) = \tilde{l}^{\alpha'}(a_{(X,Y)}) + (b_{(X,Y)}-a_{(X,Y)})\tilde{l}^{\alpha''}(a_{(X,Y)}) + \frac{(b_{(X,Y)}-a_{(X,Y)})^{2}}{2}\tilde{l}^{\alpha'''}(c_{(X,Y)}^{\alpha}),
\end{align}
where $b_{(X,Y)} = \langle YX, \theta^{*} \rangle$ and $a_{(X,Y)} = \langle YX, \hat{\theta}^{\alpha} \rangle$, hence $c_{(X,Y)}^{\alpha} \in [\langle YX, \hat{\theta}^{\alpha} \rangle, \langle YX, \theta^{*} \rangle]$. 
Examining each of~\eqref{eq:partialderivativeofonelossatthetastar} (first term) and~\eqref{eq:partialderivativeofinfinitylossatthetastar} (without coefficient) with the Taylor-Lagrange equality, we have that 
\begin{align}
\nonumber &\mathbb{E}_{X,Y}\left[ \tilde{l}^{\alpha'}(\langle YX, \theta^{*} \rangle) YX_{i} \right] \\
&\quad\quad = \mathbb{E}_{X,Y} \left[ \left( \tilde{l}^{\alpha'}(a_{(X,Y)}) + (b_{(X,Y)}-a_{(X,Y)})\tilde{l}^{\alpha''}(a_{(X,Y)}) + \frac{(b_{(X,Y)}-a_{(X,Y)})^{2}}{2}\tilde{l}^{\alpha'''}(c_{(X,Y)}^{\alpha}) \right) Y X_{i} \right].
\end{align}

Thus, for $\alpha = 1$, we have that 
\begin{align}
&\dfrac{\partial}{\partial \theta_{i}} R_{1}^{p}(\theta^{*}) = \mathbb{E}_{X,Y}\left[\tilde{l}^{1'}(\langle YX, \theta^{*} \rangle) YX_{i} \right] + p \mathbb{E}_{X,Y}[YX_{i}] \\
\nonumber &= p \mathbb{E}_{X,Y}[YX_{i}] \\
& + \mathbb{E}_{X,Y}\left[ \left( \tilde{l}^{1'}(\langle YX, \hat{\theta}^{1} \rangle) + (\langle YX, \theta^{*} \rangle-\langle YX, \hat{\theta}^{1} \rangle)\tilde{l}^{1''}(\langle YX, \hat{\theta}^{1} \rangle) + \frac{(\langle YX, \theta^{*} \rangle - \langle YX, \hat{\theta}^{1} \rangle)^{2}}{2}\tilde{l}^{1'''}(c^{1}_{(X,Y)}) \right) Y X_{i} \right],
\end{align}
where $c^{1}_{(X,Y)} \in [\langle YX, \hat{\theta}^{1} \rangle, \langle YX, \theta^{*} \rangle]$. 
Noticing that $\mathbb{E}_{X,Y} \left[\tilde{l}^{1'}(\langle YX, \hat{\theta}^{1} \rangle) YX_{i} \right] + p \mathbb{E}_{X,Y}[YX_{i}] = 0$,  we thus obtain 
\begin{align} \label{eq:taylorlagrangedone}
\dfrac{\partial}{\partial \theta_{i}} R_{1}^{p}(\theta^{*}) = \mathbb{E}_{X,Y}\left[\left((\langle YX, \theta^{*} \rangle-\langle YX, \hat{\theta}^{1} \rangle)\tilde{l}^{1''}(\langle YX, \hat{\theta}^{1} \rangle) + \frac{(\langle YX, \theta^{*} \rangle - \langle YX, \hat{\theta}^{1} \rangle)^{2}}{2}\tilde{l}^{1'''}(c^{1}_{(X,Y)}) \right) Y X_{i} \right].
\end{align}
Using similar steps, we can also obtain 
\begin{align} \label{eq:taylorlagrangedinfinity}
\nonumber &\dfrac{\partial}{\partial \theta_{i}} R_{\infty}^{p}(\theta^{*}) \\
&= (1-2p) \mathbb{E}_{X,Y}\left[\left((\langle YX, \theta^{*} \rangle-\langle YX, \hat{\theta}^{\infty} \rangle)\tilde{l}^{\infty''}(\langle YX, \hat{\theta}^{\infty} \rangle) + \frac{(\langle YX, \theta^{*} \rangle - \langle YX, \hat{\theta}^{\infty} \rangle)^{2}}{2}\tilde{l}^{\infty'''}(c^{\infty}_{(X,Y)}) \right) Y X_{i} \right], 
\end{align}
where $c^{\infty}_{(X,Y)} \in [\langle YX, \hat{\theta}^{\infty} \rangle, \langle YX, \theta^{*} \rangle]$ and we note a difference between~\eqref{eq:taylorlagrangedone} and~\eqref{eq:taylorlagrangedinfinity}, i.e. the latter has the $1-2p$ coefficient.

Now, we consider $\left|\dfrac{\partial}{\partial \theta_{i}} R_{1}^{p}(\theta^{*})\right|$ and seek an upperbound. We have that from~\eqref{eq:taylorlagrangedone}
\begin{align}
\left| \dfrac{\partial}{\partial \theta_{i}} R_{1}^{p}(\theta^{*}) \right| &= \left| \mathbb{E}_{X,Y}\left[\left((\langle YX, \theta^{*} \rangle-\langle YX, \hat{\theta}^{1} \rangle)\tilde{l}^{1''}(\langle YX, \hat{\theta}^{1} \rangle) + \frac{(\langle YX, \theta^{*} \rangle - \langle YX, \hat{\theta}^{1} \rangle)^{2}}{2}\tilde{l}^{1'''}(c^{1}_{(X,Y)}) \right) Y X_{i} \right] \right| \\
&\leq \mathbb{E}_{X,Y}\left[ \left| \left((\langle YX, \theta^{*} \rangle-\langle YX, \hat{\theta}^{1} \rangle)\tilde{l}^{1''}(\langle YX, \hat{\theta}^{1} \rangle) + \frac{(\langle YX, \theta^{*} \rangle - \langle YX, \hat{\theta}^{1} \rangle)^{2}}{2}\tilde{l}^{1'''}(c^{1}_{(X,Y)}) \right) Y X_{i} \right| \right] \\
&= \mathbb{E}_{X,Y}\left[ \left|X_{i} \right| \left| (\langle YX, \theta^{*} \rangle-\langle YX, \hat{\theta}^{1} \rangle)\tilde{l}^{1''}(\langle YX, \hat{\theta}^{1} \rangle) + \frac{(\langle YX, \theta^{*} \rangle - \langle YX, \hat{\theta}^{1} \rangle)^{2}}{2}\tilde{l}^{1'''}(c^{1}_{(X,Y)}) \right| \right] \\
&\leq \mathbb{E}_{X,Y}\left[ \left|X_{i} \right| \left(  \left| (\langle YX, \theta^{*} \rangle-\langle YX, \hat{\theta}^{1} \rangle)\tilde{l}^{1''}(\langle YX, \hat{\theta}^{1} \rangle) \right| + \left| \frac{(\langle YX, \theta^{*} \rangle - \langle YX, \hat{\theta}^{1} \rangle)^{2}}{2}\tilde{l}^{1'''}(c^{1}_{(X,Y)}) \right| \right) \right],
\end{align}
where we used Jensen's inequality via the absolute value, the triangle inequality, and the fact that $|ab| = |a|\cdot|b|$.
Continuing,
\begin{align}
&\mathbb{E}_{X,Y}\left[ \left|X_{i} \right| \left(  \left| (\langle YX, \theta^{*} \rangle-\langle YX, \hat{\theta}^{1} \rangle)\tilde{l}^{1''}(\langle YX, \hat{\theta}^{1} \rangle) \right| + \left| \frac{(\langle YX, \theta^{*} \rangle - \langle YX, \hat{\theta}^{1} \rangle)^{2}}{2}\tilde{l}^{1'''}(c^{1}_{(X,Y)}) \right| \right) \right] \\
&= \mathbb{E}_{X,Y}\left[ \left|X_{i} \right| \left(  \left| \langle YX, \theta^{*} - \hat{\theta}^{1} \rangle  \right| \left|\tilde{l}^{1''}(\langle YX, \hat{\theta}^{1} \rangle) \right| + \frac{\langle YX, \theta^{*} - \hat{\theta}^{1} \rangle^{2}}{2} \left| \tilde{l}^{1'''}(c^{1}_{(X,Y)}) \right| \right) \right] \\
&\leq \mathbb{E}_{X,Y}\left[ \left|X_{i} \right| \left(  \left\| YX \right\| \left\| \theta^{*} - \hat{\theta}^{1} \right\| \left|\tilde{l}^{1''}(\langle YX, \hat{\theta}^{1} \rangle) \right| + \frac{\left\|YX \right\|^{2} \left\|\theta^{*} - \hat{\theta}^{1}\right\|^{2}}{2} \left| \tilde{l}^{1'''}(c^{1}_{(X,Y)}) \right| \right) \right],
\end{align}
where we used the Cauchy-Schwarz inequality on both inner products. 
Next, we use the observation that $X \in [0,1]^{d}$, and thus $\|X\| \leq \sqrt{d}$, and that $\theta^{*} - \theta \in \mathbb{B}_{d}(2r)$, for all $\theta \in \mathbb{B}_{d}(r)$.
Thus, we have that 
\begin{align}
&\mathbb{E}_{X,Y}\left[ \left|X_{i} \right| \left(  \left\| YX \right\| \left\| \theta^{*} - \hat{\theta}^{1} \right\| \left|\tilde{l}^{1''}(\langle YX, \hat{\theta}^{1} \rangle) \right| + \frac{\left\|YX \right\|^{2} \left\|\theta^{*} - \hat{\theta}^{1}\right\|^{2}}{2} \left| \tilde{l}^{1'''}(c^{1}_{(X,Y)}) \right| \right) \right] \\
&\leq  \mathbb{E}_{X,Y}\left[ \sqrt{d} 2r \left|\tilde{l}^{1''}(\langle YX, \hat{\theta}^{1} \rangle) \right| + \frac{4 d r^{2}}{2} \left| \tilde{l}^{1'''}(c^{1}_{(X,Y)}) \right| \right] \\
&= 2  d^{1/2} r \mathbb{E}_{X,Y}\left[ \left|\tilde{l}^{1''}(\langle YX, \hat{\theta}^{1} \rangle) \right| \right] + 2d r^{2} \mathbb{E}_{X,Y} \left[\left| \tilde{l}^{1'''}(c^{1}_{(X,Y)}) \right| \right].
\end{align}
Thus, we obtain that 
\begin{align} \label{eq:upperboundradiusone}
\left| \dfrac{\partial}{\partial \theta_{i}} R_{1}^{p}(\theta^{*}) \right| \leq 2d^{1/2} r \mathbb{E}_{X,Y}\left[ \left|\tilde{l}^{1''}(\langle YX, \hat{\theta}^{1} \rangle) \right| \right] + 2d r^{2} \mathbb{E}_{X,Y} \left[\left| \tilde{l}^{1'''}(c^{1}_{(X,Y)}) \right| \right].
\end{align}
For $\alpha =\infty$, the exact same steps go through, so we also have that 
\begin{align} \label{eq:upperboundradiusinfinity}
\left| \dfrac{\partial}{\partial \theta_{i}} R_{\infty}^{p}(\theta^{*}) \right| \leq (1-2p) \left(2d^{1/2} r \mathbb{E}_{X,Y}\left[ \left|\tilde{l}^{\infty''}(\langle YX, \hat{\theta}^{\infty} \rangle) \right| \right] + 2d r^{2} \mathbb{E}_{X,Y} \left[\left| \tilde{l}^{\infty'''}(c^{\infty}_{(X,Y)}) \right| \right]\right).
\end{align}
Considering $\mathbb{E}_{X,Y}\left[ \left|\tilde{l}^{1''}(\langle YX, \hat{\theta}^{1} \rangle) \right| \right]$ in~\eqref{eq:upperboundradiusone}, 
we let 
\begin{align}
z_{1}^{*} = \argmax\limits_{z \in \{\langle yx, \hat{\theta}^{1} \rangle | (x,y) \in \mathcal{X} \times \{-1,+1\} \}} \left|\tilde{l}^{1''}(z) \right|,
\end{align}
and we thus obtain $\mathbb{E}_{X,Y}\left[ \left|\tilde{l}^{1''}(\langle YX, \hat{\theta}^{1} \rangle) \right| \right] \leq \left|\tilde{l}^{1''}(z_{1}^{*}) \right|$,
where we note that $z_{1}^{*} > \ln{(2+\sqrt{3})}$ by assumption.
Similarly, considering $\mathbb{E}_{X,Y}\left[ \left|\tilde{l}^{\infty''}(\langle YX, \hat{\theta}^{\infty} \rangle) \right| \right]$
in~\eqref{eq:upperboundradiusinfinity}, we let 
\begin{align}
z_{\infty}^{*} = \argmax\limits_{z \in \{\langle yx, \hat{\theta}^{\infty} \rangle | (x,y) \in \mathcal{X} \times \{-1,+1\} \}} \left|\tilde{l}^{\infty''}(z) \right|,
\end{align}
and we thus obtain $\mathbb{E}_{X,Y}\left[ \left|\tilde{l}^{\infty''}(\langle YX, \hat{\theta}^{\infty} \rangle) \right| \right] \leq \left|\tilde{l}^{\infty''}(z_{\infty}^{*}) \right|$,
where $z_{\infty}^{*} \geq z_{1}^{*} > \ln{(2+\sqrt{3})}$ again by assumption.

Indeed, since $\left| \tilde{l}^{1'''}(z) \right|$ and $\left| \tilde{l}^{\infty'''}(z) \right|$ are monotonically decreasing for $z > \ln{(2+\sqrt{3})}$ we also have that 
\begin{align}
\mathbb{E}_{X,Y} \left[\left| \tilde{l}^{1'''}(c^{1}_{(X,Y)}) \right| \right] \leq \left| \tilde{l}^{1'''}(z_{1}^{*}) \right|,
\end{align} 
and 
\begin{align}
\mathbb{E}_{X,Y} \left[\left| \tilde{l}^{\infty'''}(c^{\infty}_{(X,Y)}) \right| \right] \leq \left| \tilde{l}^{\infty'''}(z_{\infty}^{*}) \right|.
\end{align} 

Next, we invoke Lemma~\ref{lemma:secondderivativemarginalphalossbound}, i.e., that for all $z \in \mathbb{R}$, 
\begin{align}
\left| \dfrac{d^{2}}{dz^{2}} \tilde{l}^{1}(z) \right| \geq \left| \dfrac{d^{2}}{dz^{2}} \tilde{l}^{\infty}(z) \right|,
\end{align}
and Lemma~\ref{lemma:thirdderivativemarginalphalossbound}, i.e., that for $z > \ln{2}$, 
\begin{align}
\left|\dfrac{d^{3}}{dz^{3}} \tilde{l}^{\infty}(z) \right| \leq \left|\dfrac{d^{3}}{dz^{3}} \tilde{l}^{1}(z) \right|.
\end{align}
Thus, we have that (also by monotonically decreasing) 
\begin{align}
\left| \tilde{l}^{\infty''}(z_{\infty}^{*}) \right| \leq \left| \tilde{l}^{1''}(z_{1}^{*}) \right|,
\end{align}
and 
\begin{align}
\left| \tilde{l}^{\infty'''}(z_{\infty}^{*}) \right| \leq \left| \tilde{l}^{1'''}(z_{1}^{*}) \right|.
\end{align}

Hence, since the bounds on~\eqref{eq:upperboundradiusone} and~\eqref{eq:upperboundradiusinfinity} hold for all $i \in [d]$, we obtain the desired result.

\subsection{Proof of Theorem~\ref{thm:lowerbound}} \label{appendix:uniformlowerbound}
The strategy of the proof is to upperbound and lowerbound $\|\nabla_{\theta} R_{\alpha}^{p}(\theta) - \mathbb{E}[X^{[1]}]\|$. For the lowerbound, we use the reverse triangle inequality. Combining the upper and lowerbounds, we then rewrite the bounded expressions to induce a lowerbound on $\|\nabla_{\theta} R_{\alpha}^{p}(\theta)\|$ itself.
For notational convenience, we used $\gamma = C_{p,r\sqrt{d},\alpha}$ in the main body.

Now, for each $y\in\{-1,1\}$, let $X^{[y]}$ denote the random variable having the distribution of $X$ conditioned on $Y=y$. 
We further assume that $X^{[1]} \stackrel{\textnormal{d}}{=} -X^{[-1]}$, $\mathbb{E}[X^{[1]}]\neq0$, namely, a skew-symmetric distribution. 
Examining the gradient of the noisy $\alpha$-risk (under the skew-symmetric distribution), we have that ($P_{1} = \mathbb{P}[Y=1]$)
\begin{align}
\nabla_{\theta} R_{\alpha}^{p}(\theta) &= \mathbb{E}_{X,Y}\bigg[\bigg( p Y g_{\theta}(-YX)^{1-1/\alpha} g_{\theta}(YX)  - (1-p)Y g_{\theta}(YX)^{1-1/\alpha} g_{\theta}(-YX) \bigg) X \bigg] \\
\nonumber &= P_{1} \mathbb{E}_{X^{[1]}} \left[\left(p g_{\theta}(-X^{[1]})^{1-1/\alpha} g_{\theta}(X^{[1]}) - (1-p) g_{\theta}(X^{[1]})^{1-1/\alpha} g_{\theta}(-X^{[1]})  \right) X^{[1]} \right] \\
& \quad\quad + P_{-1} \mathbb{E}_{X^{[-1]}} \left[\left(-p g_{\theta}(X^{[-1]})^{1-\frac{1}{\alpha}} g_{\theta}(-X^{[-1]}) + (1-p) g_{\theta}(-X^{[-1]})^{1-\frac{1}{\alpha}} g_{\theta}(X^{[-1]}) \right) X^{[-1]} \right] \\
&= \mathbb{E}_{X^{[1]}} \left[\left(p g_{\theta}(-X^{[1]})^{1-1/\alpha} g_{\theta}(X^{[1]}) - (1-p) g_{\theta}(X^{[1]})^{1-1/\alpha} g_{\theta}(-X^{[1]})  \right) X^{[1]} \right].
\end{align}
First considering the upperbound on $\|\nabla_{\theta} R_{\alpha}^{p}(\theta) - \mathbb{E}[X^{[1]}]\|$, we have that 
\begin{align}
&\|\mathbb{E}_{X^{[1]}} \left[\left(p g_{\theta}(-X^{[1]})^{1-1/\alpha} g_{\theta}(X^{[1]}) - (1-p) g_{\theta}(X^{[1]})^{1-1/\alpha} g_{\theta}(-X^{[1]})  \right) X^{[1]} \right] - \mathbb{E}[X^{[1]}] \| \\
&= \|\mathbb{E}_{X^{[1]}} \left[\left(p g_{\theta}(-X^{[1]})^{1-1/\alpha} g_{\theta}(X^{[1]}) - (1-p) g_{\theta}(X^{[1]})^{1-1/\alpha} g_{\theta}(-X^{[1]}) - 1  \right) X^{[1]} \right] \| \\
&= \|\mathbb{E}_{X^{[1]}} \left[\left(p g_{\theta}(-X^{[1]})^{1-1/\alpha} g_{\theta}(X^{[1]}) - p - (1-p) g_{\theta}(X^{[1]})^{1-1/\alpha} g_{\theta}(-X^{[1]}) - (1-p)  \right) X^{[1]} \right] \| \\ 
&= \|\mathbb{E}_{X^{[1]}} \left[\left(p \left( g_{\theta}(-X^{[1]})^{1-1/\alpha} g_{\theta}(X^{[1]}) - 1 \right) - (1-p) \left( g_{\theta}(X^{[1]})^{1-1/\alpha} g_{\theta}(-X^{[1]}) - 1 \right)  \right) X^{[1]} \right] \| \\ 
\label{eq:lowerboundupperbound} &\leq \mathbb{E}_{X^{[1]}} \left[ \left| p \left( g_{\theta}(-X^{[1]})^{1-1/\alpha} g_{\theta}(X^{[1]}) - 1 \right) - (1-p) \left( g_{\theta}(X^{[1]})^{1-1/\alpha} g_{\theta}(-X^{[1]}) - 1 \right)  \right| \| X^{[1]} \| \right],
\end{align}
where we used Jensen's inequality due to the convexity of the norm.


We now consider the term in absolute value above, which we rewrite for simplicity as 
\begin{align}
f_{\alpha,p}(x) := p \left(\sigma(-x)^{1-\frac{1}{\alpha}} \sigma(x) - 1 \right) - (1-p)\left(\sigma(x)^{1-\frac{1}{\alpha}} \sigma(-x) - 1 \right).
\end{align}
We examine
\begin{align}
\dfrac{\partial}{\partial \alpha} f_{\alpha,p}(x) = (1-p) \frac{\sigma(x)^{1-\frac{1}{\alpha}} \log{\left(e^{-x} + 1 \right)}}{(e^{x} + 1)\alpha^{2}} - p \frac{\sigma(-x)^{1-\frac{1}{\alpha}} \log{(e^{x} + 1)}}{(e^{-x} + 1)\alpha^{2}},
\end{align}
which follows from the fact that 
\begin{align}
\dfrac{\partial}{\partial \alpha} \sigma(x)^{1-\frac{1}{\alpha}} \sigma(-x) = \frac{\sigma(x)^{1-\frac{1}{\alpha}} \log{(\sigma(x))}}{(e^{x} + 1)\alpha^{2}}.
\end{align}
Considering $x > 0$ and $0 < p < 1/2$, one can show that
\begin{align}
\dfrac{\partial}{\partial \alpha} f_{\alpha,p}(x) > 0
\end{align}
is equivalent to 
\begin{align} \label{eq:lowerboundmonotonictyassumption}
\left(\frac{1}{p} - 1 \right) > e^{\frac{x}{\alpha}} \frac{\log{(e^{x} + 1)}}{\log{(e^{-x} + 1)}},
\end{align}
and it can be shown that the term on the right-hand-side is monotonically increasing in $x > 0$ for $\alpha \in [1,\infty]$.
Hence choosing $x > 0$ (i.e., $r > 0$) small enough ensures that $f_{\alpha,p}(x)$ is monotonically increasing in $\alpha \in [1,\infty]$.
Furthermore, since $\dfrac{\partial}{\partial x} f_{\alpha,p}(x) > 0$ for $x > 0$, $p < 1/2$, and $\alpha \in [1,\infty]$, and
$X \in [0,1]^{d}$, $\theta \in \mathbb{B}_{d}(r)$, we have by the Cauchy-Schwarz inequality (i.e., $\langle \theta, X \rangle \leq r \sqrt{d}$) that 
\begin{align}
&p \left( g_{\theta}(-X^{[1]})^{1-1/\alpha} g_{\theta}(X^{[1]}) - 1 \right) - (1-p) \left( g_{\theta}(X^{[1]})^{1-1/\alpha} g_{\theta}(-X^{[1]}) - 1 \right) \\
\label{eq:lowerboundcera} & \quad\quad\quad\quad \leq p \left(\sigma(-r \sqrt{d})^{1-\frac{1}{\alpha}} \sigma(r \sqrt{d}) - 1 \right) - (1-p)\left(\sigma(r \sqrt{d})^{1-\frac{1}{\alpha}} \sigma(-r\sqrt{d}) - 1 \right) =: C_{p, r \sqrt{d}, \alpha}.
\end{align}
Note that $C_{p, r\sqrt{d}, 1} := \sigma(r\sqrt{d}) - p > 0$ (since $r\sqrt{d} > 0$ and $p < 1/2$), and $C_{p, r\sqrt{d}, \infty} := (1-2p)(1-\sigma'(r\sqrt{d}))$, 
and by the restriction on $r>0$~\eqref{eq:lowerboundmonotonictyassumption}, we have that for $\alpha \in (1,\infty)$
\begin{align}
0 < C_{p, r\sqrt{d}, 1} \leq C_{p, r\sqrt{d}, \alpha} \leq C_{p, r\sqrt{d}, \infty}.
\end{align}
Thus, considering the upperbound on $\|\nabla_{\theta} R_{\alpha}^{p}(\theta) - \mathbb{E}[X^{[1]}]\|$ in~\eqref{eq:lowerboundupperbound}, 
we have that 
\begin{align} \label{eq:lowerboundupperfinal}
\|\nabla_{\theta} R_{\alpha}^{p}(\theta) - \mathbb{E}[X^{[1]}]\| \leq C_{p,r\sqrt{d},\alpha} \mathbb{E}_{X^{[1]}}[ \| X^{[1]} \|],
\end{align}
where $C_{p,r\sqrt{d},\alpha}$ is given in~\eqref{eq:lowerboundcera}.

Now, considering a lowerbound on $\|\nabla_{\theta} R_{\alpha}^{p}(\theta) - \mathbb{E}[X^{[1]}]\|$, via the reverse triangle inequality we have that 
\begin{align}
\|\nabla_{\theta} R_{\alpha}^{p}(\theta) - \mathbb{E}[X^{[1]}]\| \geq \|\mathbb{E}[X^{[1]}]\| - \|\nabla_{\theta} R_{\alpha}^{p}(\theta)\|.
\end{align}
Combining this with our derived upperbound~\eqref{eq:lowerboundupperfinal}, we have that 
\begin{align}
C_{p,r\sqrt{d},\alpha} \mathbb{E}[\|X^{[1]}\|] \geq \|\mathbb{E}[X^{[1]}]\| - \|\nabla_{\theta} R_{\alpha}^{p}(\theta)\|.
\end{align}
Rewriting, we have that 
\begin{align}
\|\nabla_{\theta} R_{\alpha}^{p}(\theta)\| \geq \|\mathbb{E}[X^{[1]}]\| - C_{p,r\sqrt{d},\alpha} \mathbb{E}[\|X^{[1]}\|].
\end{align}
Using our observation earlier regarding the monotonically increasing property of $C_{p,r\sqrt{d},\alpha}$ in $\alpha \in [1,\infty]$,
we can write that 
\begin{align}
\nonumber \|\nabla_{\theta} R_{\alpha}^{p}(\theta)\| &\geq \|\mathbb{E}[X^{[1]}]\| - C_{p,r\sqrt{d},\alpha} \mathbb{E}[\|X^{[1]}\|] \\
&\geq \|\mathbb{E}[X^{[1]}]\| - (1-2p) \left(1 - \sigma'(r\sqrt{d}) \right) \mathbb{E}[\|X^{[1]}\|] >0,
\end{align}
which is nonnegative by distributional assumption on the skew-symmetric distribution itself, namely we assume that 
\begin{align}
(1-2p)(1-\sigma'(r\sqrt{d})) < \frac{\|\mathbb{E}(X^{[1]})\|}{\mathbb{E}(\|X^{[1]}\|)}.
\end{align}


\section{Further Experimental Results and Details} \label{appendix:experiments}



\subsection{Boosting Experiments} \label{appendix:boostingexperiments}
\subsubsection{Long-Servedio}
\paragraph{Dataset} The Long-Servedio dataset is a synthetic dataset which was first suggested in~\citep{long2010random} and also considered in~\citep{cheamanunkul2014non}. 
The dataset has input $x \in \mathbb{R}^{21}$ (which \textit{differs} from the two-dimensional theoretical version in Section~\ref{sec:lsdataset}) with binary features $x_{i} \in \{-1,+1\}$ and label $y \in \{-1,+1\}$. 
Each instance is generated as follows. First, the label $y$ is chosen to be $-1$ or $+1$ with
equal probability. Given $y$, the features $x_{i}$ are chosen according
to the following mixture distribution:
\begin{itemize}
    \item Large margin: with probability $1/4$, we choose $x_{i} = y$ for all $1 \leq i \leq 21$.
    \item Pullers: with probability $1/4$, we choose $x_{i} = y$ for $1 \leq i \leq 11$ and $x_{i} = -y$ for $12 \leq i \leq 21$.
    \item Penalizers: with probability $1/2$, we choose $5$ random coordinates from the first $11$ and $6$ from the last $10$ to be equal to the label $y$. The remaining $10$ coordinates are equal to $-y$.
\end{itemize}

\begin{figure}
    \centering
    \includegraphics[width=.7\linewidth]{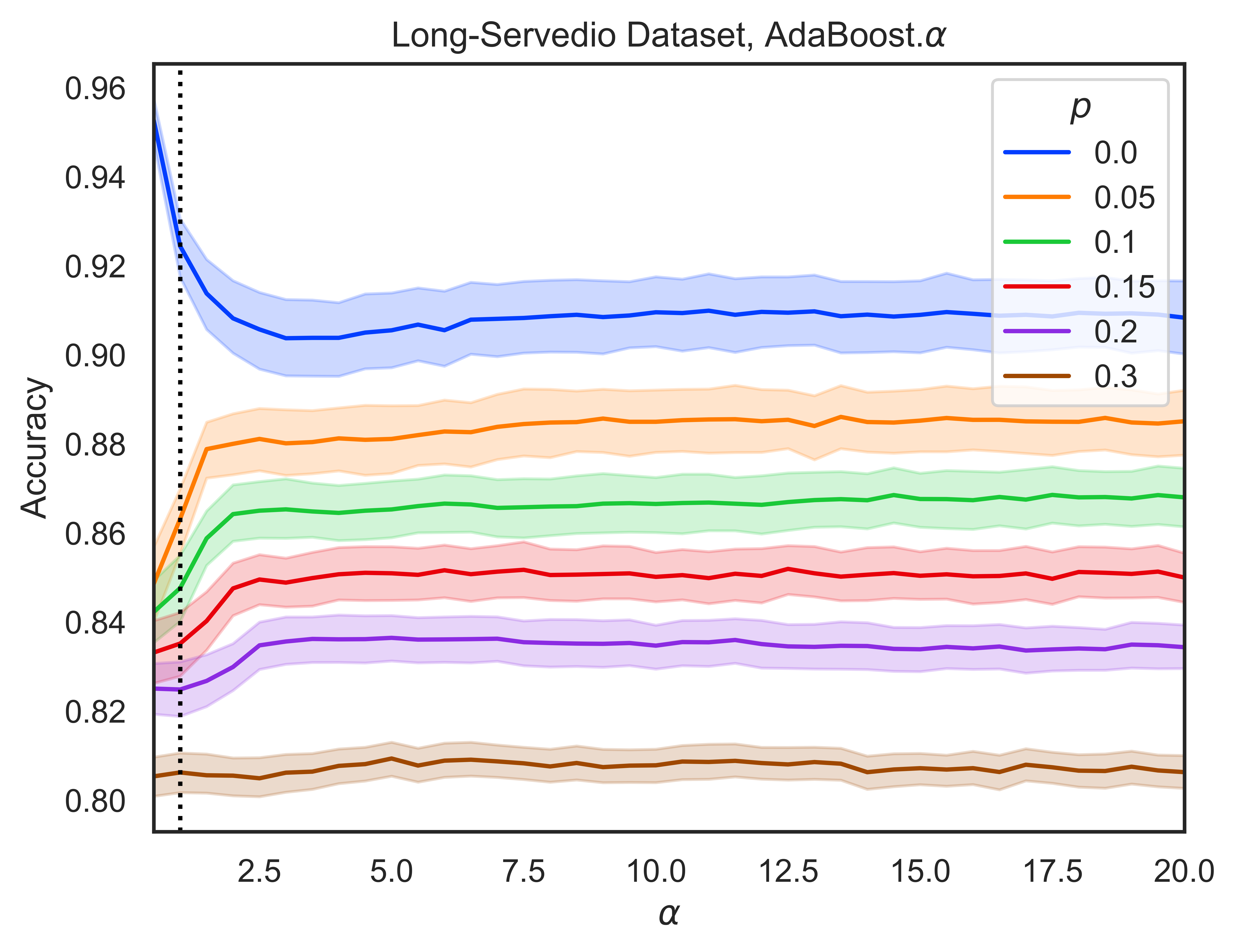}
    \caption{Accuracy of AdaBoost.$\alpha$ on the Long-Servedio dataset. We see that accuracy levels off as $\alpha$ increases, implying that tuning $\alpha$ can be as simple as choosing $\alpha \approx 5$.
    The thresholding behavior is supported by Figure~\ref{fig:LSoptimizationlandscape}
        }
    \label{fig:ls_accuracy_alpha}
\end{figure}
\begin{figure}
    \centering
    \includegraphics[width=0.7\linewidth]{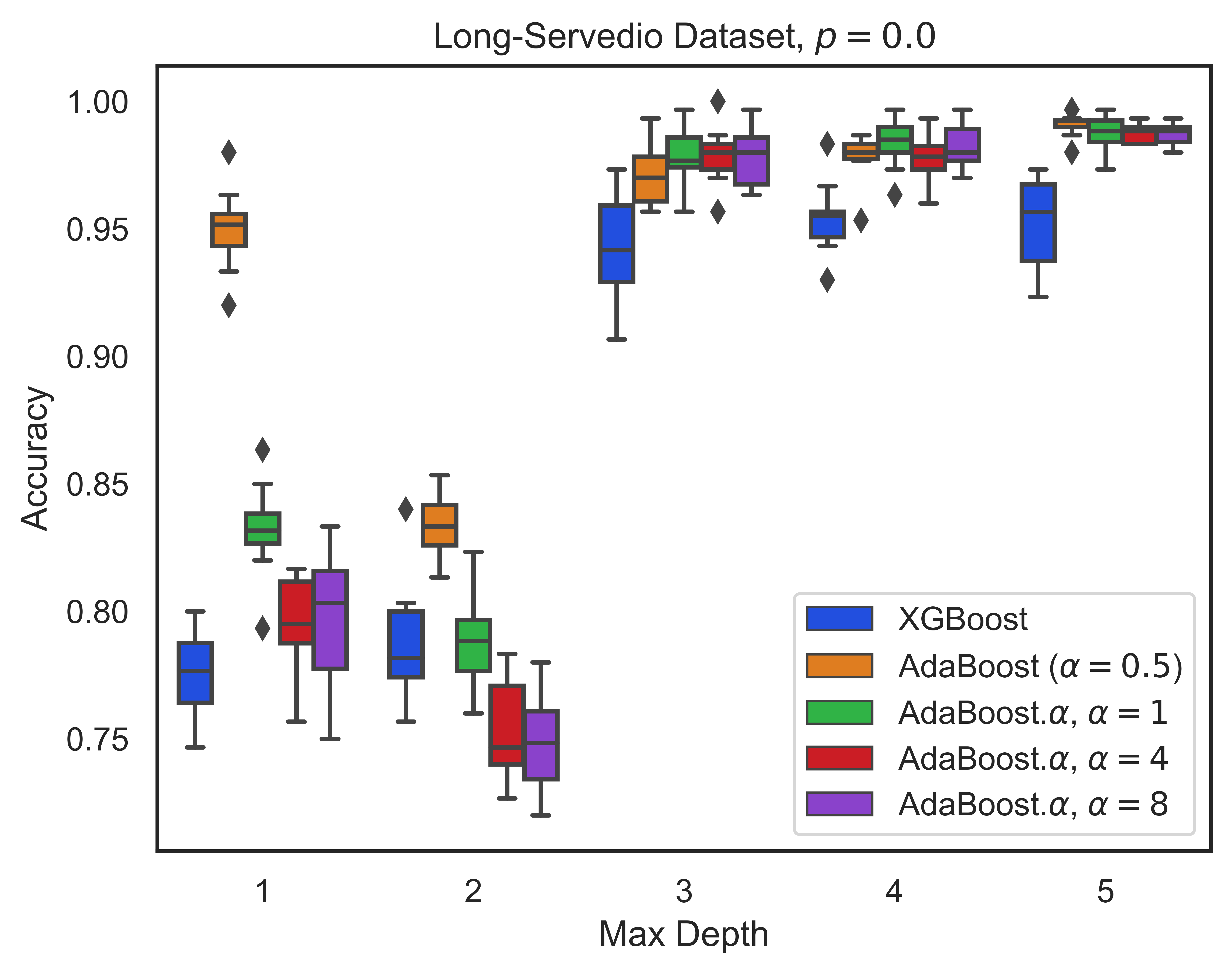}
    \caption{Clean test accuracy of various models on the Long-Servedio dataset with no added label noise. Models trained for 100 iterations. Vanilla AdaBoost performs well here, but note that Figure~\ref{fig:ls_iterations_0} implies that with a larger number of iterations, $\alpha=1,2$ would have similar performance.  }
    \label{fig:ls_boxplot_0_0}
\end{figure}
\begin{figure}
    \centering
    \includegraphics[width=0.7\linewidth]{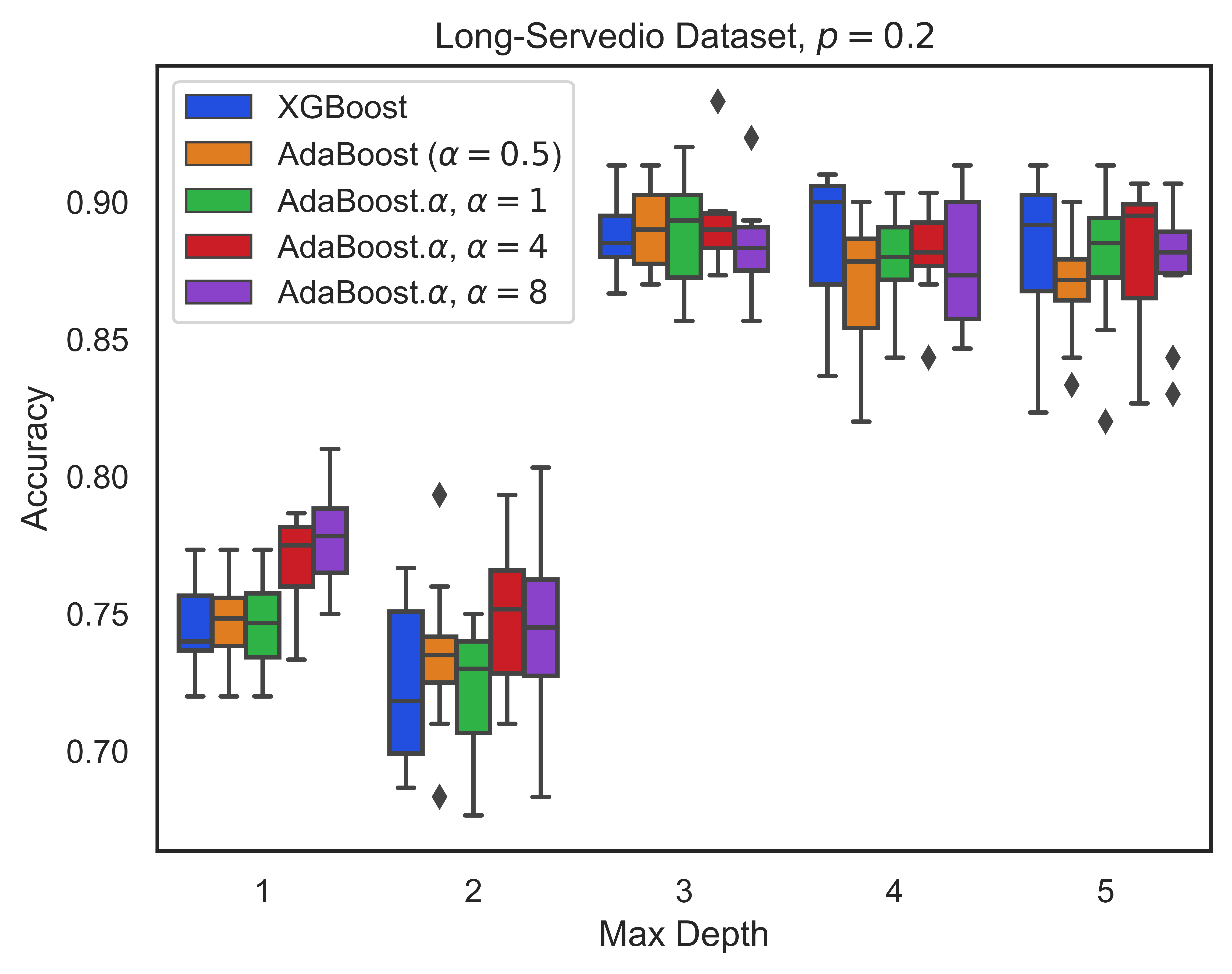}
    \caption{Clean test accuracy vs the depth of weak learners on the Long-Servedio dataset with SLN. 100 iterations of boosting. We see that that for low depth weak learners, $\alpha>1$ outperforms convex $\alpha$ in terms of clean classification accuracy. This benefit diminishes with growing depth.}
    \label{fig:ls_boxplot_0_2}
\end{figure}
\begin{figure}
    \centering
    \includegraphics[width=0.7\linewidth]{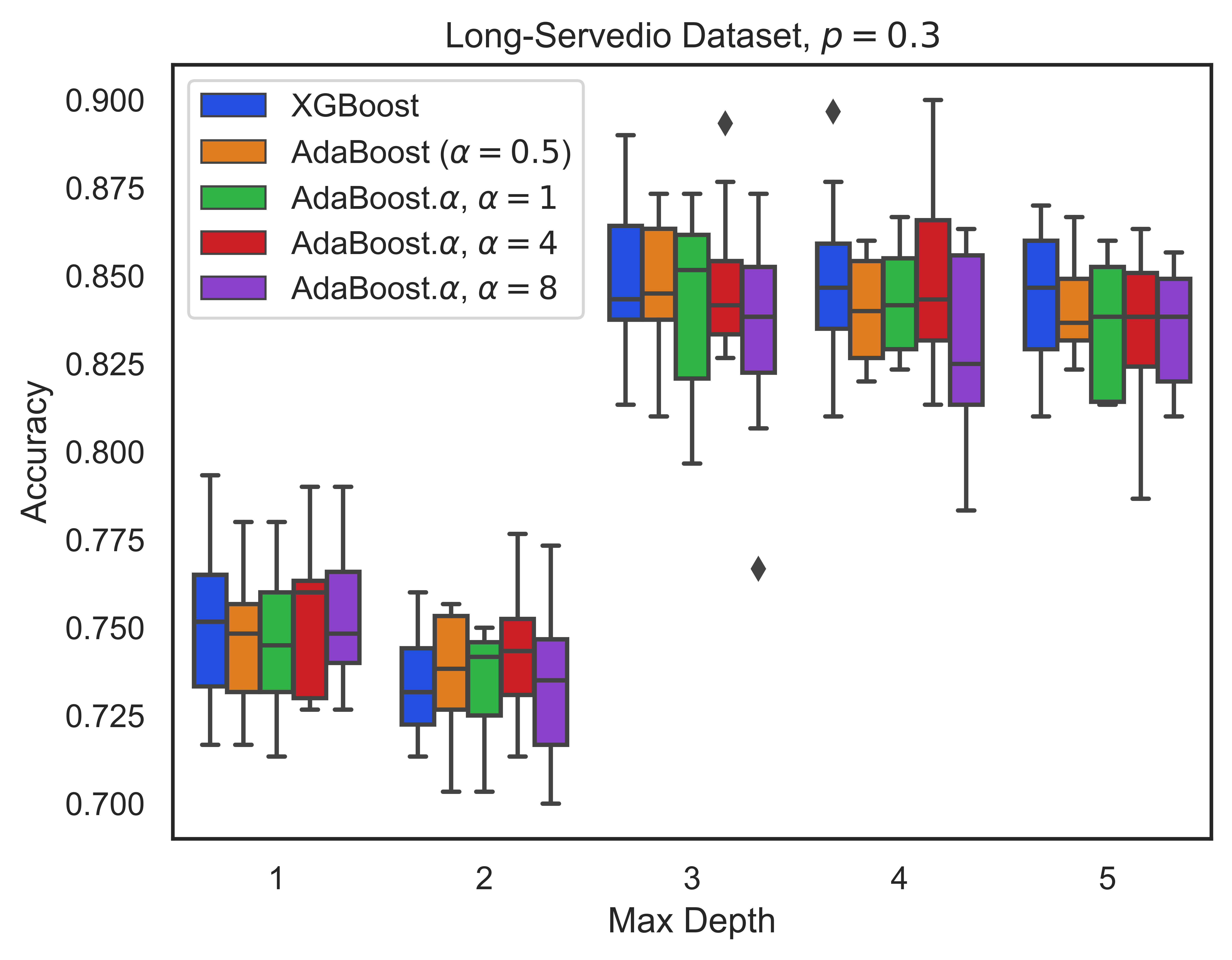}
    \caption{Clean test accuracy vs the depth of weak learners on the Long-Servedio dataset with SLN. 100 iterations of boosting. In this higher noise setting, $\alpha$ has little effect on the clean test accuracy.}
    \label{fig:ls_boxplot_0_3}
\end{figure}
\begin{figure}
    \centering
    \includegraphics[width=0.7\linewidth]{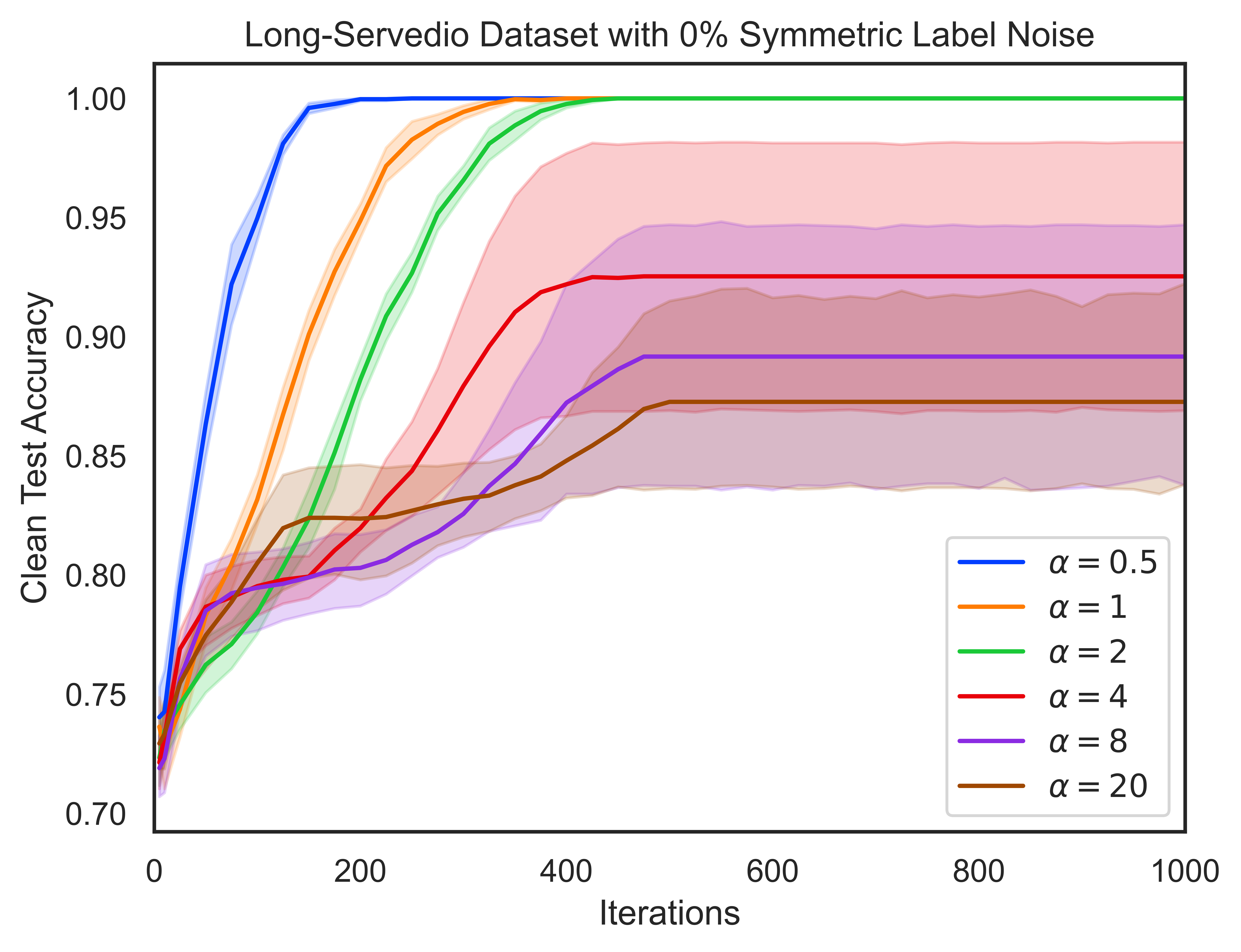}
    \caption{Clean test accuracy of AdaBoost.$\alpha$ on the Long-Servedio dataset with no added label noise. In this zero noise setting, convex $\alpha$ values perform well. Performance gains slow with increasing $\alpha$ which corresponds to increasing non-convexity in the optimization.}
    \label{fig:ls_iterations_0}
\end{figure}
\begin{figure}
    \centering
    \includegraphics[width=0.7\linewidth]{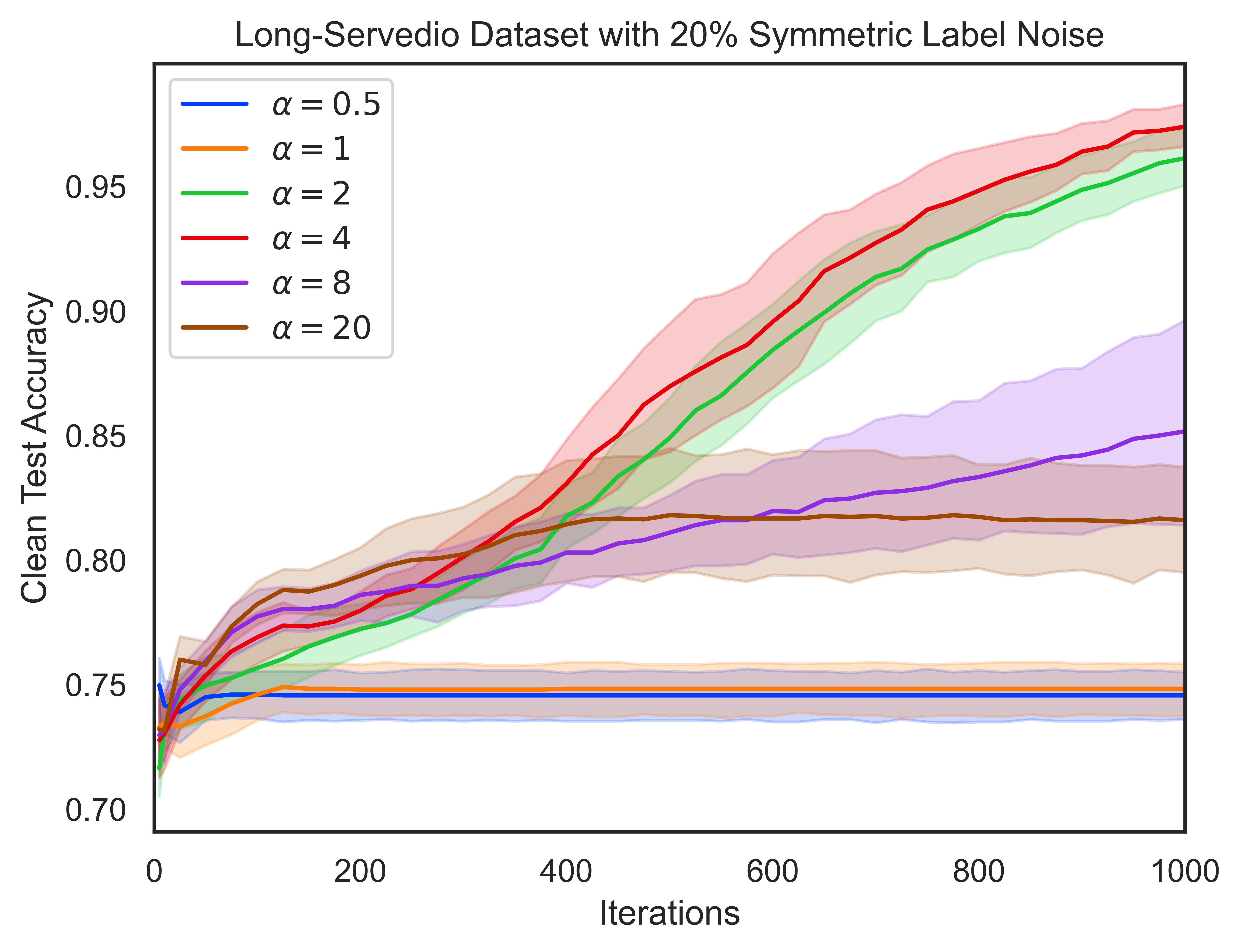}
    \caption{Accuracy of AdaBoost.$\alpha$ on the Long-Servedio dataset. We see that convex $\alpha<1$, is unable to learn by increasing the number of weak learners, likely because it is getting stuck trying to learn on large-margin example. $\alpha>1$ continues to learn with increasing iterations, though growth is slower than in smaller noise levels.}
    \label{fig:ls_iterations_2}
\end{figure}
\begin{figure}
    \centering
    \includegraphics[width=0.7\linewidth]{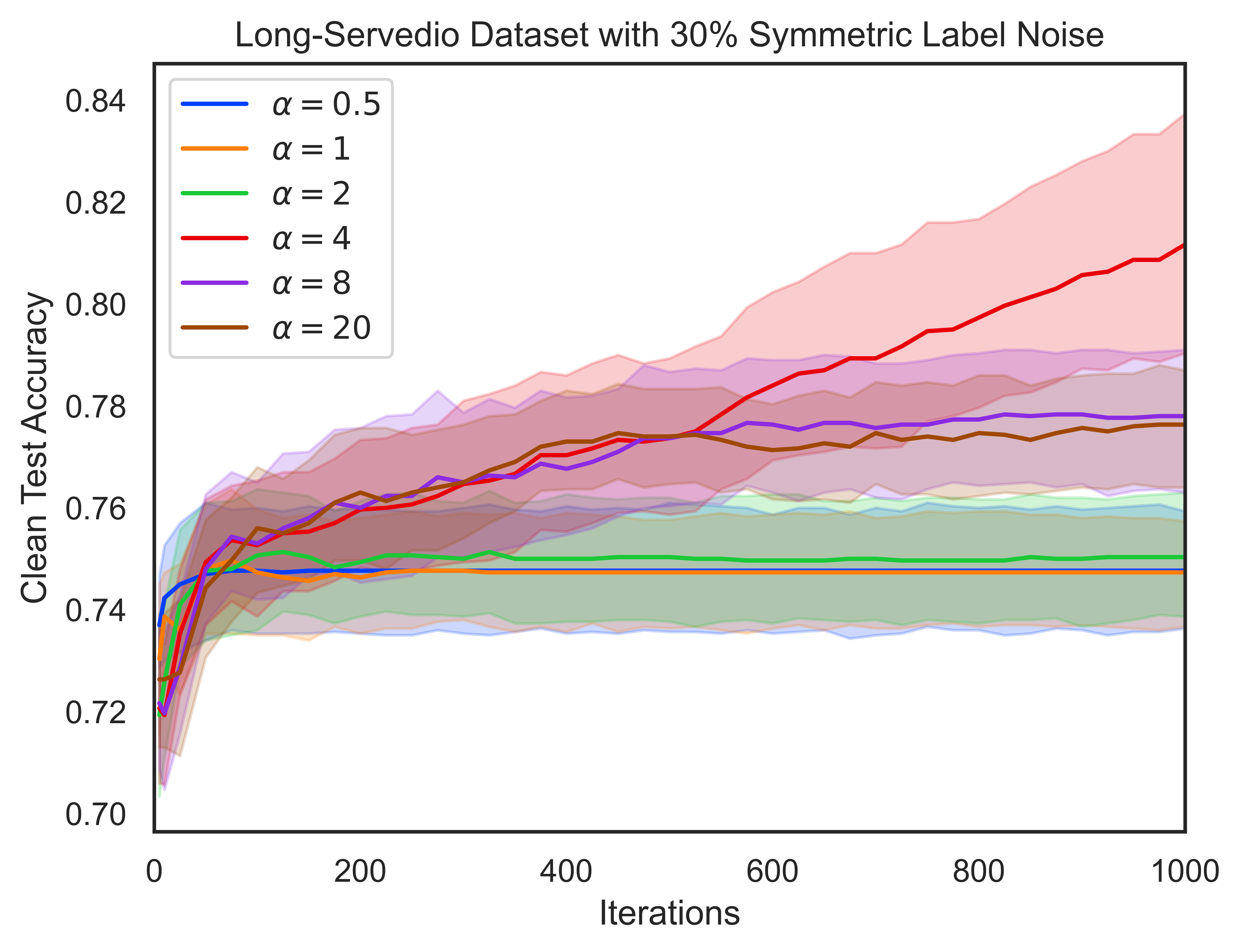}
    \caption{Accuracy of AdaBoost.$\alpha$ on the Long-Servedio dataset. We see that convex $\alpha<1$, is unable to learn by increasing the number of weak learners, likely because it is getting stuck trying to learn on large-margin example. $\alpha>1$ continues to learn with increasing iterations, though growth is slower than in smaller noise levels.}
    \label{fig:ls_iterations_3}
\end{figure}

\subsubsection{Breast Cancer}
\paragraph{Dataset} The Wisconsin Breast Cancer dataset
\citep{breastcancerdataset}
is a widely used medical dataset in the boosting community. 
\begin{figure}
    \centering
    \includegraphics[width=0.7\linewidth]{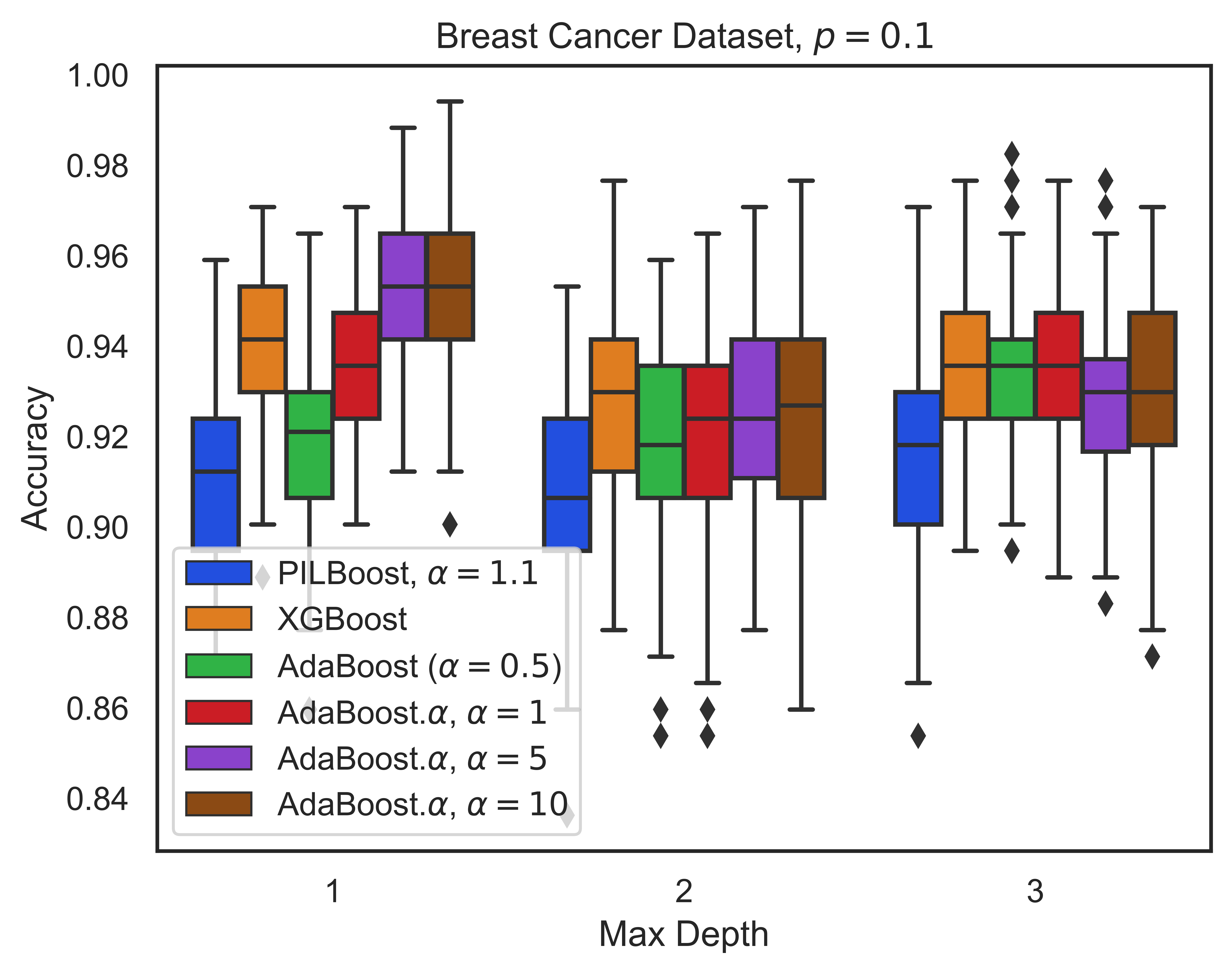}
    \caption{Accuracy of various models on the breast cancer dataset. We see that with low depth (and thus low complexity) weak learners, the use of a non-convex loss, namely $\alpha>1$, permits some gains in accuracy. These diminish for more complex weak learners.}
    \label{fig:bc_depth}
\end{figure}
\begin{figure}
    \centering
    \includegraphics[width=0.7\linewidth]{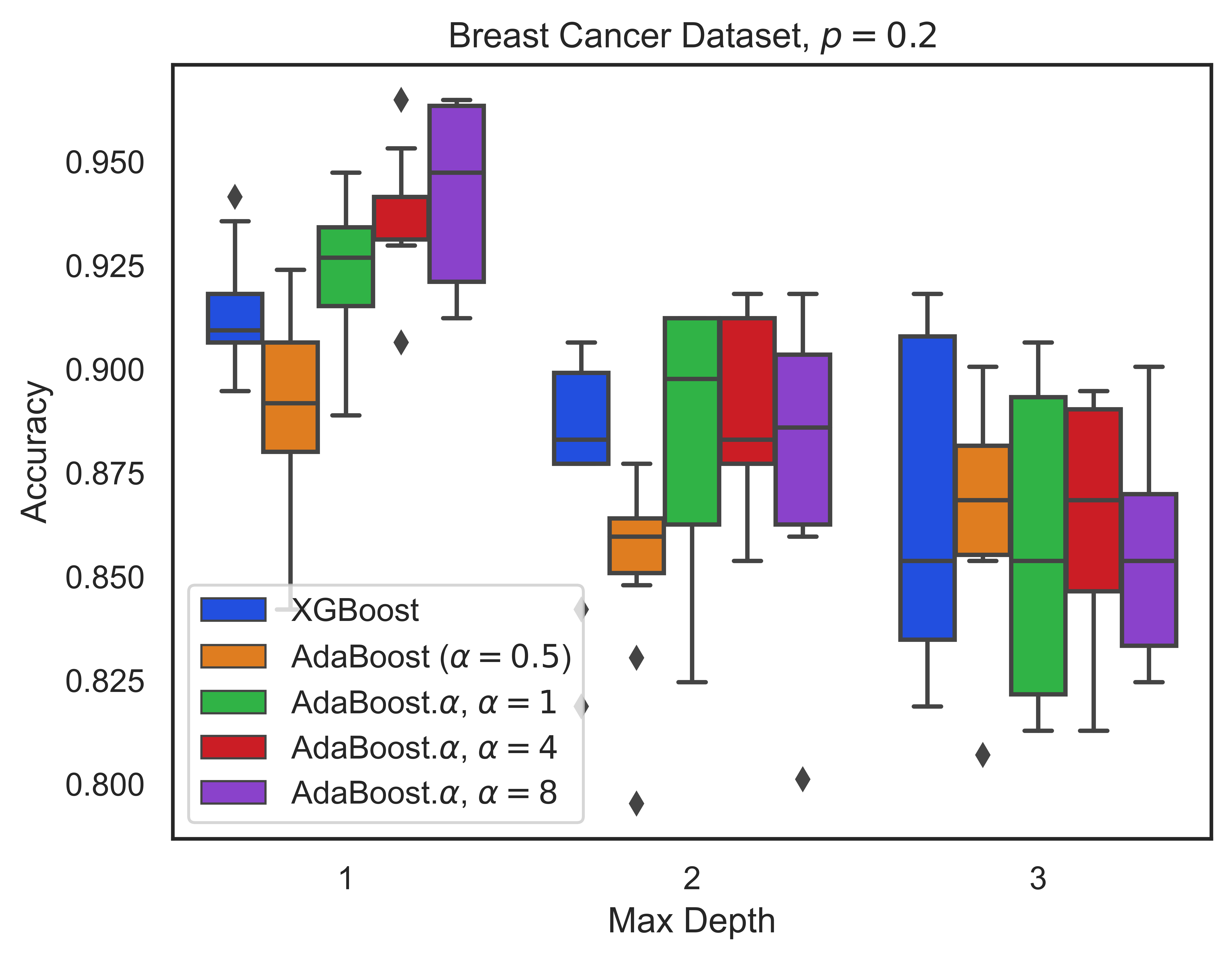}
    \caption{Accuracy of various models on the breast cancer dataset. We see that with low depth (and thus low complexity) weak learners, the use of a non-convex loss, namely $\alpha>1$, permits some gains in accuracy. These diminish for more complex weak learners.}
    \label{fig:bc_depth_0_2}
\end{figure}
\begin{figure}
    \centering
    \includegraphics[width=0.7\linewidth]{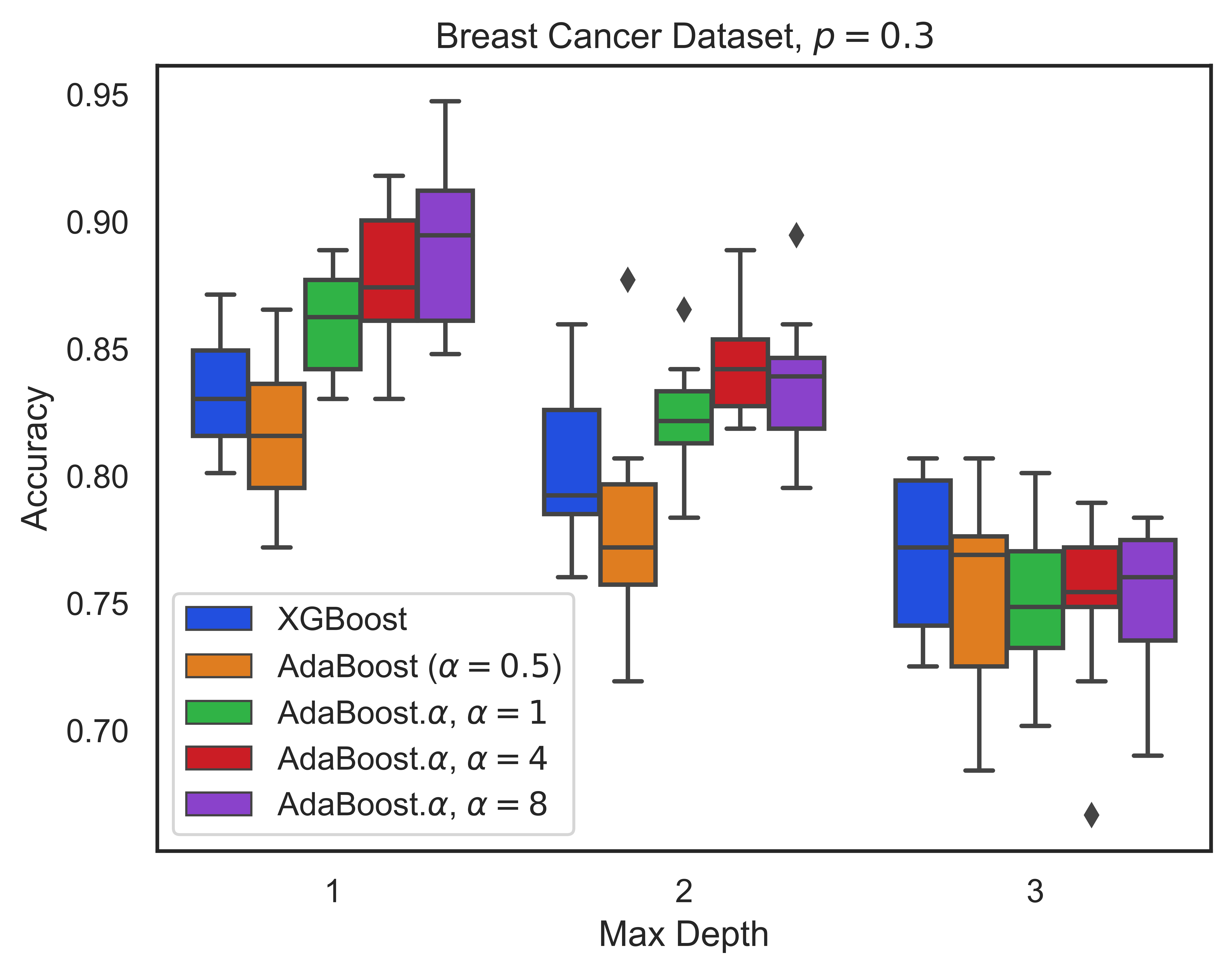}
    \caption{Accuracy of various models on the breast cancer dataset. We see that with low depth (and thus low complexity) weak learners, the use of a non-convex loss, namely $\alpha>1$, permits some gains in accuracy. These diminish for more complex weak learners.}
    \label{fig:bc_depth_0_3}
\end{figure}
\begin{figure}
    \centering
    \includegraphics[width=0.7\linewidth]{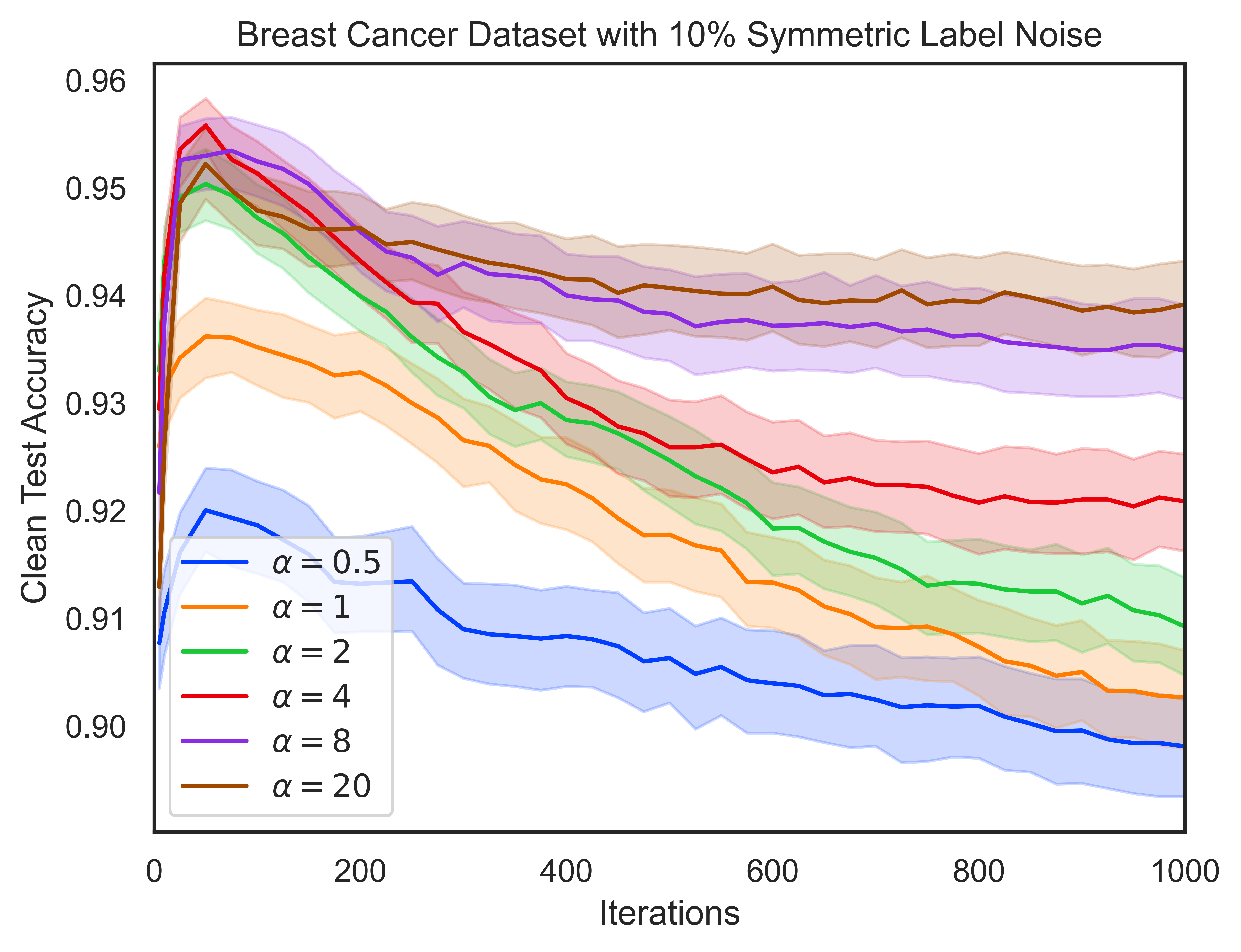}
    \caption{Accuracy of AdaBoost.$\alpha$ on the Wisconsin Breast Cancer dataset. Non-convex $\alpha$ values perform significantly better than convex $\alpha$ values. Unlike the Long-Servedio dataset, convex $\alpha$ values are still able to learn as the iterations increase, though there appears to be some overfitting.}
    \label{fig:bc_iterations}
\end{figure}
\begin{figure}
    \centering
    \includegraphics[width=0.7\linewidth]{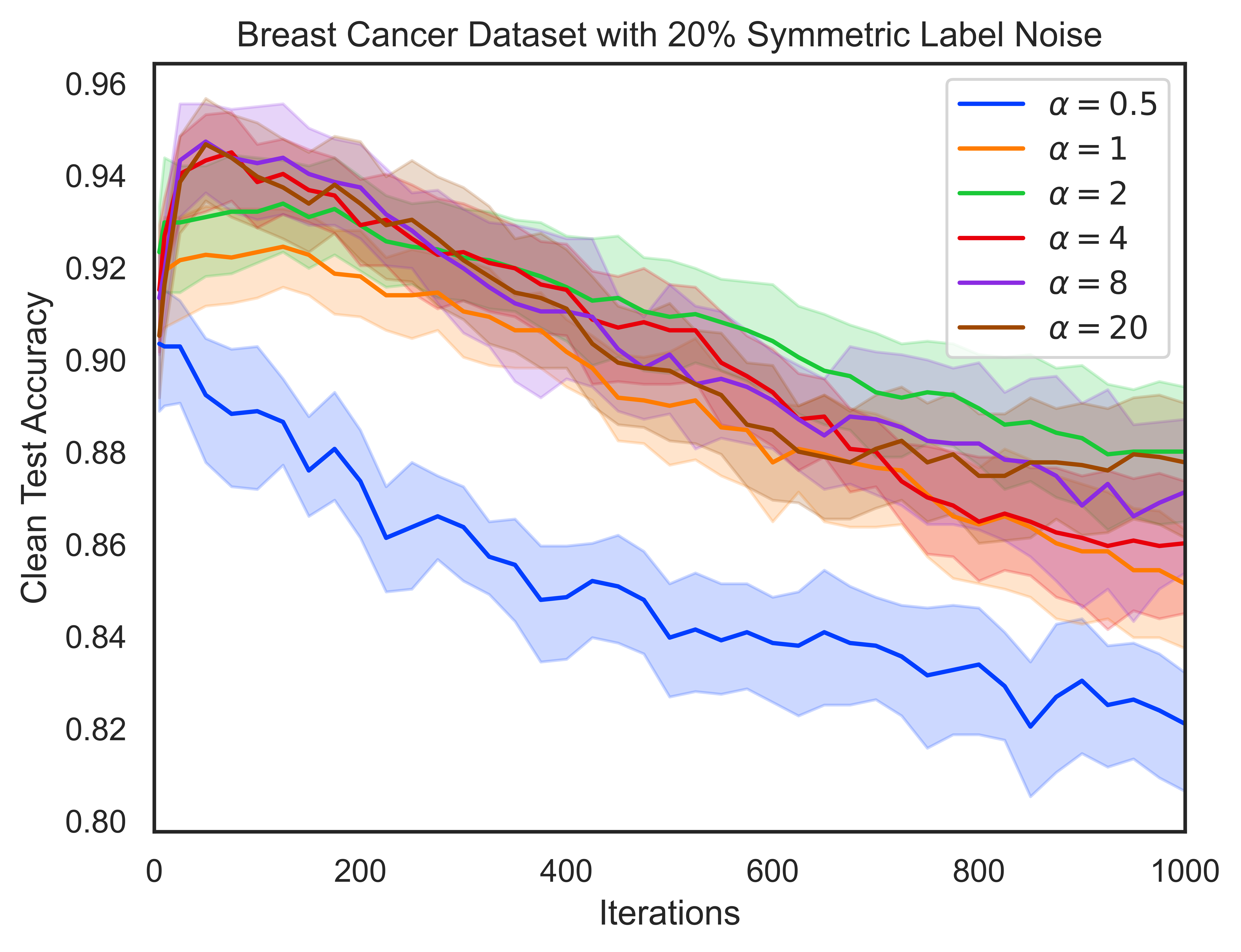}
    \caption{Accuracy of AdaBoost.$\alpha$ on the Wisconsin Breast Cancer dataset. Non-convex $\alpha$ values perform significantly better than convex $\alpha$ values. Unlike the Long-Servedio dataset, convex $\alpha$ values are still able to learn as the iterations increase, though there appears to be some overfitting.}
    \label{fig:bc_iterations_0_2}
\end{figure}
\begin{figure}
    \centering
    \includegraphics[width=0.7\linewidth]{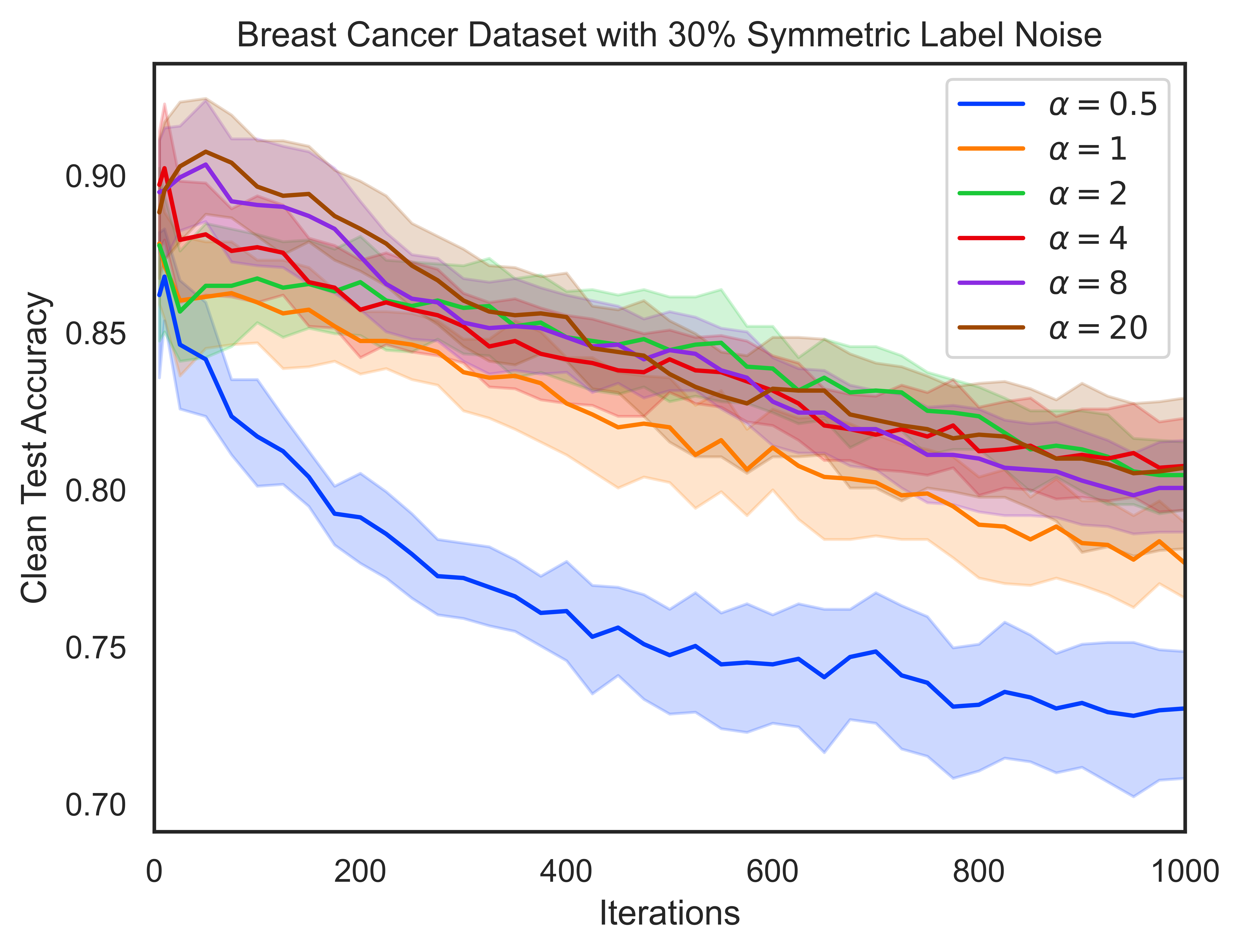}
    \caption{Accuracy of AdaBoost.$\alpha$ on the Wisconsin Breast Cancer dataset. Non-convex $\alpha$ values perform significantly better than convex $\alpha$ values. Unlike the Long-Servedio dataset, convex $\alpha$ values are still able to learn as the iterations increase, though there appears to be some overfitting.}
    \label{fig:bc_iterations_0_3}
\end{figure}

\subsection{Logistic Model Experiments} \label{appendix:logisticexperiments}
\subsubsection{GMM Setup}
\label{sec:synth_setup}
\paragraph{Dataset}
In order to evaluate the effect of generalizing log-loss with 
$\alpha$-loss in the logistic model, we first analyze its performance learning on a two-dimensional dataset with 
Gaussian class-conditional distributions. The data was distributed as follows: \begin{align*}
    Y=1: & X\sim\mathcal{N}[\mu_1, \sigma^2\mathbb{I}], \\ 
    Y=-1: & X\sim\mathcal{N}[\mu_2, \sigma^2\mathbb{I}],
\end{align*}
where $\mu_i \in \mathbb{R}^2$, $\sigma\in\mathbb{R}$, and $\mathbb{I}$ is the $2\times 2$ identity matrix. 

We evaluate this simple 
two-dimensional equivariant case for reasons of interpretability and visualization.
Additionally, we tune the prior of $Y$ in order to control the 
level of class imbalance in the dataset to demonstrate that $\alpha$-loss works well even under class imbalance conditions.
\begin{figure}[h]
    \centering
    \includegraphics[width=0.7\linewidth]{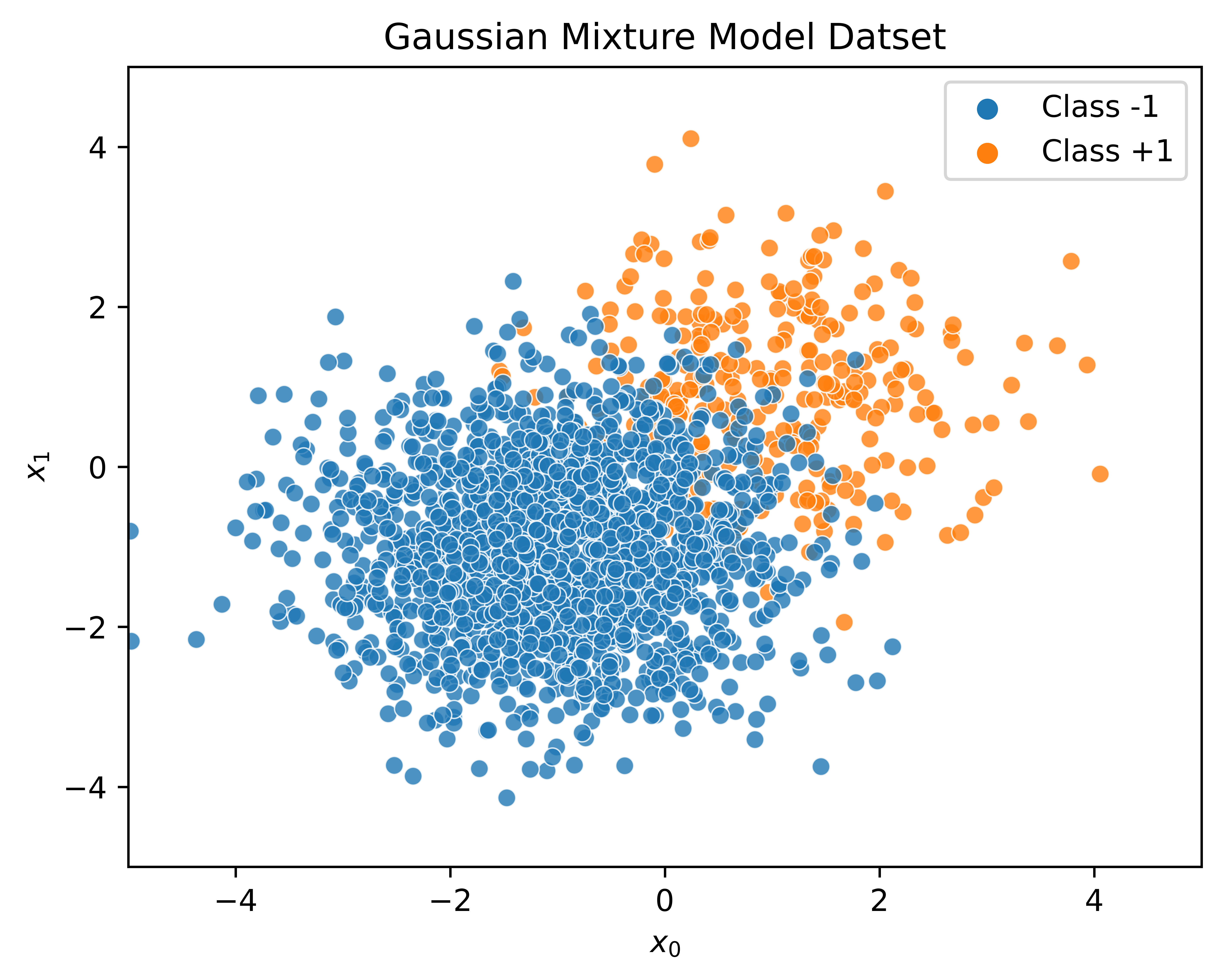}
    \caption{Sample dataset generated with Gaussian class-conditional distributions with $P(Y=1)=0.14$ and $\mu_1 = [1,1]^T, \mu_2=[-1,-1]^T$; we use a spherical covariance with $\sigma=1$ for both classes.
    }
    \label{fig:GMM}
\end{figure}
Symmetric label noise is then added to this clean data. 

Under this scenario, the Bayes-optimal classifier is
linear because the variances of the two modes are equal and the features are uncorrelated. We can see this directly through the likelihood ratio test.
Thus, we can compare the
separating line given by training with $\alpha$-loss on the logistic model and the optimal classifier.
\paragraph{Model}
A logistic model was trained on noisy data, then tested on clean
data from the same data generating distribution. Models were
trained over a grid of possible noise values, $p\in[0,0.4]$, and $\alpha\in[0.5,10]$.
Learning rate was selected as $1\mathrm{e}{-2}$ and models were 
trained until convergence.
For each pair, 30 models were trained with different
noise seeds, and metrics were then averaged across models. 
\subsubsection{COVID-19 Logistic Setup}
\paragraph{Model}
For better accuracy and a simpler, interpretable logistic 
model, we restrict the model to predict using a smaller set
of 8 features; we choose these as the features with the 
largest odds ratio on the validation set and they are enumerated
in Table~\ref{table:features}.
The learning rate was selected as $1\mathrm{e}{-3}$ and models were
trained until convergence.
Models were trained over a grid of possible noise values, $p$, and $\alpha$ values, 
$(p, \alpha) \in [0,0.15] \times [0.6, 3]$.
For each pair $(p,\alpha)$, 5 models were trained with 
a different random noise seed and results were averaged across 
these samples for every metric.
\paragraph{Baseline}
Because the underlying true statistics are not available as a ground truth, a ``clean'' model is selected for a baseline comparison. We select this model to be one with no added noise ($p=0$) and log-loss ($\alpha=1$).
Because log-loss ($\alpha=1$) is calibrated, the ``clean'' posterior distribution will be the distribution with the smallest KL divergence to the data-generating distribution.
\begin{figure}
    \centering
    \includegraphics[width=0.7\linewidth]{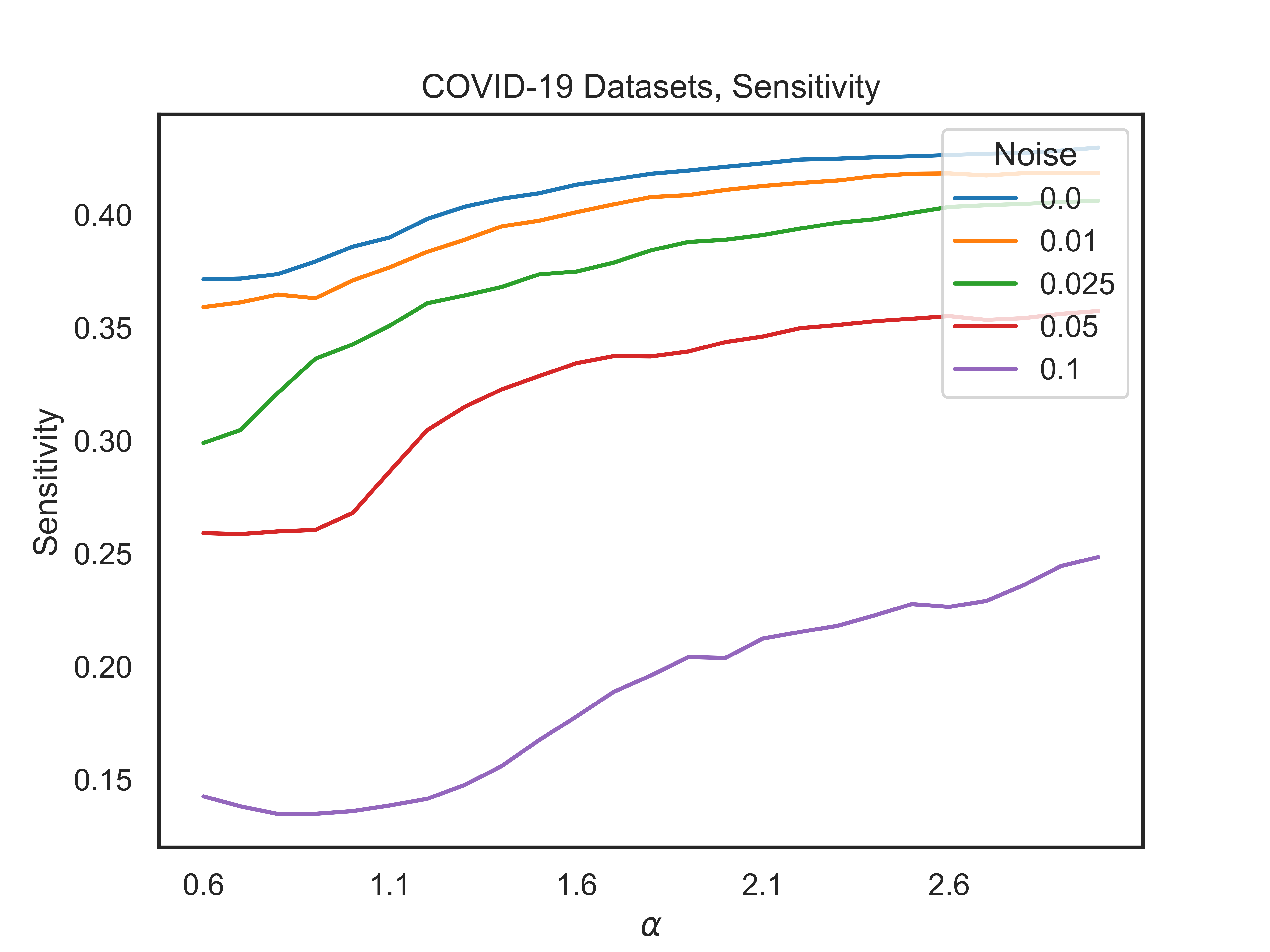}
    \caption{Sensitivity of the classifiers trained on noisy COVID-19 data. We see that $\alpha>1$ yields gains in sensitivity. This is important to note as the MSE results do not come at the cost of sensitivity. Recall that $\text{sensitivity} = \frac{\text{TP}}{\text{TP} + \text{FN}}$.}
    \label{fig:covid_sensitivity}
\end{figure}


\end{appendices}

\end{document}